\documentclass[review,onefignum,onetabnum]{siamonline250211}

\usepackage{lipsum}
\usepackage{amsfonts}
\usepackage{graphicx}
\usepackage{epstopdf}
\usepackage{algorithmic}
\ifpdf
  \DeclareGraphicsExtensions{.eps,.pdf,.png,.jpg}
\else
  \DeclareGraphicsExtensions{.eps}
\fi

\usepackage{enumitem}
\setlist[enumerate]{leftmargin=.5in}
\setlist[itemize]{leftmargin=.5in}

\newsiamremark{hypothesis}{Hypothesis}
\crefname{hypothesis}{Hypothesis}{Hypotheses}
\newsiamthm{claim}{Claim}
\newsiamremark{fact}{Fact}
\crefname{fact}{Fact}{Facts}

\headers{Non-Random Missing Multidimensional
Time Series Recovery}{Hao Shu and  Jicheng Li}

\title{Exact Recovery of Non-Random Missing Multidimensional Time Series via Temporal Isometric Delay-Embedding Transform \thanks{Submitted to the editors DATE.
}}

\author{Hao Shu\thanks{The School of Mathematics and Statistics, Xi'an Jiaotong University, Xi'an, Shaanxi, 710049, P. R. China (\email{haoshu812@gmail.com}).}
\and Jicheng Li\thanks{The Corresponding author with School of Mathematics and Statistics, Xi'an Jiaotong University, Xi'an, Shaanxi, 710049, P. R. China(\email{jcli@mail.xjtu.edu.cn}).}
\and Yu Jin\thanks{The School of Mathematics and Statistics, Xi'an Jiaotong University, Xi'an, Shaanxi, 710049, P. R. China  (\email{jinyu1491240@163.com}).}
\and Ling Zhou\thanks{The School of Mathematics and Statistics, Xi'an Jiaotong University, Xi'an, Shaanxi, 710049, P. R. China (\email{zlingu@126.com}).}
}

\usepackage{amsopn}

\usepackage{multirow}
\usepackage{graphicx}
\usepackage{diagbox}
\usepackage{subfig}

\newtheorem{Prop}{Proposition}[section]

\newtheorem{remark}{Remark}[section]
\newtheorem{defn}{Definition}[section]
\newtheorem{lemmaa}{Lemma}[section]

\usepackage{pifont}
\usepackage{tabularx}
\usepackage{booktabs}
\usepackage{bm}
\usepackage{mathabx}
\usepackage{cancel}

\newcommand{\A}{\mathcal{A}}
\newcommand{\B}{\mathcal{B}}
\newcommand{\C}{\mathcal{C}}

\newcommand{\E}{\mathcal{E}}
\newcommand{\G}{\mathcal{G}}

\newcommand{\I}{\mathcal{I}}

\newcommand{\M}{\mathcal{M}}
\newcommand{\N}{\mathcal{N}}

\newcommand{\U}{\mathcal{U}}
\newcommand{\V}{\mathcal{V}}
\newcommand{\W}{{\mathcal{W}}}
\newcommand{\X}{\mathcal{X}}
\newcommand{\Y}{\mathcal{Y}}
\newcommand{\Z}{\mathcal{Z}}

\newcommand{\cmark}{\ding{52}}  
\newcommand{\xmark}{\ding{56}}  
\newcommand{\norm}[1]{\lVert#1\rVert}
\newcommand{\normF}[1]{{\lVert#1\rVert}_F}
\newcommand{\normop}[1]{{\lVert#1\rVert}_{op}}

\newcommand{\normlarge}[1]{\left\lVert#1\right\rVert}
\newcommand{\0}{\mathcal{O}}

\newcommand{\Pomega}{\mathcal{P}_{\Omega}}
\newcommand{\Pomegac}{\mathcal{P}_{{\Omega}^{\perp}}}
\newcommand{\Pomegah}{\mathcal{P}_{\Omega_{\mathcal{H}} }}
\newcommand{\Pomegahc}{\mathcal{P}_{\Omega_{\mathcal{H}}^{\perp} }}

\newcommand{\Pt}{\mathcal{P}_{\mathbb{T}}}
\newcommand{\Ptc}{\mathcal{P}_{\mathbb{T}^{{\perp}}}}
\newcommand{\Pu}{\mathcal{P}_{\mathbb{U}}}

\newcommand{\Pv}{\mathcal{P}_{\mathbb{V}}}

\begin{document}

\maketitle

\begin{abstract}
Non-random missing data is a ubiquitous yet undertreated flaw in multidimensional time series, fundamentally threatening the reliability of data-driven analysis and decision-making. 
Pure low-rank tensor completion, as a classical data recovery method, falls short in handling non-random missingness, both methodologically and theoretically.
Hankel-structured tensor completion models provide a feasible approach for recovering multidimensional time series with non-random missing patterns. However, most Hankel-based multidimensional data recovery methods both suffer from unclear sources of Hankel tensor low-rankness and lack an exact recovery theory for non-random missing data.
To address these issues, we propose the temporal isometric delay-embedding transform, which constructs a Hankel tensor whose low-rankness is naturally induced by the smoothness and periodicity of the underlying time series. Leveraging this property, we develop the \textit{Low-Rank Tensor Completion with Temporal Isometric Delay-embedding Transform} (LRTC-TIDT) model, which characterizes the low-rank structure under the \textit{Tensor Singular Value Decomposition} (t-SVD) framework. 
Once the prescribed non-random sampling conditions and mild incoherence assumptions are satisfied, the proposed LRTC-TIDT model achieves exact recovery, as confirmed by simulation experiments under various non-random missing patterns. Furthermore, LRTC-TIDT consistently outperforms existing tensor-based methods across multiple real-world tasks, including network flow reconstruction, urban traffic estimation, and temperature field prediction.
Our implementation is publicly available at https://github.com/HaoShu2000/LRTC-TIDT.
\end{abstract}



\section{Introduction}

Multidimensional time series data are prevalent and play a crucial role across various real-world domains, such as transportation systems \cite{shu2024low}, climate science \cite{chen2021bayesian}, and international relations \cite{schein2016bayesian}. However, such data often contain various types of missing values due to unexpected events like sensor failures or signal loss. This not only complicates data utilization but also significantly impedes downstream applications, including conventional classification and regression \cite{che2018recurrent}, sequential data integration \cite{li2016novel}, and forecasting tasks \cite{yan2012toward}. Consequently, there is a pressing need for effective data recovery methods.

     In the task of time series recovery, the data are often represented as a $p+1$-order tensor  \cite{chen2022factor}, where the first mode denotes time and the remaining $p$ modes correspond to spatial or feature dimensions.
    When $p=1$, this corresponds to a multivariate time series matrix \cite{zheng2022multivariate}, and when $p\geq2$, it represents a multidimensional time series tensor \cite{chen2021bayesian}.
    For example, time-varying temperature field data, a common type of multidimensional time series, is typically structured as time $\times$ longitude $\times$ latitude. 
    The fundamental modeling approach for  time series recovery involves extracting the intrinsic features of the data to reconstruct complete datasets from a subset of specific observations \cite{chen2024laplacian}.
    Low-rankness has been recognized as an important prior for time series data \cite{chen2018autoregressive,chen2021scalable,wang2021generalized}, giving rise to a series of low-rank matrix/tensor completion models for time series recovery tasks \cite{liu2012tensor,chen2020nonconvex}.

However, the pure low-rank matrix/tensor completion models and their underlying theories are primarily developed under the assumption of randomly missing data \cite{liu2012tensor, candes2009exact, lu2019low, candes2010matrix,  zhang2016exact,  qin2022low}. 
    In real-world applications, sensor malfunctions (e.g., power failures) may cause severe segment missing, where all observations of a multidimensional time series are unavailable within a specific time window \cite{chen2021low}. Such missing patterns are particularly challenging to handle, as there is insufficient correlated information available within the current time segment to support accurate recovery. Even pure low-rank completion methods become completely ineffective in this setting; for instance, employing the \textit{Tensor Nuclear Norm} (TNN) often leads to a trivial all-zero solution \cite{Shu2025Guaranteed}. 
    From a theoretical standpoint, the lack of any distributional assumption on the missing entries renders existing recovery guarantees based on random sampling inapplicable.


 In recent studies, some Hankel-structured  matrix/tensor completion have been proposed to address the limitations of pure low-rank completion models in handling non-random missing data \cite{trickett2013interpolation,yokota2018missing,sedighin2020matrix,wang2023low,yamamoto2022fast}. 
These models introduce structural modifications that implicitly capture sequential dependencies, making them capable of dealing with non-random missing patterns, such as entirely absent rows or columns and missing temporal segments \cite{chen2024laplacian}.
However, current low-rank Hankel-structured recovery methods still have certain limitations. 
First, a significant number of low-rank Hankel methods are designed for univariate or multivariate time series \cite{trickett2013interpolation,wang2023low,gillard2018structured,butcher2017simple,zhang2018multichannel,zhang2019correction}, whereas existing Hankel-structured approaches applicable to multidimensional time series largely rely on the \textit{Multiway Delay-embedding Transform} (MDT) coupled with various tensor decompositions \cite{yokota2018missing,sedighin2021image,yamamoto2022fast}, 
 which typically require pre-estimation of the rank of the transformed tensor.
Second, unlike convolution-based structural adjustment schemes \cite{chen2024laplacian,Shu2025Guaranteed,liu2022recovery,liu2022time}, the MDT does not preserve consistency between the original data and the transformed tensor. Moreover, it lacks quantitative analysis linking the low-rankness of the transformed tensor to the fundamental characteristics of the original data \cite{sedighin2020matrix,sedighin2021image,yamamoto2022fast}.
Third, existing Hankel-structured tensor completion methods provide only qualitative descriptions of their recovery capabilities under non-random missing data patterns and lack exact recovery theories tailored to such scenarios \cite{yokota2018missing,sedighin2020matrix,sedighin2021image,wang2023low,jin2025high}.


\begin{table*}[!t]
    \centering
    \footnotesize
    \setlength{\tabcolsep}{3pt}
    \caption{Summary of classical time series recovery methods via low-rank  completion model.}\label{DT}
    \begin{tabular}{c||c|c|c|c|c}
        \hline
        \multirow{2}{*}{Type} & \multirow{2}{*}{Model} & Rank-free & Segment & Multi- & Exact recovery theory for \\
        & & estimation & missing & dimensional & non-random missing\\ \hline  
        \multirow{7}{*}{Non-Hankel} 
        & BCPF \cite{zhao2015bayesian} & \xmark & \xmark & \cmark & \xmark \\  
        & SPC-TV \cite{yokota2016smooth} & \xmark & \xmark & \cmark & \xmark \\  
        & MF-TV \cite{ji2016tensor} & \xmark & \xmark & \cmark & \xmark \\  
        & NN \cite{candes2009exact} & \cmark & \xmark & \xmark & \xmark \\
        & SNN \cite{liu2012tensor} & \cmark & \xmark & \cmark & \xmark \\
        & TNN \cite{lu2019low} & \cmark & \xmark & \cmark & \xmark \\   
        & TCTV \cite{wang2023guaranteed} & \cmark & \xmark & \cmark & \xmark \\ 
        \cline{1-6} 
        \multirow{8}{*}{Hankel} 
        & HTI\cite{trickett2013interpolation} & \xmark & \cmark & \xmark & \xmark \\
        & MDT-Tucker\cite{yokota2018missing} & \xmark & \cmark & \cmark & \xmark \\
        & MDT-TT \cite{sedighin2020matrix} & \xmark & \cmark & \cmark & \xmark \\
        & MDT-TR \cite{sedighin2021image} & \xmark & \cmark & \cmark & \xmark \\
        & STH-LRTC \cite{wang2023low} & \cmark & \cmark & \xmark & \xmark \\ 
        & MDT-HTNN \cite{jin2025high} & \cmark & \cmark & \cmark & \xmark \\ 
        & LRTC-TIDT (this work) & \cmark & \cmark & \cmark & \cmark \\ 
        \cline{1-6} 
    \end{tabular}
\end{table*}

To address these challenges, we propose the \textit{Low-Rank Tensor Completion with Temporal Isometric Delay-embedding Transform} (LRTC-TIDT) model and develop a unified theoretical framework for recovering multidimensional time series with non-random missingness. Specifically, we introduce the \textit{Temporal Isometric Delay-embedding Transform} (TIDT), a structural adjustment mechanism that preserves the consistency between the original time series and its Hankelization, and further theoretically demonstrate that the low-rank structure of the resulting Hankel tensor stems from the intrinsic periodicity and smoothness of the original time series. Building on the reliable t-SVD framework, we formulate LRTC-TIDT by minimizing the tensor nuclear norm to exploit this low-rank Hankel structure and establish exact recovery guarantees for general non-random missing patterns.
As shown in Table \ref{DT}, the proposed LRTC-TIDT model not only reconstructs multidimensional time series without requiring rank estimation but also supports segment-missing recovery, and further provides exact recovery guarantees for non-random missingness. The results from both synthetic and real-world experiments corroborate the effectiveness of the model in handling non-random missing patterns, thereby validating its theoretical soundness and practical performance.
Overall, our contributions can be summarized in three key aspects:
\begin{itemize}

\item The \textit{Temporal Isometric Delay-Embedding Transform} (TIDT) is introduced as a novel tool for structural adjustment. Building upon this, we propose the LRTC-TIDT model, which exploits the low-rank structure of the transformed tensor within the t-SVD framework to tackle the challenging problem of recovering multidimensional time series with non-random missing entries.

\item
By introducing the minimum temporal sampling rate and the incoherence of the transformed tensor, we establish theoretical recovery guarantees for the LRTC-TIDT model under non-random missing patterns, covering both  noiseless and noisy scenarios. Notably, the prediction problem is naturally incorporated, thereby yielding corresponding theoretical guarantees for forecasting.

\item An efficient algorithm based on \textit{alternating direction method of multipliers} (ADMM) is developed to solve the optimization problem of the LRTC-TIDT  model. Furthermore, extensive numerical experiments conducted on several real-world multidimensional time datasets address three common scenarios of non-random missingness, highlighting the superiority of our proposed model over existing tensor models.
\end{itemize}

The remainder of this paper is organized as follows. 
Section~\ref{PRELIMINARIES} introduces  necessary preliminaries.
Section~\ref{methodology} presents the proposed LRTC-TIDT method, while Section~\ref{Non-random Missing Time Series Imputation Theory} develops the theoretical framework for time series recovery under non-random missingness. Section~\ref{sec:Algorithm} details the corresponding optimization algorithm. Experimental results on several real-world datasets are reported in Section~\ref{sec:experiments}. Finally, Section~\ref{sec_conclusion} concludes the paper.

\section{Preliminaries}\label{PRELIMINARIES}
\subsection{Overview of Tensor Algebraic Framework}\label{Overview of tensor algebraic framework}

We first give a short overview of the main ideas in the \textit{tensor Singular Value Decomposition} (t-SVD) framework \cite{kilmer2011factorization} \cite{kilmer2021tensor} \cite{lu2019tensor} \cite{martin2013order} \cite{qin2022low}. This overview covers the tensor–tensor product, the tubal rank, and the tensor nuclear norm. Our aim in this subsection is to emphasize that these tensor notions are closely aligned with the familiar concepts from matrix analysis.

For clarity, we adopt the following notation throughout the paper: scalars are written in lowercase letters, vectors in bold lowercase, matrices in capital letters, and tensors in Euler script, e.g., $z \in \mathbb{R} $, $\bm{z} \in \mathbb{R}^{n}$, $Z \in \mathbb{R}^{n_1 \times n_2}$, and $\mathcal{Z} \in \mathbb{R}^{n_1 \times n_2 \times \cdots \times n_d}$.
Consider an order-$d$ tensor $\mathcal{Z}$ with dimensions $n_1 \times n_2 \times \cdots \times n_d$. Its entries and slices are denoted by $\mathcal{Z}{(...)}$. For instance, $\mathcal{Z}{(i_1,i_2,\cdots,i_d)}$ denotes the $(i_1,i_2,\cdots,i_d)$-th element, while $\mathcal{Z}{(:,:,i_3,\cdots,i_d)}$ denotes the $(i_3,\cdots,i_d)$-th face slice. Similarly, the horizontal and lateral slices are written as $\mathcal{Z}(i_1, :, \cdots, :)$ and $\mathcal{Z}(:, i_2, :, \cdots, :)$, respectively, in line with the third-order case \cite{kolda2009tensor}.

\begin{defn}[T-product \cite{jiang2021dictionary}]
For  tensors $\mathcal{A}\in\mathbb{R}^{n_1\times a\times n_3\times\cdots\times n_d}$ and $\mathcal{B}\in\mathbb{R}^{a\times n_2\times n_3\times\cdots\times n_d}$,  the t-product $\mathcal{A}*\mathcal{B}$  is defined to be a tensor
of size $n_1\times n_2\times n_3\times\cdots\times n_d$,
\begin{equation}\label{tprod-conv}
[\mathcal{A}*\mathcal{B}]{(i_1,i_2,:,\cdots,:)}=\sum_{j=1}^{a}\A{(i_1,j,:,\cdots,:)} \star \B{(j,i_2,:,\cdots,:)},
\end{equation}
where $\star$ denotes the operator of circular convolution \cite{fahmy2012new}.
\end{defn}
\begin{remark}\label{tensor-matrix}
The t-product of tensors is an operation analogous to matrix multiplication
$
[AB]{(i_1,i_2)}=\sum_{j=1}^{a} A{(i_1,j)}B{(j,i_2)},
$
where $A \in \mathbb{R}^{n_1\times a}$ and $B \in \mathbb{R}^{a\times n_2}$,
thus many classical properties of matrix multiplication can be naturally extended to the t-product \cite{qin2022low,lu2019tensor}. Under the t-product framework,  tensor horizontal slice $\A(i_1,:,\cdots,:)$ plays a role analogous to the row $A(i_1,:)$ of a matrix, while lateral slice $\B(:,i_2,\cdots,:)$ corresponds to the matrix column $B(:,i_2)$. 
Specifically, in matrix multiplication, the entry $[AB]{(i_1,i_2)}$ is obtained by multiplying the $i_1$-th row of $A$ with the $i_2$-th column of $B$. Similarly, in the tensor case, $[\mathcal{A}*\mathcal{B}]{(i_1,i_2,:,\cdots,:)}$ is computed via the circular convolution between the $i_1$-th horizontal slice of $\A$ and the $i_2$-th lateral slice of $\B$.
\end{remark}

In addition, the discrete Fourier transform (DFT) converts circular convolution into element wise multiplication \cite{qin2022low}. 
Thus, the t-product can be efficiently computed in the Fourier domain as
$
\mathcal{A} * \mathcal{B} 
= \mathcal{F}^{-1}\!\left( \mathcal{F}(\mathcal{A}) \, \Delta \, \mathcal{F}(\mathcal{B}) \right),
$
where $\mathcal{F}$ applies DFT along the last $d-2$ modes and $\Delta$ denotes face-wise matrix multiplication, i.e., 
$\mathcal{Z}=\mathcal{X}\Delta\mathcal{Y}\Leftrightarrow \mathcal{Z}{(:,:,i_3,\cdots,i_d)}=\mathcal{X}{(:,:,i_3,\cdots,i_d)}\mathcal{Y}{(:,:,i_3,\cdots,i_d)}$ for all face slices.

\begin{defn}[Transpose \cite{qin2022low}]
For a tensor $\mathcal{Z}\in\mathbb{R}^{n_1\times n_2 \times\cdots\times n_d}$, its transpose $\mathcal{Z}^\mathrm{T}$ satisfies 
$\mathcal{Z}^\mathrm{T}_\mathcal{F}(:,:,i_3,\cdots,i_d)=\mathcal{Z}_\mathcal{F}(:,:,i_3,\cdots,i_d)^\mathrm{T}$ for all face slices.
\end{defn}

\begin{defn}[Identity tensor \cite{qin2022low}]
The identity tensor $\mathcal{I}_n \in \mathbb{R}^{n\times n\times n_3\times\cdots\times n_d}$ is defined such that every face slice $\mathcal{I}_n(:,:,i_3,\cdots,i_d)$ equals the $n\times n$ identity matrix $I_n$.

\end{defn}

\begin{defn}[Orthogonal tensor \cite{qin2022low}]
A tensor $\mathcal{U}\in\mathbb{R}^{n\times n\times n_3\times\cdots\times n_d}$ is orthogonal if $\mathcal{U}^\mathrm{T}*\mathcal{U}
=\mathcal{U}*\mathcal{U}^\mathrm{T}=\mathcal{I}_n$, where $\mathcal{I}_n \in\mathbb{R}^{n\times n\times n_3\times\cdots\times  n_d}$  is  an  identity tensor.
\end{defn}

\begin{defn}[F-diagonal tensor \cite{qin2022low}]
A tensor is said to be f-diagonal if all of its face slices are diagonal.  
\end{defn}

\begin{theorem}[T-SVD \cite{qin2022low}]
For any  tensor $\mathcal{Z}\in\mathbb{R}^{n_1\times n_2\times\cdots\times  n_d}$, it can be
factorized as
$\label{eq.6}
\mathcal{Z} = \mathcal{U}*\mathcal{S}*\mathcal{V}^\mathrm{T},$
where $\mathcal{U}\in\mathbb{R}^{n_1\times n_1\times n_3\times\cdots\times  n_d}$, $\mathcal{V}\in\mathbb{R}^{n_2\times n_2\times n_3\times\cdots\times  n_d}$ are orthogonal tensors, and $\mathcal{S}\in\mathbb{R}^{n_1\times n_2\times\cdots\times  n_d}$ is a f-diagonal tensor.
\end{theorem}

\begin{defn}[T-SVD rank \cite{qin2022low}]
For $\mathcal{Z}\in\mathbb{R}^{n_1\times n_2\times n_3\times\cdots\times  n_d}$ with t-SVD $\mathcal{Z} = \mathcal{U}*\mathcal{S}*\mathcal{V}^\mathrm{T}$, its t-SVD rank is defined as
$
\operatorname{rank}_{\operatorname{t-SVD}}(\mathcal{Z}):=\sharp\{i: \mathcal{S}(i, i,:, \cdots,:) \neq  \bm{0}\},$
where $\sharp$ denotes the cardinality of a set.
\end{defn}

\begin{defn}[Tensor multi-rank \cite{qin2022low}]
For  a tensor  $\mathcal{Z}\in\mathbb{R}^{n_1\times n_2\times n_3\times\cdots\times  n_d}$, its multi-rank is represented by a vector $\bm{r}\in\mathbb{R}^{n_3\cdots n_d}$. The 
$i$-th element of $\bm{r}$ corresponds to the rank of the 
$i$-th  block of $\operatorname{bdiag}(\mathcal{X}_\mathcal{F})$.
 We denote the tensor multi-rank sum  as $r_s$, which means $r_s=\sum_{i=1}^{n_3\cdots n_d}\bm{r}^{(i)}$.
\end{defn}

\begin{defn}[Tensor spectral norm \cite{qin2022low}]
For tensor $\mathcal{Z}\in\mathbb{R}^{n_1\times n_2\times n_3\times\cdots\times n_d}$ , its tensor spectral norm is defined as
$
\|\mathcal{Z}\| :=
\|\operatorname{bdiag}(\mathcal{Z}_\mathcal{F})\|,
$
where $\|\cdot\|$ denotes the spectral norm of a  matrix or tensor.
\end{defn}

\begin{defn}[Tensor nuclear norm \cite{qin2022low}]
For tensor $\mathcal{Z}\in\mathbb{R}^{n_1\times n_2\times n_3\times\cdots\times  n_d}$ under the t-SVD framework, its tensor nuclear norm (TNN) is defined as
$
\|\mathcal{Z}\|_{\circledast}:= \frac{1}{n} \norm{\operatorname{bdiag}(\mathcal{Z}_\mathcal{F})}_*,
$
where $n=n_3\times\cdots\times n_d$  and  $\|\cdot\|_*$ denotes the nuclear norm of a matrix. Note that the tensor nuclear norm of an order-$d$ tensor is the dual norm of its tensor spectral norm, which is consistent with the matrix case \cite{candes2009exact} \cite{lu2019tensor}.
\end{defn}

\begin{theorem}[T-SVT \cite{qin2022low}]\label{th.2}
Given $\mathcal{Z}\in\mathbb{R}^{n_1\times n_2\times\cdots\times  n_d}$ with t-SVD  $\mathcal{Z} = \mathcal{U}*\mathcal{S}*\mathcal{V}^\mathrm{T}$, its tensor singular value thresholding (t-SVT) is defined by $\operatorname{t-SVT}_\tau(\mathcal{Z}):=\mathcal{U}*\mathcal{S}_\tau *\mathcal{V}^\mathrm{T}$, where $\mathcal{S}_\tau = \mathcal{F}^{-1}(\mathcal{F}(\mathcal{S}-\tau)_+)$, $a_+=\max(0,a)$, which obeys
\begin{equation}\label{eq.9}
\operatorname{t-SVT}_\tau(\mathcal{Z})=\arg\min_{\mathcal{X}} \tau\|\mathcal{X}\|_{\circledast}+\frac{1}{2}\|\mathcal{X}-\mathcal{Z}\|_\mathrm{F}^2.
\end{equation}
\end{theorem}

\subsection{Multiway Delay-embedding Transform}
The standard delay embedding operation
maps a vector $\bm{m}=[m_1,m_2,\cdots,m_t]^{\mathrm{T}} \in \mathbb{R}^{t}$ into a Hankel matrix of size ${\tau} \times (t-\tau+1)$:
\begin{equation}
\mathcal{D}_{\tau}(\bm{m})=\left[\begin{array}{cccc}
m_1 & m_2 & \cdots & m_{t-{\tau}+1} \\
m_2 & m_3 & \cdots & m_{t-{\tau}+2} \\
\vdots & \vdots & \vdots & \vdots \\
m_{\tau} & m_{{\tau}+1} & \cdots & m_{t}
\end{array}\right].
\end{equation}
The \textit{Multiway Delay-embedding Transform} (MDT) generalizes the standard delay-embedding transform by performing Hankelization along each dimension of the tensor \cite{yokota2018missing,yamamoto2022fast}.
For $\mathcal{M}\in\mathbb{R}^{n_1\times\cdots\times n_d}$ with the delay vector $\bm{\tau}=(\tau_1,\ldots,\tau_d)$, MDT is given by
\begin{equation}
\mathcal{D}_{\bm{\tau}}(\mathcal{M})=
\mathrm{fold}_{(\bm{\tau},\bm{n-\tau+1})}\!\left(\mathcal{M}\times_1 S_1 \times_2 \cdots \times_d S_d\right),
\end{equation}
where $\bm{n}=[n_1,n_2,\cdots,n_d]$, $S_j \in \{0,1\}^{\tau_j(n_j-\tau_j+1)\times n_j}(j=1,2,\cdots,d)$ are duplication matrices and $\mathrm{fold}_{(\bm{\tau},\bm{n-\tau+1})}$ maps an order-$d$ tensor to an order-$2d$ tensor.

\subsection{Non-Random Missing Time Series Recovery}
The problem of recovering time series with non-random (deterministic) missing patterns fundamentally differs from the random missing case. 
Let $\M \in \mathbb{R}^{t \times n_1 \times n_2 \cdots \times n_p}$ denote the unknown target tensor, with observations available only on a deterministic sampling set 
$\Omega \subseteq  [t] \otimes [n_1] \otimes [n_2] \otimes \cdots \otimes [n_p]$. 
Unlike random sampling, the observation pattern here cannot be modeled or reduced to any probability distribution. 
The corresponding mask tensor $\bar{\Omega}$ is defined entrywise as
\begin{equation}
[\bar{\Omega}]{(i_t,i_1, \ldots, i_p)} =
\begin{cases}
1, & \text{if } (i_t,i_1, \ldots, i_p) \in \Omega, \\
0, & \text{otherwise},
\end{cases}
\end{equation}
and the sampling operator $\Pomega$ is given by
$
\Pomega(\M) = \bar{\Omega} \circ \M,
$
where $\circ$ denotes the Hadamard product.  
In the noiseless setting, the goal is to recover $\M$ exactly from the deterministic observations $\Pomega(\M)$ by solving
\begin{equation}
\min_{\X \in \mathbb{R}^{t \times n_1 \times \cdots \times n_p}} ~ \mathcal{R}(\X), 
\quad \text{s.t.} \quad \|\Pomega(\X - \M)\|_F = 0,
\end{equation}
where $\mathcal{R}(\X)$ denotes a prescribed regularizer. 
When the observed entries are corrupted by additive noise, the noisy observation is written as
$
\Y = \Pomega(\M + \E),
$
with $\E$ representing a noise term that may be stochastic or deterministic.  
Following \cite{candes2010matrix}, it is assumed that
$\|\Pomega(\M - \Y)\|_F \leq \delta$, where $\delta$ controls the noise level. 
The recovery model in this case becomes
\begin{equation}
\min_{\X \in \mathbb{R}^{t \times n_1 \times \cdots \times n_p}} ~ \mathcal{R}(\X), 
\quad \text{s.t.} \quad \|\Pomega(\X - \Y)\|_F \leq \delta.
\end{equation}

\section{Methodology}\label{methodology}
In this section, we propose the isometric delay-embedding transform  and the temporal isometric  delay-embedding transform. Subsequently, we establish the LRTC-TIDT model.

\subsection{Isometric Delay-embedding Transform}
Unlike the standard delay‐embedding transform, we employ a modified version with desirable properties, which allows the univariate time series
$ \bm{m} \in \mathbb{R}^{t}$ to be transformed into a Hankel matrix with size $ t \times k$   in the following form
\begin{equation}
\mathcal{H}_k(\bm{m})=\frac{1}{\sqrt{k}}
\left[\begin{array}{cccc}
m_1 & m_2 & \cdots & m_{k} \\
m_2 & m_3 & \cdots & m_{k+1} \\
\vdots & \vdots & \vdots & \vdots \\
m_t & m_{1} & \cdots & m_{k-1}
\end{array}\right],
\end{equation}
where $k$ is the coefficient that determines the number of columns in the Hankel matrix. 
In fact, the resulting matrix corresponds to the first $k$ columns of a anti-circulant Hankel matrix \cite{karner2003spectral}.
Since the delay‐embedding operator
$\mathcal{H}_k:\mathbb{R}^t\to\mathbb{R}^{t\times k}$
satisfies
\begin{equation}
\|\bm{x}-\bm{y}\|_2
=\|\mathcal{H}_k(\bm{x})-\mathcal{H}_k(\bm{y})\|_F,
  \forall \bm{x},\bm{y}\in\mathbb{R}^t,
\end{equation}
it is an isometry and is therefore referred to as the isometric delay‐embedding transform.
We now invoke the following two lemmas to articulate a fundamental insight: for a univariate time‐series vector $\bm{m}$, the low‐rankness of its Hankel matrix $\mathcal{H}_k(\bm{m})$ can derives from the smoothness and periodicity of the underlying data $\bm{m}$.
\begin{lemmaa}\label{smoothness}
For a univariate time series vector $\bm{m}=[m_1,m_2,\cdots,m_t]^{\mathrm{T}}\in\mathbb{R}^t$, its smoothness can be characterized by: $\eta(\bm{m})=\|\bm{m}-\mathcal{S}(\bm{m})\|_2$, where  $\mathcal{S}(\bm{m})=[m_2,m_3,\cdots,m_t,m_1]^{\mathrm{T}}$.
For the transformed Hankel matrix $\mathcal{H}_{k}(\bm{x})\in\mathbb{R}^{t \times k}$, its rank-$r$ approximate error is denoted as
\[\epsilon_r(\mathcal{H}_{k}(\bm{m}))=\min_{Z}\|\mathcal{H}_k(\bm{m})-Z\|_F,\ \   \text{s.t.} \ \  \text{rank}(Z)\leq r,\]
where $Z\in\mathbb{R}^{t \times k}$.
Then we have
\begin{equation}
\epsilon_r(\mathcal{H}_{k}(\bm{m}))\leq  \sqrt{\frac{k-r}{3k}}\left\lceil\frac{k}{r}\right\rceil\eta(\bm{m}).
\end{equation}
\end{lemmaa}
\begin{proof}
The proof is provided in Appendix \ref{appedix:lemma4.1}.
\end{proof}

\begin{lemmaa}\label{periodicity}
For a univariate time series vector $\bm{m}=[m_1,m_2,\cdots,m_t]^{\mathrm{T}}\in\mathbb{R}^t$, its periodicity can be characterized by: $\beta_{\tau}(\bm{m})=\|\bm{m}-\mathcal{N}(\bm{m})\|_2$,
where $\mathcal{N}(\bm{m}) = [m_{1+\tau}, \cdots, m_{t+\tau}]^{\mathrm{T}}$,
assuming that $\tau$ is the period, and $m_i = m_{i-t}$ for $i > t$.
 For the transformed Hankel matrix $\mathcal{H}_{k}(\bm{m})\in\mathbb{R}^{t \times k}$,
its  rank-$r$ approximate error satisfies
$$
\epsilon_r(\mathcal{H}_{k}(\bm{m}))\leq \frac{\tau}{\sqrt{k}}(\left\lceil\frac{k}{\tau}\right\rceil-1)\beta_{\tau}(\bm{m}).
$$
\end{lemmaa}
\begin{proof}
The proof is provided in Appendix \ref{appedix:lemma4.2}.
\end{proof}
Here, $\lceil \cdot \rceil$ in Lemmas \ref{smoothness} and \ref{periodicity} represents the ceiling operation.  
Lemmas \ref{smoothness} and \ref{periodicity} demonstrate that as the periodicity or smoothness of the original univariate time series vector $\bm{m}$ increases along the temporal dimension, the low-rank property of its Hankel matrix $\mathcal{H}_{k}(\bm{m})$ is correspondingly enhanced.   Specifically, strong smoothness in $\bm{m}$ results in a very small $\eta(\bm{m})$, which in turn yields a small $\varepsilon_r(\mathcal{H}_{k}(\bm{m}))$, meaning $\mathcal{H}_{k}({\bm{m}})$ is very close to a rank-$r$ matrix. Thus, temporal smoothness in $\bm{m}$ can induce a low-rank structure in $\mathcal{H}_{k}({\bm{m}})$. The same argument holds for periodicity.
The study of the Hankel low-rankness mechanism provides a reliable guarantee for its application, which is lacking in other Hankel tensor completion methods 
\cite{yokota2018missing,sedighin2020matrix,wang2023low,yamamoto2022fast,gillard2018structured,butcher2017simple,zhang2019correction}.

For multidimensional time series, their smoothness and periodicity are typically manifested only in the temporal dimension, as observed in network traffic data, power consumption data, taxi pickup/dropoff data. Consequently, for multivariate/multidimensional time series, we have designed the following temporal isometric delay-embedding transform.

\subsection{Temporal Isometric Delay-embedding Transform}
For a multidimensional time series $\M \in \mathbb{R}^{t \times n_1 \times \cdots \times n_p}$, temporal isometric delay-embedding transform is denoted as $\mathcal{H}_k(\cdot)$, also referred to as the temporal Hankel transform. The transformed tensor is called the temporal Hankel tensor $\mathcal{H}_k(\M) \in \mathbb{R}^{t \times k \times n_1 \times \cdots \times n_p}$, and it can be expressed in the following form:
\begin{equation}
\mathcal{H}_k(\M) = \frac{1}{\sqrt{k}}\left[\begin{array}{cccc}
\M_1 & \M_2 & \cdots & \M_k \\
\M_2 & \M_3 & \cdots & \M_{k+1} \\
\vdots & \vdots & \vdots & \vdots \\
\M_t & \M_1 & \cdots & \M_{k-1}
\end{array}\right],
\end{equation}
where $\M_i \in \mathbb{R}^{1 \times 1 \times n_1 \times \cdots \times n_p}$ represents the information of the time series $\M$ at the $i$-th time sampling point. In other words, $\M_i = reshape(\M(i, :, :, \dots), 1, 1, n_1, \dots, n_p)$.
It is readily apparent that  each face slice of the tensor $\mathcal{H}_k(\M)$  is a Hankel matrix of size $t \times k$.

Inspired by Lemmas \ref{smoothness} and \ref{periodicity}, the low-rank property of the constructed temporal Hankel tensor can be attributed to the smoothness and periodicity of the original time series. 
A formal exposition of this claim are provided in  Appendix \ref{appde:b}.
To offer an intuitive illustration, we consider a multivariate time series exhibiting temporal smoothness, as shown in Fig.\ref{fig:PS-TCPS}(a,b). Applying the temporal Hankel transform to the time series matrix yields the Hankel tensor depicted in Fig.\ref{fig:PS-TCPS}(c). As observed in Fig.\ref{fig:PS-TCPS}(d), the singular values of the Hankel tensor decay rapidly, indicating a pronounced low-rank structure. This numerically confirms that the smoothness of the original time series indeed induces the low-rankness of the temporal Hankel tensor, thereby supporting our claim.

\begin{figure}[!ht]
\centering
\vspace{-0.2cm}
\includegraphics[width=1\linewidth]{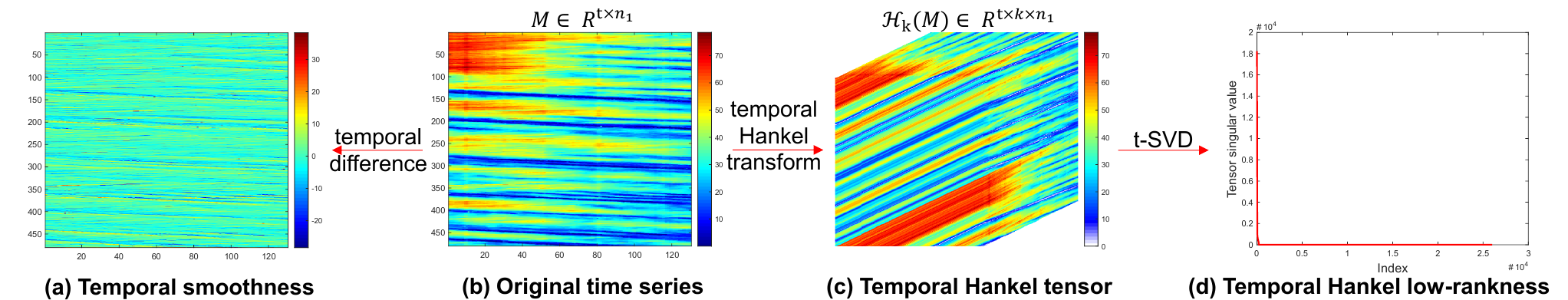}
\caption{Illustrations of temporal Hankel low-rankness from smoothness using  real-world data.(a): Difference diagram of the original time series in the time dimension, the phenomenon that the values in the difference graph are concentrated at 0 indicates that this time series is smooth along the time dimension; (b): multivariate time series from traffic data;  (c): the temporal Hankel tensor, which is derived from the temporal Hankel transform applied to the
real-world data; (d): the corresponding curves of tensor singular values of the temporal Hankel tensor. }\label{fig:PS-TCPS}
\vspace{-0.3cm}
\end{figure}

\begin{remark}
    Compared with the \textit{Multiway Delay-embedding Transform}(MDT),  our proposed TIDT has three key advantages.
1. Origin of Hankel Low-Rankness:
   the source of the Hankel tensor’s low-rank property induced by TIDT is explicitly interpretable, providing a solid theoretical foundation for low-rank Hankel tensor completion—something that MDT lacks.
2. Hankel Consistency:
   MDT fails to preserve the Frobenius norm due to
   $
      \normF{\M} \cancel{\propto} \normF{\mathcal{D}_\tau(\M)},  
   $
   whereas our method ensures consistency between the original data and its Hankelization, i.e.,
   $\normF{\M} = \normF{\mathcal{H}_k(\M)}.$
3. Computational Efficiency:
   MDT transforms a ($p+1$)-order tensor into a ($2p+2$)-order one, while TIDT only increases it to ($p+2$)-order.
   The resulting smaller tensor size significantly reduces computational cost in subsequent applications.
\end{remark}

\subsection{Non-Random Missing Multidimensional Time Series Recovery}
To gain insights into the recovery of non-random missing patterns, we first examine the lower-order case of tensor (multidimensional time series), namely matrix completion. For matrices, exact recovery under non-random missing patterns generally requires that each row and column contain a sufficient number of observed entries \cite{liu2019matrix}. Consequently, recovering matrices with entirely missing rows or columns is regarded as one of the most challenging problems \cite{liu2022recovery}. In such cases, even the matrix nuclear norm fails to recover from fully missing rows or columns \cite{candes2009exact}.

As noted in Remark \ref{tensor-matrix}, under the t-SVD framework, the rows and columns of matrices correspond to the horizontal and lateral slices of tensors. Analogous to the matrix case, under non-random sampling patterns, the minimum number of horizontal and lateral slice samples is crucial for recovery \cite{Shu2025Guaranteed}, which makes the recovery of tensors with entirely missing horizontal or lateral slices extremely challenging.
The tensor nuclear norm fails to handle this segment missing  scenario, as illustrated in Fig.\ref{fig:temporal hankel}(a,d). Fortunately, the aforementioned temporal isometric delay-embedding transform can mitigate this issue by reordering and restructuring the data.

To clearly illustrate the  role of the temporal isometric delay-embedding transform for addressing non-random missing patterns in multidimensional time series, we first introduce the concept of temporal Hankel sampling set in the following.
\begin{defn}[Temporal Hankel sampling set]\label{Temporal Hankel Sampling Set} 
For a non-random sampling set $\Omega \in  \left[t \right] \otimes \left[n_1\right]  \otimes  \cdots \otimes \left[n_p\right] $, its temporal Hankel sampling set associated with  scale coefficient $k$ is denoted by $\Omega_{\mathcal{H}}$ and given by
\begin{equation}
\bar{\Omega}_{\mathcal{H}}=\mathcal{H}_k(\bar{\Omega}) \text { and } \Omega_{\mathcal{H}}=\operatorname{supp}(\bar{\Omega}_{\mathcal{H}}),
\end{equation}
where $\bar{\Omega} \in \mathbb{R}^{t \times n_1 \times \cdots \times n_p}$  is the original sampling tensor (the mask tensor 
of $\Omega$) and $\bar{\Omega}_{\mathcal{H}} \in \mathbb{R}^{t \times k\times n_1 \times \cdots \times n_p}$ is temporal Hankel sampling tensor (the mask tensor 
of $\Omega_{\mathcal{H}}$). Note that
the consistency between sampling after Hankelization and
Hankelization after sampling: $\mathcal{H}_k (\mathcal{P}_{\Omega}(\M))$ $=\mathcal{P}_{\Omega_{\mathcal{H}}} (\mathcal{H}_k(\M))$,
where $\mathcal{P}_{\Omega_{\mathcal{H}}}(\Z)=\bar{\Omega}_{\mathcal{H}} \circ \Z, \forall \Z \in \mathbb{R}^{t \times k \times  n_1 \times \cdots \times n_p }$.
\end{defn}
Regardless of how the observed entries are selected, the temporal Hankel sampling set $\Omega_{\mathcal{H}}$ 
consistently exhibits a well-posed pattern, as illustrated in Fig.\ref{fig:temporal hankel}(b). In other words, even if the original sampling set contains entirely missing horizontal or lateral slices, each horizontal or lateral slice in the temporal Hankel sampling set still contains a sufficient number of observed entries. 
Therefore, performing the temporal Hankel  transform on the original data provides a feasible path for the  non-random missing time series recovery problem and  holds the promise of achieving a theoretically precise recovery.

\begin{figure}[t]
\centering
\vspace{-0.2cm}
\includegraphics[width=0.6\linewidth]{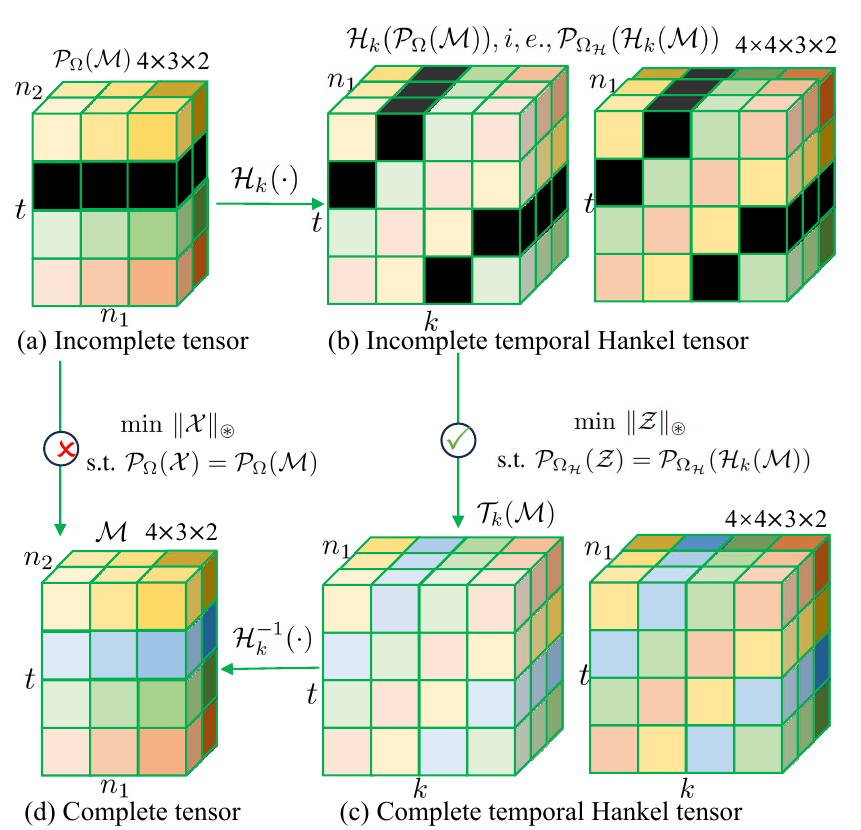}
\vspace{-0.3cm}
\caption{Illustrations of temporal Hankel  tensor nuclear norm
for multidimensional time series prediction. (a): incomplete
tensor, where black squares represent 0 values (not observed)
and other color blocks represent non-zero values (observed); (b):
incomplete temporal Hankel tensor; (c): complete temporal
Hankel tensor; (d): complete tensor. This illustrates that
introducing temporal Hankel transform helps distribute unsampled
regions, ensuring that each horizontal/lateral subtensor contains
a sufficient number of sampled elements.}\label{fig:temporal hankel}
\vspace{-0.4cm}
\end{figure}

\subsection{ The LRTC-TIDT Model}
The smoothness and periodicity of the original time series tensor $\M \in \mathbb{R}^{t \times n_1 \times \cdots \times n_p}$ induce a low-rank structure in its temporal Hankel tensor $\mathcal{H}_k(\M) \in \mathbb{R}^{t \times k \times n_1 \times \cdots \times n_p}$. To exploit this property, we adopt the tensor nuclear norm as a convex surrogate of the tensor rank, leading to the following completion problem:
\begin{equation}\label{DTC}
    \begin{aligned}
    \min_{\Z\in \mathbb{R}^{t \times k \times n_1 \times\cdots\times n_p}} ~\norm{\Z}_{\circledast},
    ~\text { s.t.} ~ \Pomegah(\Z)= \Pomegah(\mathcal{H}_k(\M)). \\
    \end{aligned}
\end{equation}
Let $\hat{\Z}$ denote an estimate of $\mathcal{H}_k(\M)$. The final prediction of $\M$ is then obtained through inverse temporal Hankelization, $\mathcal{H}_k^{-1}(\hat{\Z})$.
As illustrated in Fig.\ref{fig:temporal hankel}, the overall recovery process consists of three sequential steps:
1. Temporal Hankel transform : Transform the incomplete tensor $\Pomega(\M)$ into the incomplete temporal Hankel tensor $\mathcal{H}_k (\mathcal{P}_{\Omega}(\M)),i,e.,\Pomegah(\mathcal{H}_k(\M))$;
2. Tensor completion: Solve Eq.~(\ref{DTC}) to recover the full temporal Hankel tensor $\mathcal{H}_k(\M)$;
3. Inverse transform: Apply $\mathcal{H}_k^{-1}(\cdot)$ to reconstruct the complete tensor $\M$.
Integrating these steps into a single optimization yields the \textit{Low-
Rank Tensor Completion with Temporal Isometric
Delay-embedding Transform} (LRTC-TIDT) model for non-random missing  time series recovery:
\begin{equation}\label{LRTC-TIDT}
\min_{\X\in \mathbb{R}^{t \times n_1 \times \cdots \times n_p}} \norm{\mathcal{H}_k(\X)}_{\circledast}, \quad
\text{s.t.} \ \Pomega(\X) = \Pomega(\M),
\end{equation}
where $\Omega \in  \left[t \right] \otimes \left[n_1\right]  \otimes  \cdots \otimes \left[n_p\right] $ is a non-random sampling set.
When the observations are contaminated by additive noise, the LRTC-TIDT model can be formulated as
\begin{equation}\label{LRTC-TIDT_noise}
\min_{\X \in \mathbb{R}^{t \times n_1 \times \cdots \times n_p}} \norm{\mathcal{H}_k(\X)}_{\circledast}, \quad
\text{s.t.} \quad \norm{\Pomega(\X - \Y)}_F \leq \delta,
\end{equation}
where the observed tensor is given by $\Y = \Pomega(\M + \E)$,  $\E$ representing the noise term and $\delta$ controlling the noise level, assuming $\norm{\Pomega(\M - \Y)}_F \leq \delta$. Note that when $\delta=0$, model (\ref{LRTC-TIDT_noise}) reduces to model (\ref{LRTC-TIDT}). Therefore, in both the algorithm and experiments, we only use model (\ref{LRTC-TIDT_noise}) to evaluate the recovery performance of the LRTC-TIDT model.

\section{Recovery Theory}\label{Non-random Missing Time Series Imputation Theory}	
In this section, we establish the recovery theory for the LRTC-TIDT model under non-random missingness by introducing key concepts, including the temporal time sampling rate and  the temporal Hankel tensor incoherence.

\subsection{Minimum Temporal Sampling Rate}
In tensor completion tasks with random sampling, it is commonly assumed that the data sampling adheres to a particular distribution, such as the Bernoulli distribution. Under this assumption, the sampling probability 
 naturally serves as a metric for sampling \cite{candes2009exact} \cite{lu2019tensor} \cite{qin2022low}. Nonetheless, this method proves inadequate when dealing with non-random sampling scenarios.
Inspired by the concept of the minimum row/column sampling ratio for matrices under non-random sampling \cite{liu2019matrix}, we introduce a new sampling metric for time series, termed the minimum temporal sampling rate $\rho(\Omega)$. Since this metric is independent of the underlying distribution, it is particularly suitable for scenarios with non-random missing data.

\begin{defn}[Temporal sampling number] 
For any fixed sampling set $\Omega \subseteq \left[t\right] \otimes\left[n_1\right] \otimes  \cdots \otimes \left[n_p\right] $, its $(i_1,i_2,\cdot,i_p)$-th  temporal sampling number is defined as
\begin{equation}\left|\Omega{(i_1,i_2,\cdot,i_p)}\right|:=\sharp\{i_t| (i_t,i_1,i_2,\cdots,i_p)\in \Omega\},
\end{equation}
where $\sharp$ denotes the cardinality of a set.
\end{defn}

\begin{defn}[Minimum temporal sampling rate] For any fixed sampling set $\Omega \subseteq \left[t\right] \otimes\left[n_1\right] \otimes  \cdots \otimes \left[n_p\right] $, its minimum temporal sampling rate is defined as the smallest fraction of sampled entries in each time series vector; namely,
\begin{equation}\label{samplingrate}
\rho (\Omega) =\min _{1 \leq i_j \leq n_j} \frac{\left|\Omega{(i_1,i_2,\cdot,i_p)}\right|}{t}, \text{where}  \ j=1,2,\cdots,p.
\end{equation}

\label{temporal sampling rate}
\end{defn}
Note that the prediction problem can be regarded as a special case of deterministic completion, wherein $\rho(\Omega) = (t-h)/t$ and $h$ represents the forecast horizon. This serves to underscore the practical value of the sampling metric.

Importantly, independent of the choice of observed entries, as long as the original tensor $ \M $ has a sufficiently large minimum temporal sampling rate $\rho(\Omega)$, each horizontal or lateral slice in the temporal Hankel sampling set $\Omega_{\mathcal{H}}$ will contain enough observed entries. For instance, consider an observed tensor $ \Pomega(\M)\in \mathbb{R}^{t \times n_1 \times \cdots \times n_p}$ with a minimum temporal sampling rate of $\rho(\Omega)$.
Regardless of the original sampling set $\Omega$, when the temporal Hankel transform is applied with $k = t$, each horizontal or lateral slice in the temporal Hankel sampling set $\Omega_{\mathcal{H}}$ contains at least $t n_1 n_2 \cdots n_p\rho(\Omega)$ observed entries, making successful recovery possible.

\subsection{Temporal Hankel Tensor Incoherence}
Inspired by the incoherence conditions in the classic low-rank recovery theoretical framework \cite{candes2009exact,zhang2016exact,qin2022low,wang2023guaranteed,peng2022exact}, we define the temporal Hankel tensor incoherence conditions to avoid pathological situations in the temporal Hankel tensor.
Note that the  temporal Hankel tensor incoherence  applies  tensor incoherence condition in \cite{qin2022low} to the temporal Hankel tensor, rather than to the original tensor.

\begin{defn} [Temporal Hankel  tensor incoherence]
For $\M \in\mathbb{R}^{ t\times n_1\times\cdots\times n_p}$, assume that the temporal Hankel  tensor $\mathcal{H}_k(\M)\in\mathbb{R}^{ t \times k\times n_1\times\cdots\times n_p}$   with t-SVD rank $r$ and it has  the skinny t-SVD $\mathcal{H}_k(\M)=\U * \mathcal{S} *\V^\mathrm{T}$, and then $\M$ is said to satisfy the temporal Hankel  tensor incoherence conditions with size coefficient $k$ and parameter $\mu>0$ if
\begin{equation}\label{incoherence_hankel}
\max_{i_t = 1,\cdots,t} \|\U^\mathrm{T}*\mathring{\mathfrak{e}}_t^{(i_t)}\|_\mathrm{F}\leq
\sqrt{\frac{\mu r}{t n}}, \
\max_{i_k = 1,\cdots,k} \|\V^\mathrm{T}*\mathring{\mathfrak{e}}_k^{(i_k)}\|_\mathrm{F}\leq
\sqrt{\frac{\mu r}{k n}},
\end{equation}
where $n =n_1 \times  \cdots \times n_p$, $\mathring{\mathfrak{e}}_t^{(i_t)}$ is the order-$p+2$ tensor sized $t\times1\times n_1\times\cdots\times n_p$, whose $(i_t,1,i_1,\cdots,i_p)$-th entry equals 1 and the rest equal 0, and $\mathring{\mathfrak{e}}_k^{(i_k)}:=(\mathring{\mathfrak{e}}_t^{(i_t)})^\mathrm{T}$.
\end{defn}

\subsection{Main Results}
\begin{theorem}[Exact Recovery Theory] \label{thm: exact non-random tensor completion}
Suppose that  $\M \in \mathbb{R}^{t\times n_1\times\cdots\times n_p}$  obeys the  temporal Hankel tensor incoherence conditions   (\ref{incoherence_hankel}) with scale factor k and $\Omega \in \left[t\right] \otimes \left[n_1\right]  \otimes  \cdots \otimes \left[n_p\right] $. If
\begin{equation}\label{sampling-exact-a0.24}
\rho(\Omega) > 1-\frac{ k }{2 \mu r (r_s+1)t},
\end{equation}
where  $\rho(\Omega)$ is minimum temporal sampling rate, $r$ is the t-SVD rank of tensor $\mathcal{H}_k(\M)$, also referred to as the temporal Hankel rank, $r_s$ is the multi-rank sum of tensor $\mathcal{H}_k(\M)$  and  $\mu$ is the parameter of  temporal Hankel tensor incoherence conditions,
then $\M$ is the unique solution to  LRTC-TIDT model  (\ref{LRTC-TIDT}) in the noise-free setting.
\end{theorem}
\begin{proof}
The proof of Theorem \ref{thm: exact non-random tensor completion} follows directly from that of Theorem \ref{thm: approximate non-random tensor completion} as a degenerate case in the noise-free setting.
\end{proof}
 
\begin{theorem}\label{thm: approximate non-random tensor completion}
Suppose that  $\M \in \mathbb{R}^{t\times n_1\times\cdots\times n_p}$  obeys the  temporal Hankel tensor incoherence conditions   (\ref{incoherence_hankel}) with scale factor k and $\Omega \in \left[t\right] \otimes \left[n_1\right]  \otimes  \cdots \otimes \left[n_p\right] $. If
\begin{equation}\label{sampling-pr-a0.24}
\rho (\Omega) > 1-\frac{\alpha k }{2 \mu r (r_s+1)t},
\end{equation}
let $\hat{\M}$ denote the solution to the LRTC-TIDT model (\ref{LRTC-TIDT_noise}) under noisy observations  and $\norm{\Pomega(\M-\Y)}_F \leq \delta$,
then  $\hat{\M}$  obeys 
\begin{equation}
\norm{\hat{\M}-\M}_F \leq  (2 \frac{\sqrt{r_s+1-\alpha}+\sqrt{r_s+1}}{\sqrt{r_s+1-\alpha}-\sqrt{\alpha r_s}}\sqrt{nt}+2) \delta,
\end{equation}
where $\alpha \in (0,1)$ is a number that controls the relationship between sampling and  recovery  accuracy, $n=n_1 \times  \cdots \times n_p$.
Here, $r$, $r_s$, and $\mu$ follow the same definitions as in Theorem~\ref{thm: exact non-random tensor completion}.
\end{theorem}
\begin{proof}
 The proof is provided in Appendix \ref{appendix:c}.   
\end{proof}

Theorems \ref{thm: exact non-random tensor completion} and \ref{thm: approximate non-random tensor completion} establish the recovery guarantees of the LTRC-TIDT model in the noiseless and noisy settings, respectively. These results show that, as long as the condition
$
\rho(\Omega) > k/{2 \mu r (r_s+1)t}
$
is satisfied, the LRTC-TIDT model achieves exact recovery in the noiseless scenario.
Moreover, in noisy scenarios, as the noise level $\E$ decreases, the recovered result $\hat{\M}$ becomes increasingly closer to the ground truth $\M$; in other words, even when a portion of the observed entries is perturbed by mild noise, the recovery remains accurate.
It is worth emphasizing that the sampling condition depends on the low-rankness of the temporal Hankel tensor. A stronger temporal Hankel low-rankness allows exact recovery to be achieved even under a lower minimum temporal sampling rate.
Furthermore, if the original time series exhibits pronounced periodicity and smoothness, the temporal Hankel tensor becomes more low-rank, which further aids model (\ref{LRTC-TIDT}) in achieving exact recovery.

As noted in Definition \ref{temporal sampling rate}, the forecasting problem can be regarded as a special case of deterministic completion, where the observation region
$\Omega \in [t] \otimes [n_1] \otimes \cdots \otimes [n_p]$
degenerates into a historical region
$\Omega = [t-h] \otimes [n_1] \otimes \cdots \otimes [n_p]$,
with $h$ representing the forecast horizon. In this case, the non-random exact recovery theory of Theorem  \ref{thm: exact non-random tensor completion} naturally reduces to its exact prediction theory. 
\begin{Prop}\label{prop:pridction}
Suppose that  $ \M \in \mathbb{R}^{t\times n_1\times\cdots\times n_p} $  satisfies the temporal Hankel tensor incoherence conditions (\ref{incoherence_hankel}) with scale factor $k$, and let
$\Omega = [t-h] \otimes [n_1] \otimes \cdots \otimes [n_p]$
denote the historical observation region. If
\begin{equation}
h < \frac{k}{2 \mu r (r_s + 1)},
\end{equation}
where $h$ is the forecasting horizon, 
and $r$, $r_s$, and $\mu$ are specified as in Theorem \ref{thm: exact non-random tensor completion},
then $\M$ is the unique solution to the LRTC-TIDT model (\ref{LRTC-TIDT}) in the noise-free setting.
\end{Prop}
\begin{proof}
  Combined with Theorem \ref{thm: exact non-random tensor completion} and the relation $\rho(\Omega) = (t-h)/t$, Proposition \ref{prop:pridction} can be directly derived.
\end{proof}
Moreover, the proposed deterministic exact recovery theory can be extended to the random sampling case. The following corollary reveals that when the sampling set satisfies $\Omega\sim\operatorname{Ber}(\theta)$ and $\theta\geq  1-k/{ 2 \mu rt (r_s+1)}$, exact recovery holds with high probability.
\begin{Prop}\label{Ber}
Suppose that  $ \M \in \mathbb{R}^{t\times n_1\times\cdots\times n_p} $  satisfies the temporal Hankel tensor incoherence conditions (\ref{incoherence_hankel}) with scale factor $k$ and $\Omega\sim\operatorname{Ber}(\theta)$.  
If
$\theta \geq  1-k/{ 2 \mu rt (r_s+1)}$, then $\M$ is the unique solution to the LRTC-TIDT  model  (\ref{LRTC-TIDT}) with probability at least $1-e^{-2a^2 m_0}$, 
where  $a=p-1+k/{ 2 \mu r t(r_s+1)}, m_0=tkn_1\cdots n_p$.
\end{Prop}
\begin{proof}
 The proof is provided in  Appendix \ref{appendix:d}.  
\end{proof}

\section{Optimization Algorithm}\label{sec:Algorithm}
This section gives the optimization algorithm of the proposed  the LRTC-TIDT model for the time series recovery task.

\subsection{Optimization to LRTC-TIDT}
we shall not try to solve the problem (\ref{LRTC-TIDT_noise}) directly, but instead
consider its equivalent version as in the following:
\begin{equation}
    \begin{aligned}
    &\min_{\X} \   \norm{\mathcal{H}_k(\X)}_{\circledast} +\frac{\lambda}{2} \norm{\Pomega(\X-\Y)}_F^2,\\
    \end{aligned}
    \label{TSP-aa}
\end{equation}
where $\lambda$ is a trade-off  parameter for balancing constraint and regularization. We adopt the famous optimization framework \textit{alternating direction method of multipliers} (ADMM)  to solve the model. First of all, we introduce the auxiliary variable  $\Z=\mathcal{H}_k(\X)$ to separate the temporal Hankel
operation $\mathcal{H}_k(\cdot)$, 
then the LRTC-TIDT model can be rewritten as below:
\begin{equation}
    \min_{\X,\Z} \   \norm{\Z}_{\circledast} +\frac{\lambda}{2} \norm{\Pomega(\X-\Y)}_F^2, ~ ~ \text { s.t.}~ \Z=\mathcal{H}_k(\X). 
    \label{tctnn_aau}
\end{equation}
For model \eqref{tctnn_aau}, we write its  augmented Lagrangian function  
\begin{equation} \label{tctnn_alf}
	L(\X,\Z,\N)  = 
  \norm{\Z}_{\circledast}+ 
 \frac{\mu_j}{2}\normlarge{\Z-\mathcal{H}_k(\X)+{\N}/{\mu_j}}_F^2 +\frac{\lambda}{2} \norm{\Pomega(\X-\Y)}_F^2+\C,
\end{equation}
where  $\N$ is Lagrange multiplier, $\mu_j>0$ is a positive scalar, and $\C$ is only the multiplier dependent squared items.
According to the augmented Lagrangian function, the ADMM optimization framework transforms the minimization problem  \eqref{tctnn_alf} into the following subproblems for each variable in an iterative manner, i.e.,
\begin{align}
\label{subba1}\Z^{j+1}&=\operatorname{arg}\min_{\Z} ~L(\X^{j},\Z,\N^j), \\
\label{subba2}\X^{j+1}=\operatorname{arg}\min_{\X} ~L(\X,&\Z^{j+1},\N^j),~~   \text{s.t.}~ \Pomega(\X^{j+1})= \Pomega(\Y),
\end{align}
and the multiplier is updated by
\begin{align}
\N^{j+1}&=\N^{j}+\mu_j(\Z^{j+1}-\mathcal{H}_k(\X^{j+1})),\label{subba4} 
\end{align}
where $j$ denotes the count of iteration in the ADMM. In the following, we deduce the solutions for \eqref{subba1} and \eqref{subba2} respectively, each of which has the closed-form solution.

\textit{1) Updating $\Z^{j+1} $}: Fixed other variables in \eqref{subba1}, then it degenerates into the following  tensor nuclear norm minimization with respect to $\Z$, i.e.,
\begin{equation}
    \begin{aligned}
   \Z^{j+1}=\operatorname{arg}\min_{\Z} \frac{1}{\mu_j}\norm{\Z}_{\circledast}+
 \frac{1}{2}\normlarge{\Z-(\mathcal{H}_k(\X)-{\N^j}/{\mu_j})}_F^2,
    \end{aligned}\label{TNN-PRO}
\end{equation}
The close-form solution of this sub-problem is given as
\begin{align}\label{aGJ1}
\mathcal{Z}^{j+1} = \operatorname{t-SVT}_{1/\mu_j}(\mathcal{H}_k(\X^{j})-{\N^j}/{\mu_j})
\end{align}
via the order-$p+2$ t-SVT as stated in Theorem \ref{th.2}.

\textit{2) Updating $\X^{j+1}$}:
Taking the derivative in \eqref{subba2} with respect to $\X$, it gets the following linear system :
\begin{equation}\label{aaXJ1}
 (\lambda \Pomega +\mu_j \I) \X =\mathcal{H}_k^{*}(\mu_j \Z^{j+1}+\N^j)+\lambda \Pomega(\Y),
\end{equation}
where  $ \mathcal{H}_k^{*}$ is the Hermitian adjoint of $\mathcal{H}_k$. 
According to the definition of the Hermitian adjoint operator and Hankelization consistency $
\|\M\|_F = \left\|\mathcal{H}_k(\M)\right\|_F$, it is clear that $ \mathcal{H}_k^{*}$=$\mathcal{H}_k^{-1}$.
Then the close-form solution of this sub-problem is given as
\begin{equation}\label{aXJ1}
 \X =  (\mathcal{H}_k^{*}(\mu_j \Z^{j+1}+\N^j)+\lambda \Pomega(\Y))/(\lambda \Pomega +\mu_j \I),
\end{equation}
where the operator $/$ is simply the entry-wise tensor division. 

The whole optimization procedure for solving the proposed LRTC-TIDT model is summarized in Algorithm \ref{alg1}.

\begin{algorithm}[tbp]\vspace{-1mm}
\renewcommand{\algorithmicrequire}{ \textbf{Input}:}
\renewcommand{\algorithmicensure}{ \textbf{Output}:}
\caption{ADMM for solving the LRTC-TIDT  model}\label{alg1}
\begin{algorithmic}[1]
\REQUIRE observed tensor $\Pomega(\Y)$, $k$, $\lambda$.
\STATE Initialize $\X^0=\Pomega(\Y)$, $\G^0=\mathcal{H}_k(\X)$, $\N^0=\0$, $\mu_0=1e-6$.
\STATE \textbf{while} not converge \textbf{do}
\STATE \quad Update $\Z^{j+1}$ by \eqref{aGJ1};
\STATE \quad Update $\X^{j+1}$ by \eqref{aXJ1};
\STATE \quad Update multipliers $\N^{j+1}$  by \eqref{subba4} ;
\STATE \quad Let $\mu_{j+1}=1.1\mu_j$; $j = j +1$.
\STATE \textbf{end while}
\ENSURE imputed tensor $\hat{\X}=\X^{j+1}$.
\end{algorithmic}
\end{algorithm}

\subsection{Computational Complexity Analysis}

For Algorithm 1, the computational complexity in each iteration contains three parts, i.e., steps $3\sim5$. First, the time complexity for order-$d$ t-SVT in step 3 is $O(tk(n_1\cdots n_p)(n_1+\cdots+n_p)+tk^2n_1\cdots n_p)$, corresponding to the linear transform  and the matrix SVD, respectively \cite{qin2022low}. The steps 4 and 5 have the same complexity $O(tkn_1\cdots n_p)$ with only element-wise computation. In all, the pre-iteration computational complexity of Algorithm 1 is $O(tk(n_1\cdots n_p)(n_1+\cdots+n_p)+tk^2n_1\cdots n_p)$.


To demonstrate the computational efficiency of our proposed method, we conduct a time complexity analysis and comparison with two classical tensor completion methods that incorporate structural adjustments.  Considering a ($p$+1)-order multi-dimensional time series $\M \in \mathbb{R}^{n \times n \times \cdots \times n}$, the computational complexity per iteration is $O(n^{3p+3})$ for the CNNM model \cite{liu2022recovery} and $O(n^{2p+4})$  for the MDT-Tucker model \cite{yokota2018missing}. In contrast, our LRTC-TIDT model requires only $O(n^{p+3})$ per iteration. We use randomly generated tensors of size $a \times a \times a ( a = 10, 15, 20, \ldots, 50 )$ to empirically compare the computational time of the three methods. As shown in Fig.\ref{tab:time}, the proposed LRTC-TIDT model exhibits significantly lower computational cost than the two baseline methods, clearly demonstrating its superior efficiency.
\begin{figure}[t]
\centering
\vspace{-0.3cm}
\includegraphics[width=0.6\linewidth]{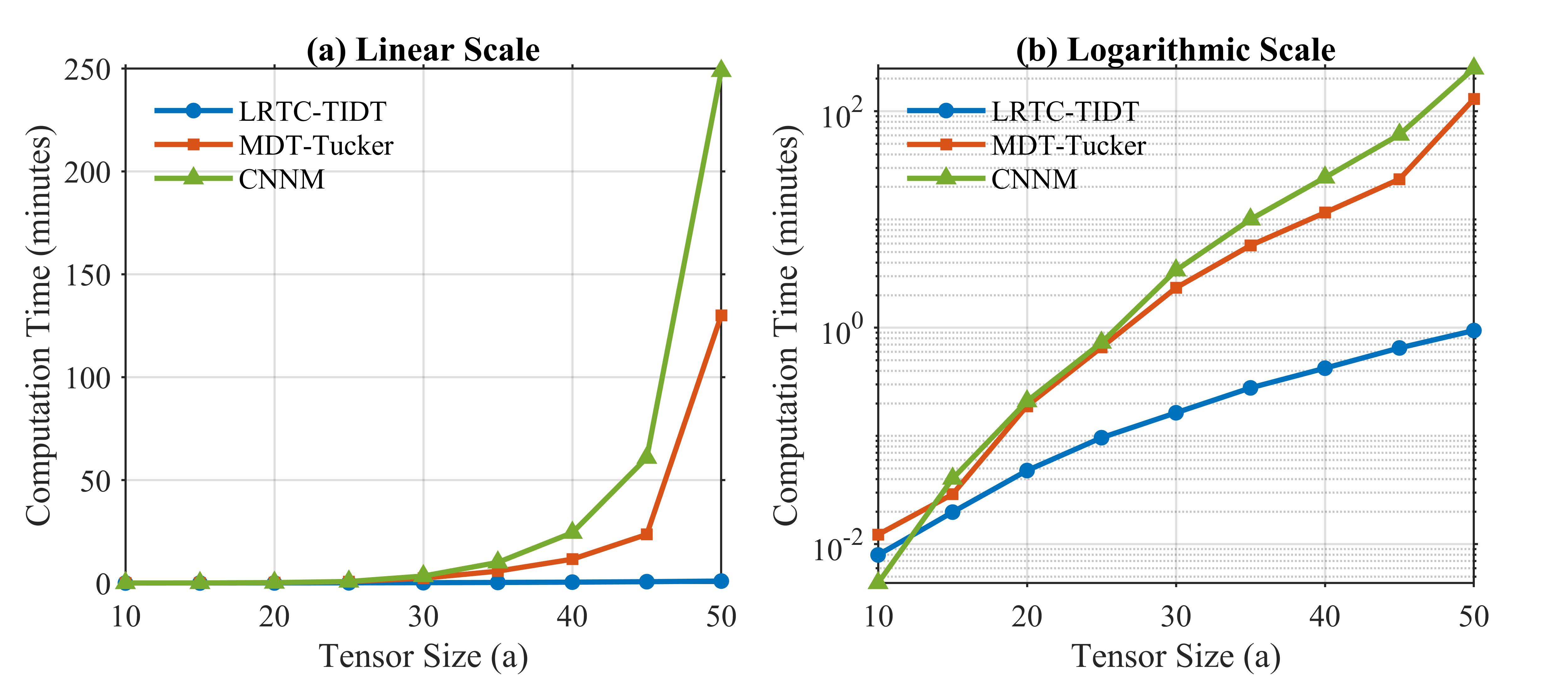}
\vspace{-0.2cm}
\caption{Line plots of computation time on linear and logarithmic scales}\label{tab:time}
\vspace{-0.3cm}
\end{figure}

\section{Experimental Results}\label{sec:experiments}
In this section, we first validate the main theoretical results through experiments on synthetic tensors, and then demonstrate the effectiveness of the proposed method in real-world applications.
The non-random sampling patterns used for testing are illustrated in Fig.\ref{sampling_pattern23}. 

\begin{figure}[!ht]
\centering
\includegraphics[width=0.6\linewidth]{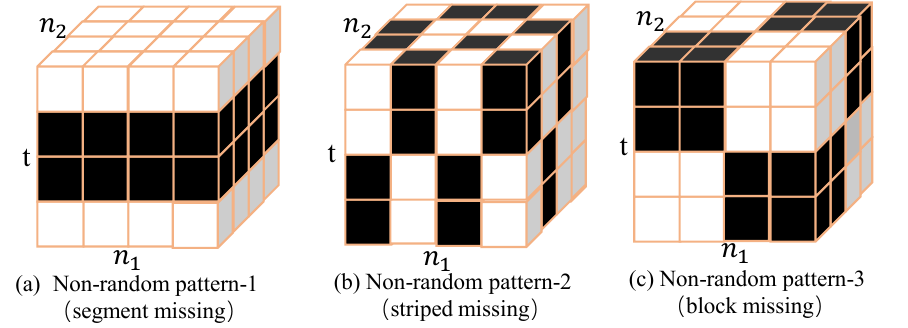}
\vspace{-0.2cm}
\caption{Different non-random sampling patterns in multidimensional time series, where black squares represent 0 values (not observed)
and  white blocks represent 1 values (observed) }\label{sampling_pattern23}
\vspace{-0.3cm}
\end{figure}

\subsection{Simulations for Exact Recovery Theory}

To validate the proposed exact recovery theory (Theorem \ref{thm: exact non-random tensor completion}) for LRTC-TIDT model (\ref{LRTC-TIDT}), we employ the procedure to generate  tensors with prescribed  temporal Hankel tensor rank  via 
$\M(i_t,i_1,i_2)=\sum_{l=1}^{a_{i_1i_2}}sin(\frac{2\pi li_t}{t}),$
where $a_{i_1i_2} \overset{\text{i.i.d.}}{\sim} \mathrm{Unif}\big(\{1, 2, \ldots, a_{\max}\}\big)$ for $i_1\cdot i_2< n^2$ and $a_{nn}=a_{max}$.
The generated tensor $\M \in \mathbb{R}^{t \times n \times n} $ satisfies $r=\operatorname{rank}_{\operatorname{t-SVD}}(\mathcal{H}_t(\M)) \leq 2a_{\max}$, since each frontal slice of $\mathcal{H}_t(\mathcal{M})\in \mathbb{R}^{t \times t \times n \times n} $ is a left circulant matrix \cite{karner2003spectral}  with rank at most $2a_{\max}$ \cite{liu2022recovery}.  
We set $t=n=21$, $r = 2, 4, \ldots, 20$, and $\rho(\Omega) = 1/21, 2/21, \ldots, 20/21$.
For each $(r, \rho(\Omega))$ pair, we simulate 50 test instances and consider a trial successful if the average RMSE is below 0.01.
The phase transition diagrams under three non-random sampling patterns are shown in Fig.\ref{fig.xiangbian}(a–c). The boundary between successful and failed recovery aligns well with the theoretical relationship between the deterministic sampling rate $\rho(\Omega)$ and the temporal Hankel rank $r$. Specifically, a lower temporal Hankel rank $r$ allows for exact recovery with a smaller  $\rho(\Omega)$, which verifies the theoretical recovery guarantee (Theorem \ref{thm: exact non-random tensor completion}). Furthermore, among the three patterns, the exact recovery region of non-random sampling pattern-1 is the smallest, indicating that segment missing poses greater challenges in multidimensional time series recovery. 
In addition, we construct a phase transition diagram under random sampling using a similar strategy (Fig.\ref{fig.xiangbian}(d)), indicating that a lower temporal Hankel rank likewise benefits recovery under random sampling.

\begin{figure}[!htbp]
\renewcommand{\arraystretch}{0.5}
\setlength\tabcolsep{0.5pt}
\centering
\vspace{-0.2cm}
\begin{tabular}{ccccccc}
\centering
\includegraphics[width=32mm, height = 30.3mm]{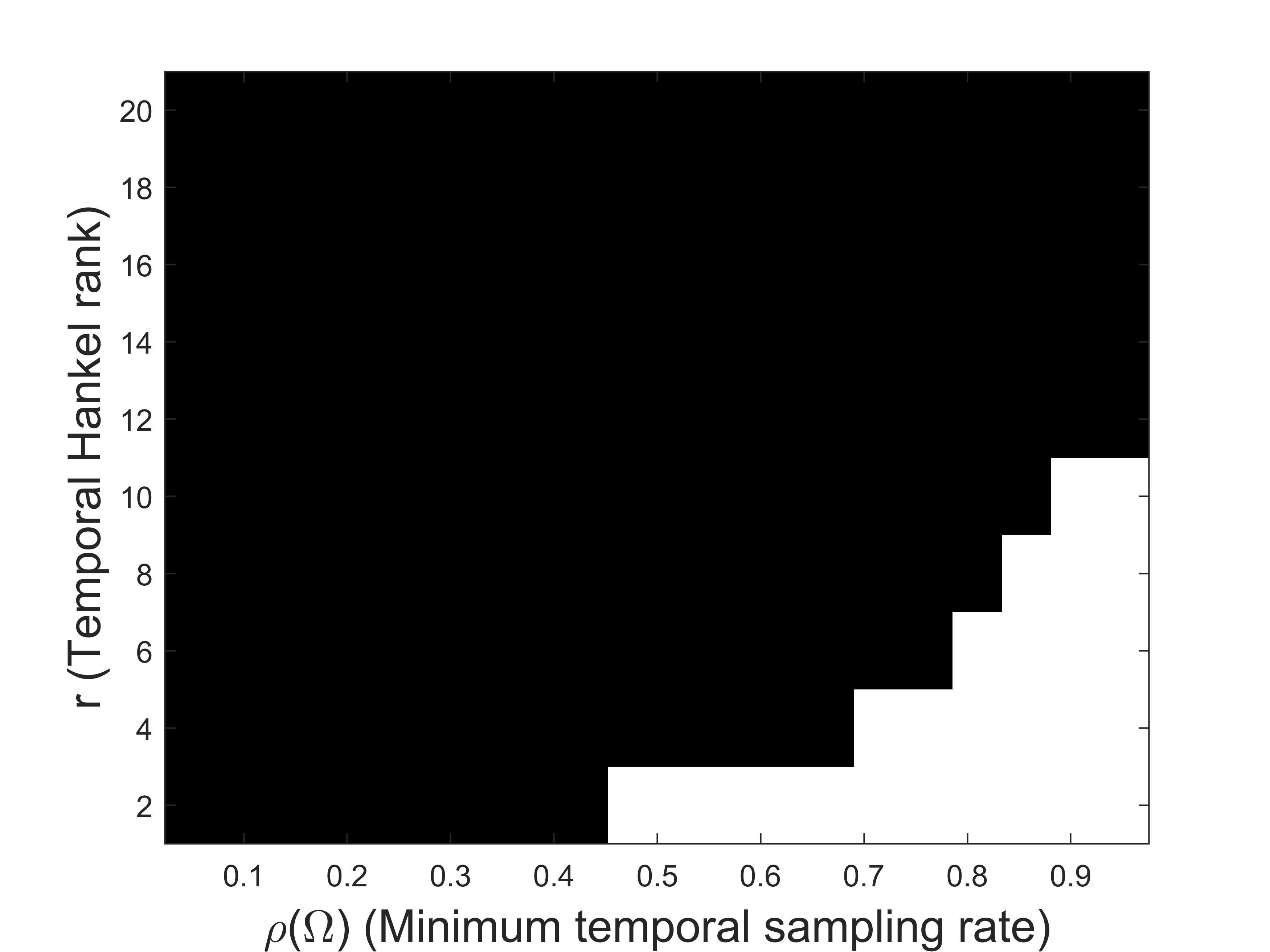}&
\includegraphics[width=32mm, height = 30.3mm]{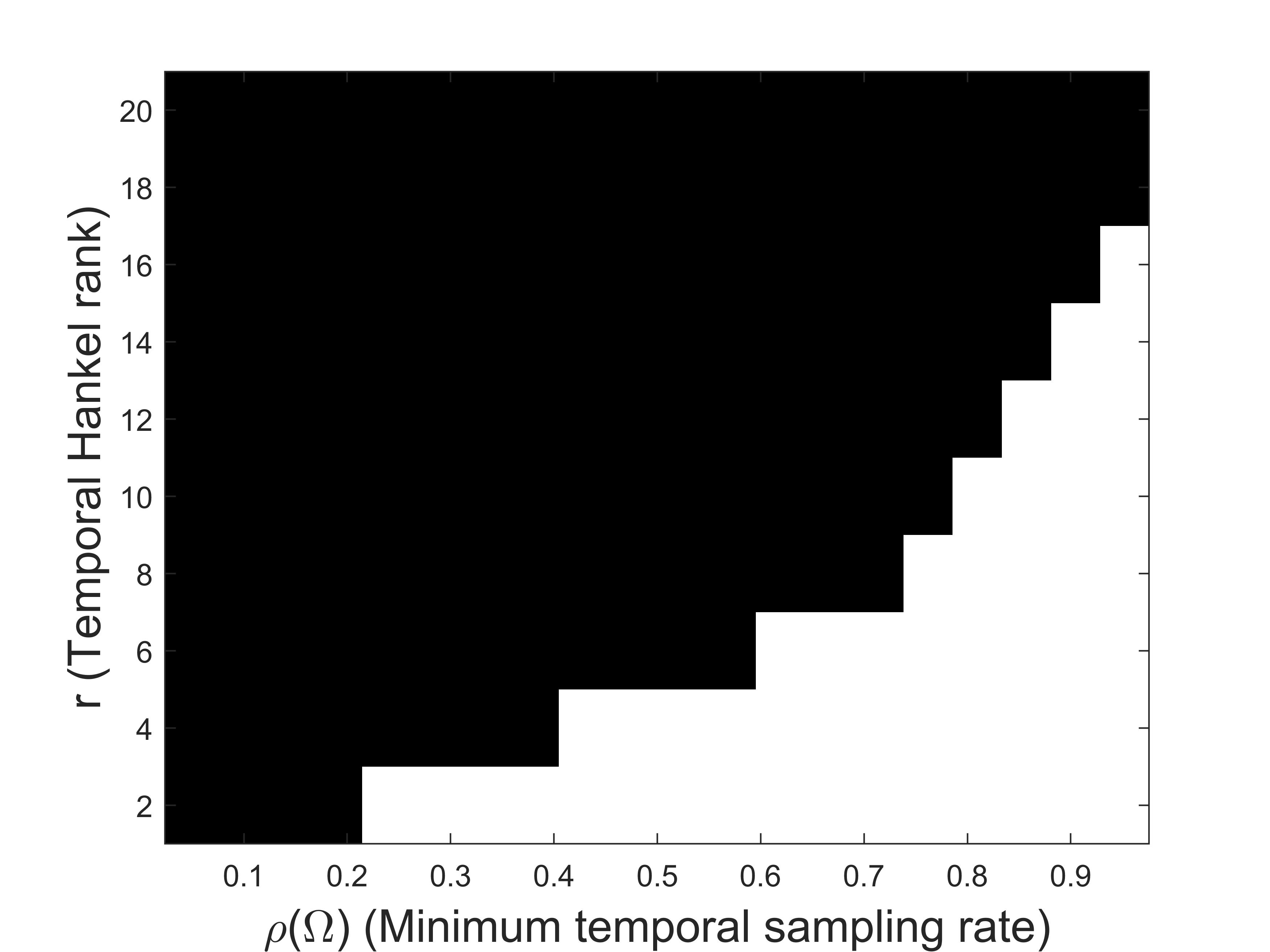}&
\includegraphics[width=32mm, height = 30.3mm]{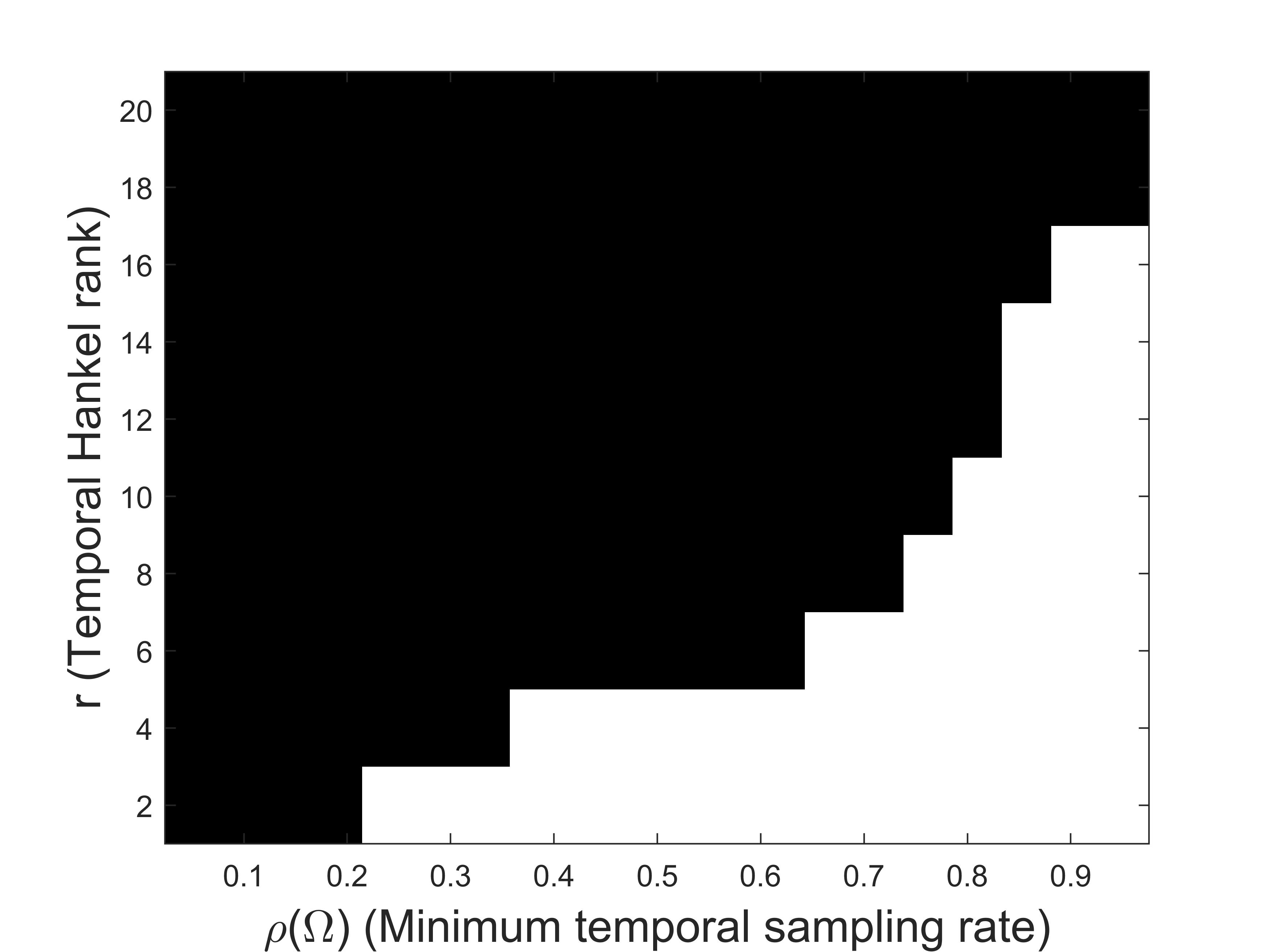}&
\includegraphics[width=32mm, height = 30.3mm]{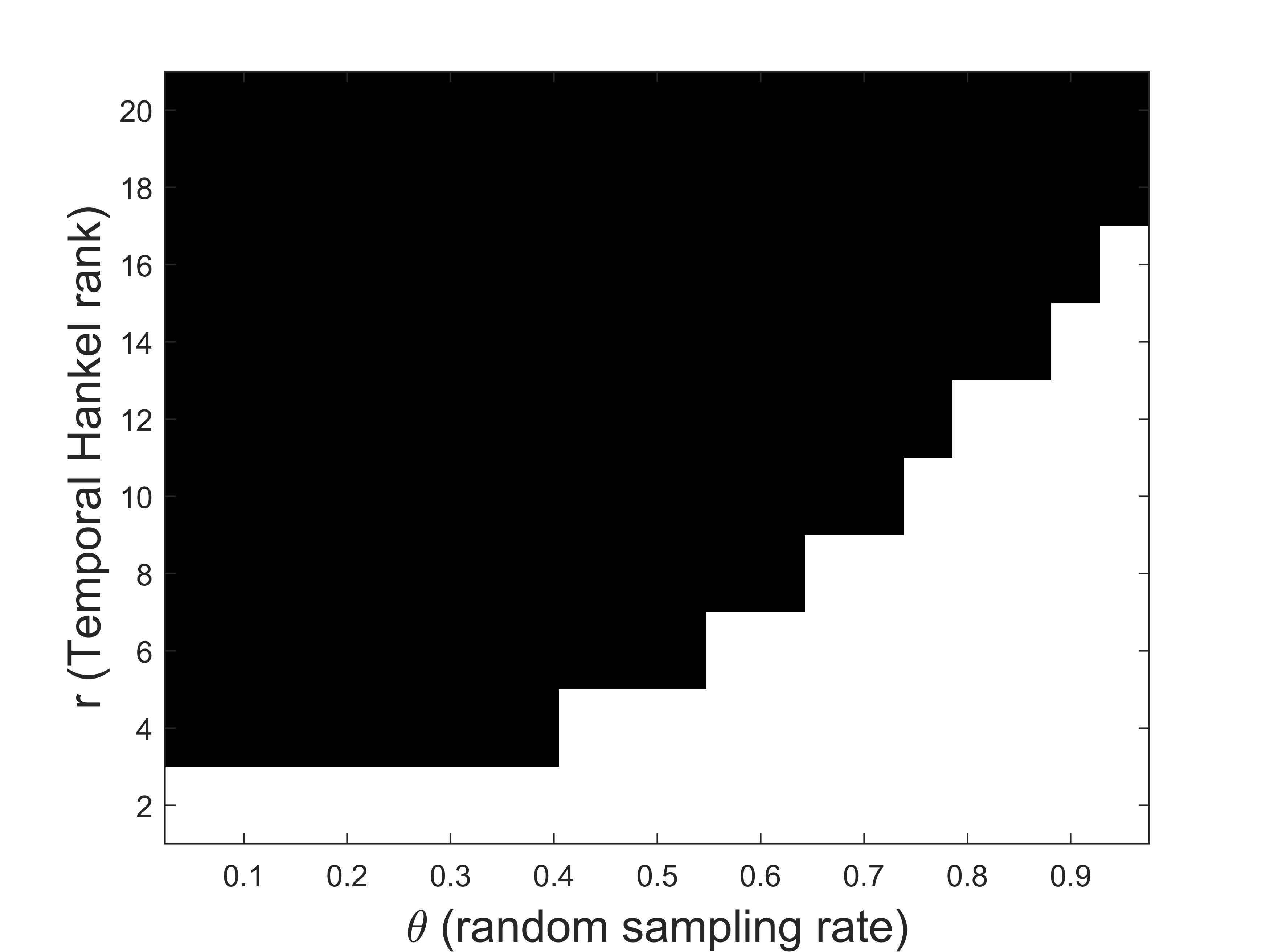}\\

\scriptsize \textbf{(a) Non-random }& \scriptsize \textbf{(b)  Non-random  }&\scriptsize \textbf{(c)  Non-random }& \scriptsize \textbf{(d)  random }\\
\scriptsize \textbf{ pattern-1}& \scriptsize \textbf{ pattern-2}&
\scriptsize \textbf{  pattern-3}& \scriptsize \textbf{pattern}\\
\end{tabular}
\caption{Phase transition diagrams for validating the exact recovery theory. The white and black regions indicate “success” and “failure,” respectively.}\label{fig.xiangbian}
\vspace{-0.7cm}
\end{figure}

\subsection{Baseline Methods}

To evaluate the proposed LRTC-TIDT method, we compare it
with the following baselines for multidimensional time series recovery:
 \textit{High-accuracy Low-Rank Tensor Completion} (HaLRTC)\cite{liu2012tensor},
 \textit{Tensor Nuclear Norm} (TNN)\cite{zhang2016exact},
 \textit{Bayesian CP factorization} (BCPF)\cite{zhao2015bayesian1}, 
 \textit{Tensor Correlated Total Variation} (TCTV)\cite{wang2023guaranteed},
 \textit{Multiway Delay-embedding Transform with  Tucker decomposition} (MDT -Tucker)
\cite{yokota2018missing}, 
and \textit{Convolution Nuclear Norm Minimization} (CNNM) \cite{liu2022recovery}.
Notably, the first three models rely solely on low-rank modeling of multidimensional time series. In contrast, the latter three models take low-rank properties into account within transformed domains. Specifically, TCTV utilizes low-rank characteristics in the gradient domain, MDT-Tucker considers low-rank properties in the delay-embedding transformed domain, and CNNM addresses low-rank features in the convolutional transformed domain.

\subsection{Applications to Network Flow Reconstruction}
\label{Applications to Network Traffic Data}

\begin{table*}[!htbp]
\renewcommand{\arraystretch}{1.4}
\setlength\tabcolsep{3pt} 
\footnotesize
\vspace{-0.3cm}
\caption{Performance comparison (in MAE/RMSE) of LRTC-TIDT and other baselines with different non-random sampling scenarios. The best is highlighted in \textbf{bold} and the second best is highlighted with \underline{underline}.}
 \vspace{-0.2cm}
\label{Abilene_nonoise}
\centering
\begin{tabular}{l|cccc|cc|c} 
     \hline
    \multirow{2}{*}{Method} & 
    \multicolumn{4}{c|}{Pattern-1} & 
    \multicolumn{2}{c|}{Pattern-2} & 
    Pattern-3 \\
    \cline{2-8}
    & 20\% & 40\% & 60\% & 80\% & 30\% & 70\% & 50\% \\
     \hline
     HaLRTC & 2.79/5.43 & 2.67/5.08 & 2.69/5.14 & 2.68/5.07 & 2.57/5.45 & 2.12/3.99 & 2.56/5.50\\
     TNN & 2.79/5.43 & 2.67/5.08 & 2.69/5.14 & 2.68/5.07 & 2.41/5.29 & 1.83/3.57 & 2.27/4.70 \\
     BCPF & 2.79/5.43 & 2.67/5.08 & 2.69/5.14 & 2.68/5.07 & 2.35/5.26 & 1.67/3.50 & 2.15/6.26 \\   
     TCTV & 2.45/4.59 & 2.40/4.10 & 2.35/4.11 & 2.24/3.95 & 2.37/5.02 & 1.74/3.07 & 2.16/4.07 \\    
     MDT-Tucker & 1.04/2.42 & \underline{0.77}/1.71 & \textbf{0.64}/\underline{1.52} & \underline{0.68}/1.41 & \underline{1.77}/\underline{4.36} & 1.18/2.33 & \underline{1.36}/\underline{3.45} \\ 
     CNNM & \underline{1.04}/\underline{2.28} & 0.78/\underline{1.61} & 0.70/1.67 & 0.69/\underline{1.40} & 2.08/4.76 & \underline{0.91}/\underline{1.76} & 1.65/3.91 \\ 
    LRTC-TIDT & \textbf{0.98}/\textbf{1.94} & \textbf{0.73}/\textbf{1.40} & \underline{0.69}/\textbf{1.23} & \textbf{0.68}/\textbf{1.24} & \textbf{1.46}/\textbf{2.99} & \textbf{0.71}/\textbf{1.35} & \textbf{1.13}/\textbf{2.53} \\ 
     \hline
\end{tabular}
\vspace{-0.1cm}
\end{table*}

\begin{figure}[!htbp]
\renewcommand{\arraystretch}{0.5}
\setlength\tabcolsep{0.5pt}
\centering
\vspace{-0.1cm}
\begin{tabular}{ccccccc}
\centering

\includegraphics[width=29.3mm, height = 28.3mm]{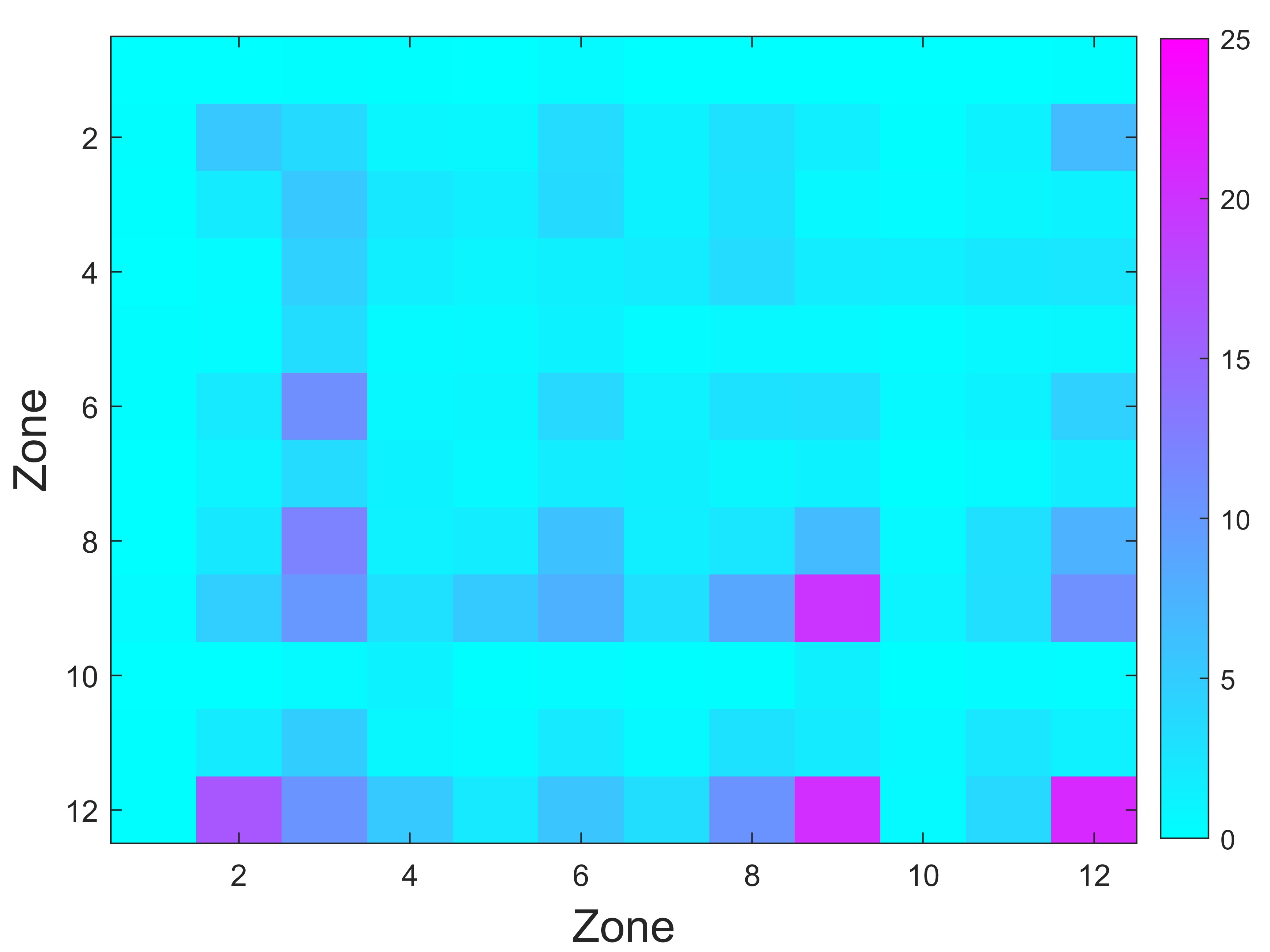}&
\includegraphics[width=29.3mm, height = 28.3mm]{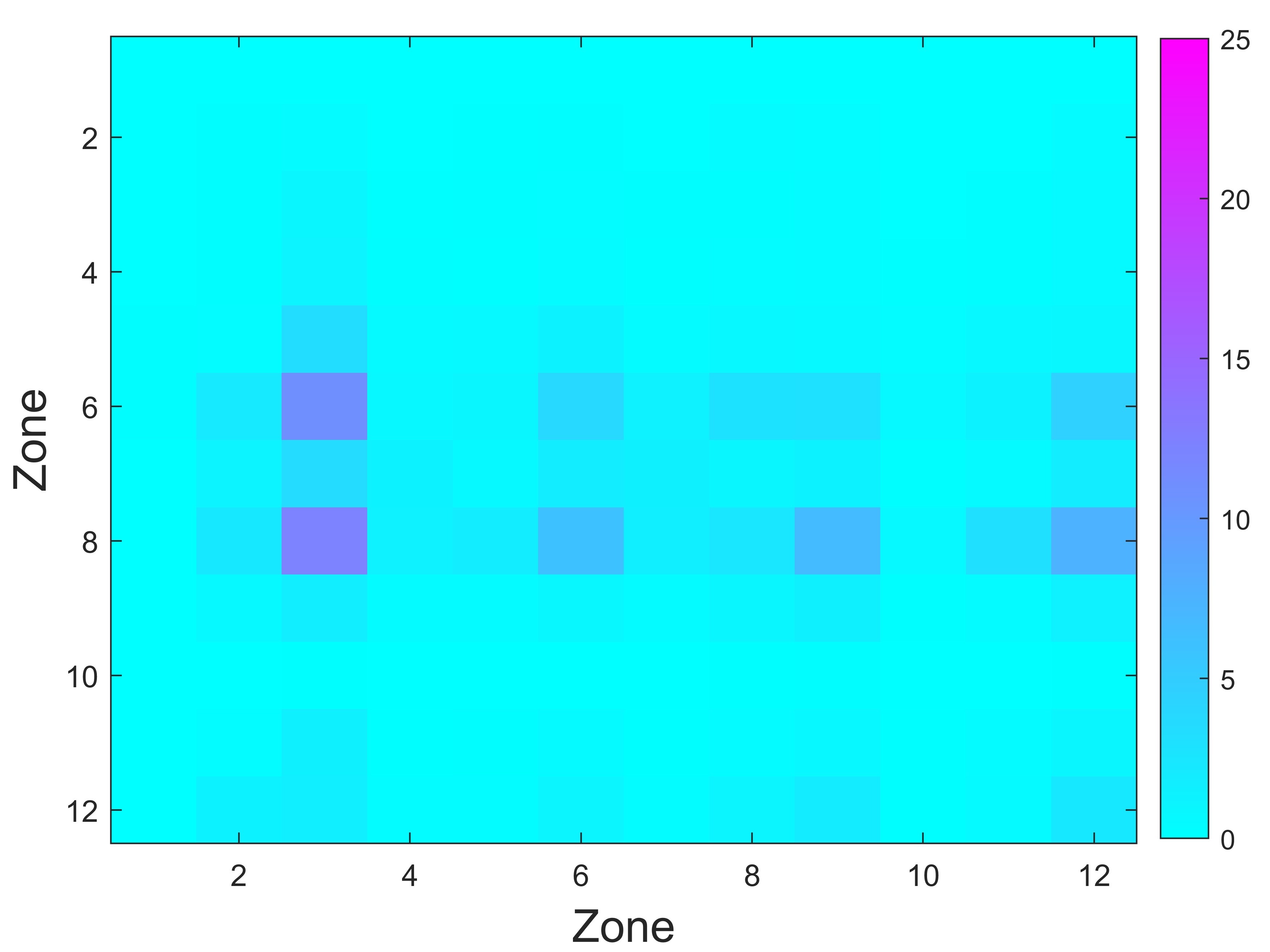}&
\includegraphics[width=29.3mm, height = 28.3mm]{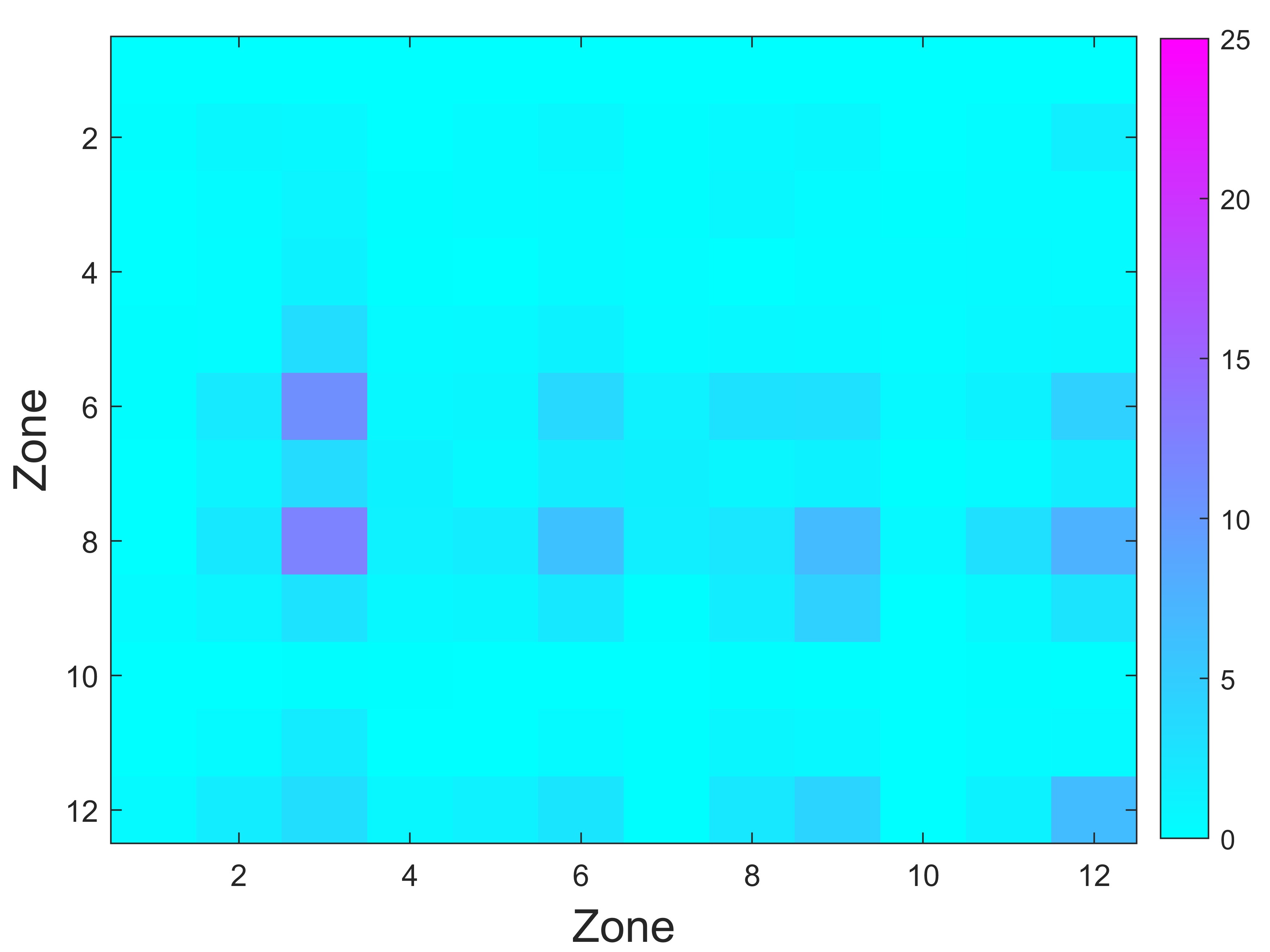}&
\includegraphics[width=29.3mm, height = 28.3mm]{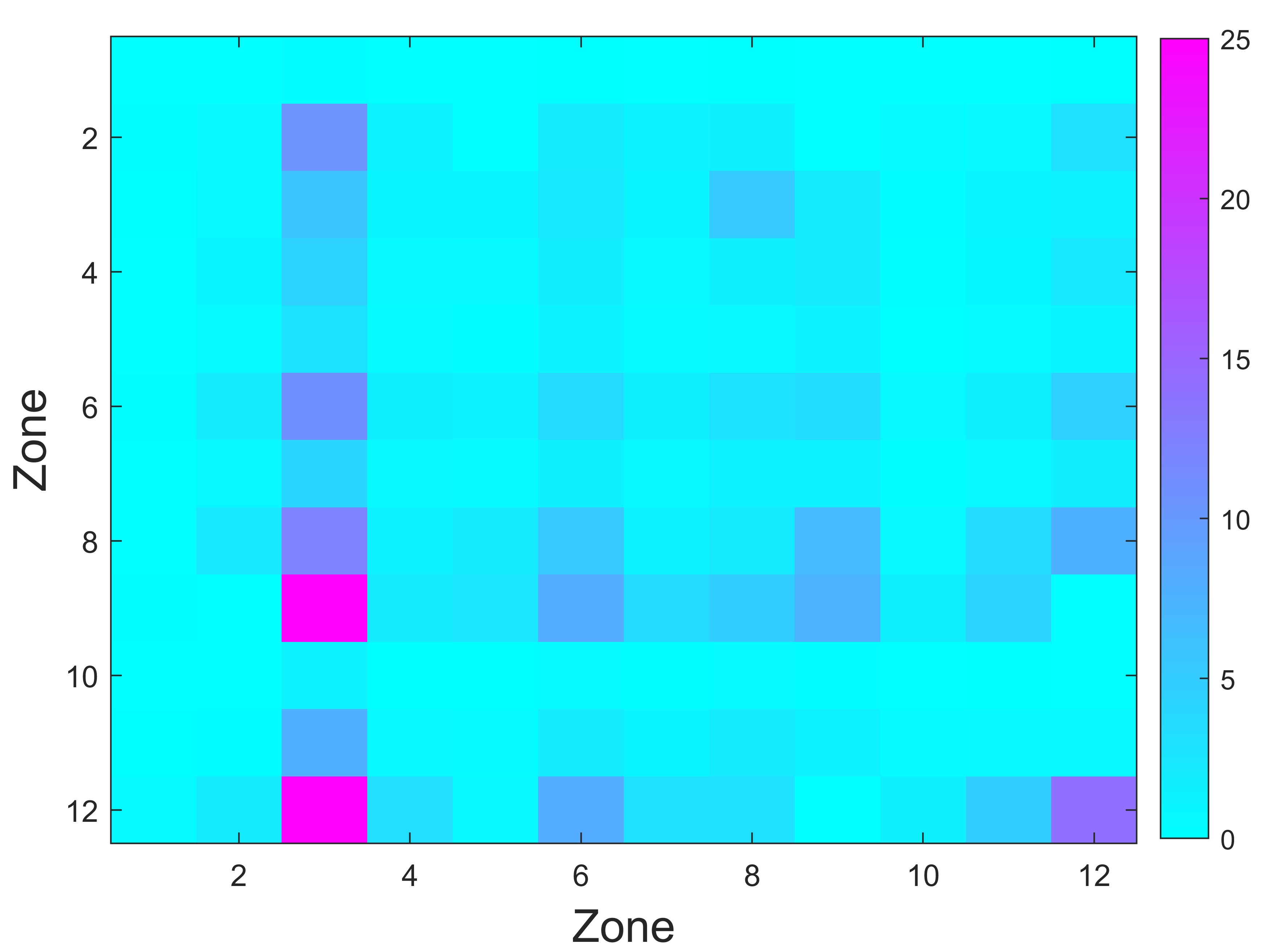}\\
\scriptsize \textbf{Ground truth}& \scriptsize \textbf{HaLRTC} & \scriptsize \textbf{TNN} & \scriptsize \textbf{BCPF} \\

\includegraphics[width=29.3mm, height = 28.3mm]{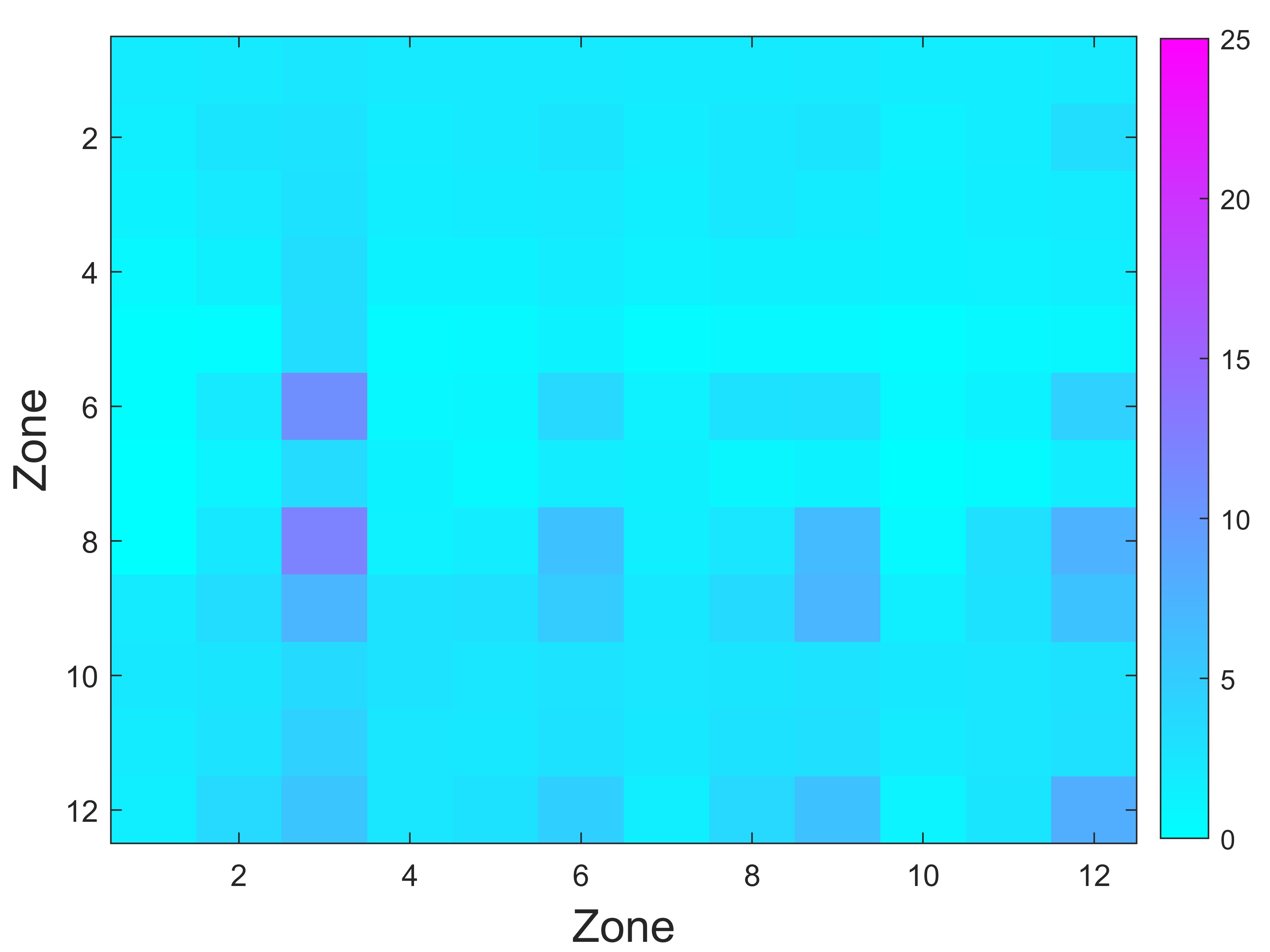}&
\includegraphics[width=29.3mm, height = 28.3mm]{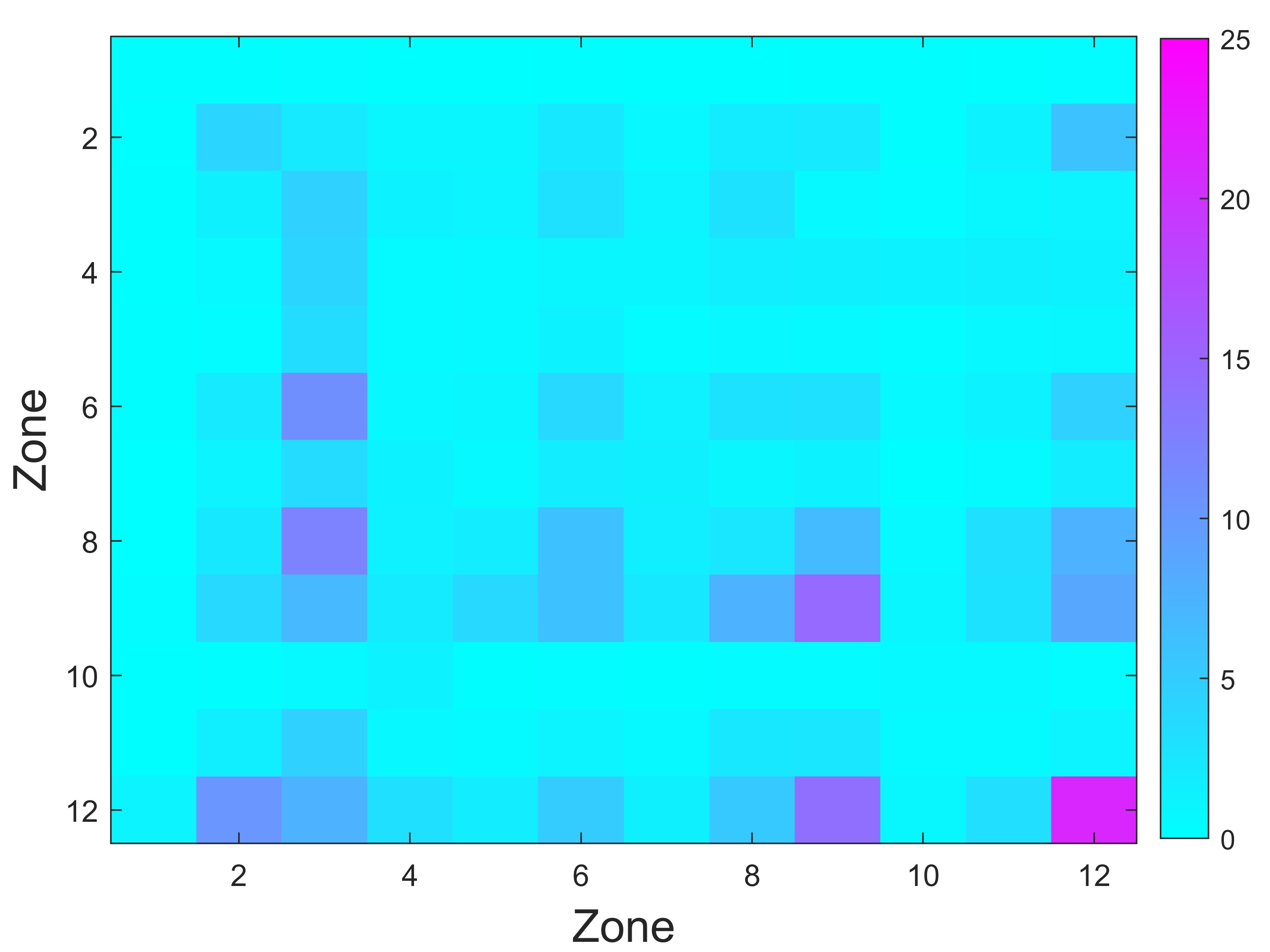}&
\includegraphics[width=29.3mm, height = 28.3mm]{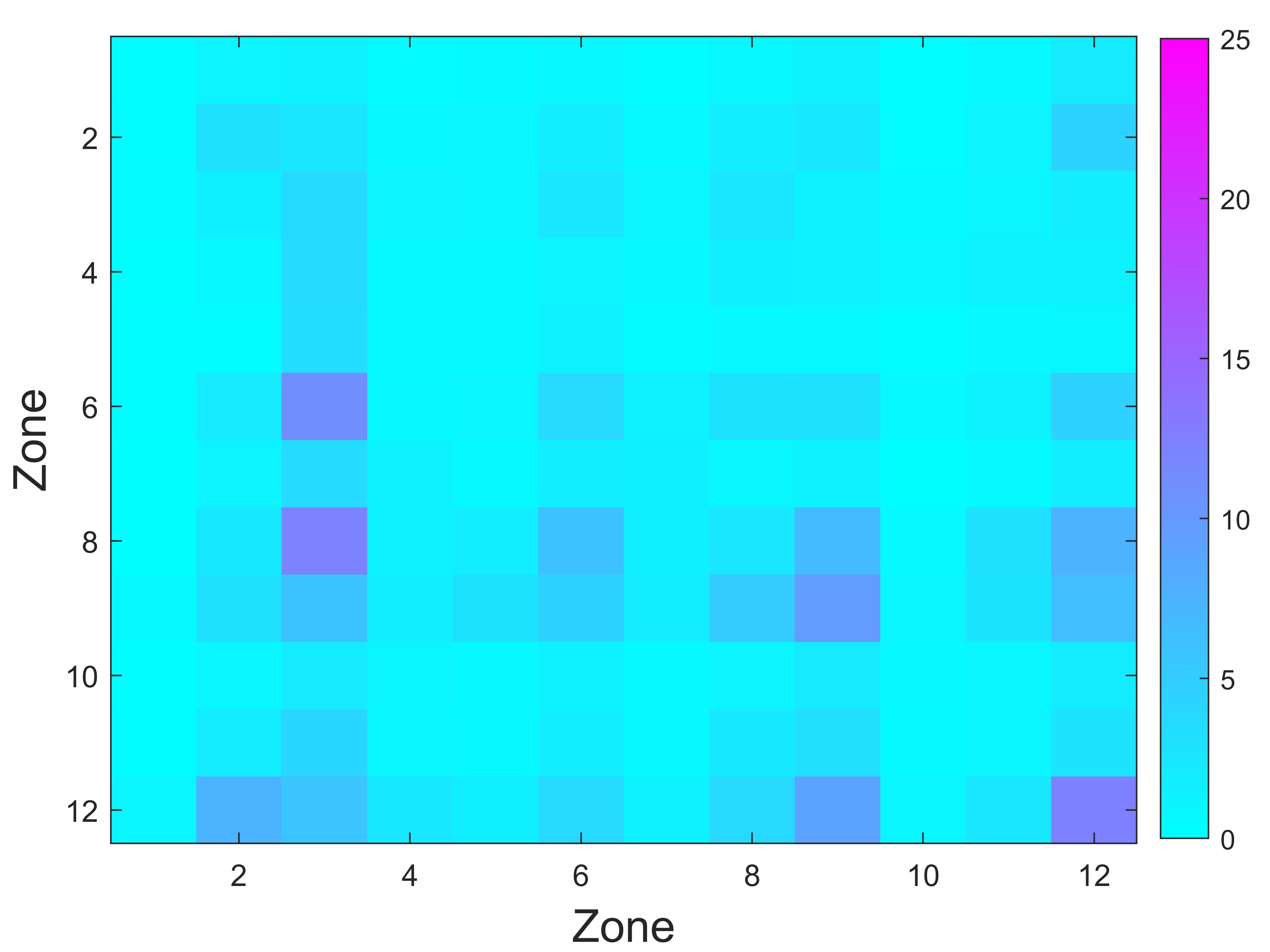}&
\includegraphics[width=29.3mm, height = 28.3mm]{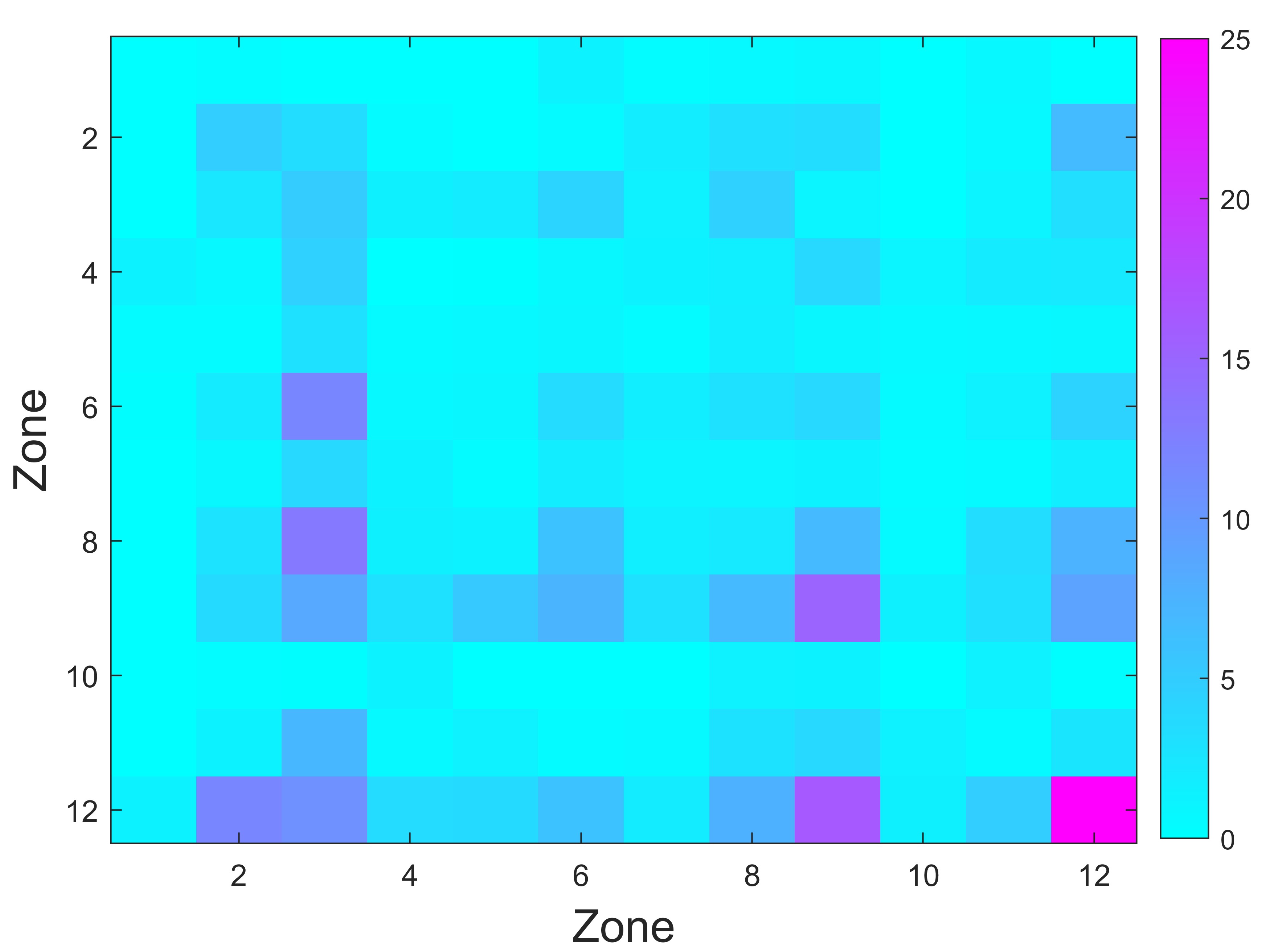}\\
 \scriptsize \textbf{TCTV} &
\scriptsize \textbf{MDT-Tucker}& \scriptsize \textbf{CNNM} & \scriptsize \textbf{LRTC-TIDT}\\
\end{tabular}
\caption{The recovery results of different methods at the 95th time sampling point of Abilene under sampling pattern-3 with a 50\% missing rate.}\label{fig.Abilene95}
\vspace{-0.5cm}
\end{figure}

Network flow data consist of a sequence of matrices recording the volumes of data exchanged between origin–destination (OD) pairs, which are widely used in computer network analysis.
However, it is often compromised by missing values due to hardware or software failures. In subsequent experiments, we apply our method to estimate network flow data from incomplete measurements under various non-random sampling scenarios.
The  network flow dataset we use is Abilene\footnote{\url{http://abilene.internet2.edu/observatory/data-collections.html}}, which is recorded from 12:00 AM to 5:00 PM on March 1, 2004, with a temporal resolution of 5 minutes. We structure this data into a tensor with dimensions 204 × 12 × 12, where the first mode represents 204 time intervals, the second mode denotes 12 source routers, and the third mode signifies 12 destination routers.

Based on the three different non-random sampling patterns illustrated in Fig.\ref{sampling_pattern23}(a, b, c), we establish seven non-random sampling scenarios (  Pattern-1 20\%, Pattern-1 40\%, Pattern-1 60\%,  Pattern-1 80\%, Pattern-2 30\%, Pattern-2 70\%, Pattern-3 50\%). For the proposed LRTC-TIDT model, we use the parameters 
$k=50$ and $\lambda= 1e10$ for all scenarios in this task. 
Table \ref{Abilene_nonoise} summarizes the MAE and RMSE values of the recovery results for various methods applied to the Abilene dataset without noise. As shown in Table \ref{Abilene_nonoise}, the proposed LRTC-TIDT method outperforms other baseline methods in most sampling cases, achieving nearly the lowest MAE and RMSE values. 
Under sampling pattern-1, HaLRTC, TNN, and BCPF yield identical all-zero imputations, as they rely solely on low-rank modeling. In contrast, TCTV, MDT-Tucker, and CNNM exploit low-rankness in transformed domains to capture richer structures, achieving moderate recovery but still lagging behind LRTC-TIDT. For sampling pattern-2 and pattern-3, all models show some recovery capability, yet those incorporating transform-domain low-rankness consistently outperform the pure low-rank approaches.
We illustrate the case of a 50\% sampling rate under sampling pattern-3 by showing all OD pairs at the 95th time point, corresponding to a 12×12 network flow matrix. LRTC-TIDT achieves the closest recovery to the ground truth, as shown in Fig.~\ref{fig.Abilene95}.

\subsection{Applications to  Urban Traffic Estimation}

Urban traffic datasets (e.g., population mobility data) are often affected not only by missing entries but also by noise, which typically arises from inherent GPS positioning errors and interference during data communication or transmission. The objective of this section is to estimate the unobserved urban traffic data from partially observed and noisy measurements.
In this study, we employ the \textit{New York City Yellow Taxi trip} (NYC-yellow) dataset\footnote{\url{https://www1.nyc.gov/site/tlc/about/tlc-trip-record-data.page}}. 
Specifically, we consider 69 regions in Manhattan as pickup and drop-off zones and aggregate the daily taxi trip counts over the first 60 
days of 2021, resulting in a tensor of size 60 $\times$ 69 $\times$ 69.

\begin{table}[!htbp]
\renewcommand{\arraystretch}{1.4}
\setlength\tabcolsep{3.5pt}
\footnotesize
 \vspace{-0.3cm}
  \caption{Performance comparison (in MAE/RMSE) of LRTC-TIDT and other baselines under different non-random sampling scenarios with 30\% noisy observations and 70\% missing entries.
  }\label{tab: nyc_noise}
  \centering
  \vspace{-0.2cm}
\begin{tabular}{l|ccccccc}
    \hline
    Pattern/Method & HaLRTC & TNN & BCPF & TCTV & MDT-Tucker & CNNM & LRTC-TIDT \\
    \hline
    Pattern-1 & 8.65/22.27 & 8.65/22.27 & 8.65/22.27 & 7.48/19.96 & \underline{3.73}/\underline{9.05} & 4.93/13.42 & \textbf{3.58}/\textbf{8.96} \\
    Pattern-2 & 7.39/19.74 & 6.74/18.53 & 6.53/18.40 & 5.22/14.99 & 5.88/14.76 & \underline{3.96}/\underline{11.06} & \textbf{3.71}/\textbf{7.89} \\
    Pattern-3 & 7.63/20.31 & 6.09/17.42 & 5.99/17.14 & 5.70/16.36 & 4.65/\underline{11.35} & \underline{4.51}/13.10 & \textbf{4.31}/\textbf{8.27} \\
    \hline
\end{tabular}
\vspace{-0.2cm}
\end{table}

\begin{figure}[!htbp]
\renewcommand{\arraystretch}{0.5}
\setlength\tabcolsep{0.5pt}
\centering
\begin{tabular}{ccccccc}
\centering
\includegraphics[width=32.3mm, height = 30.3mm]{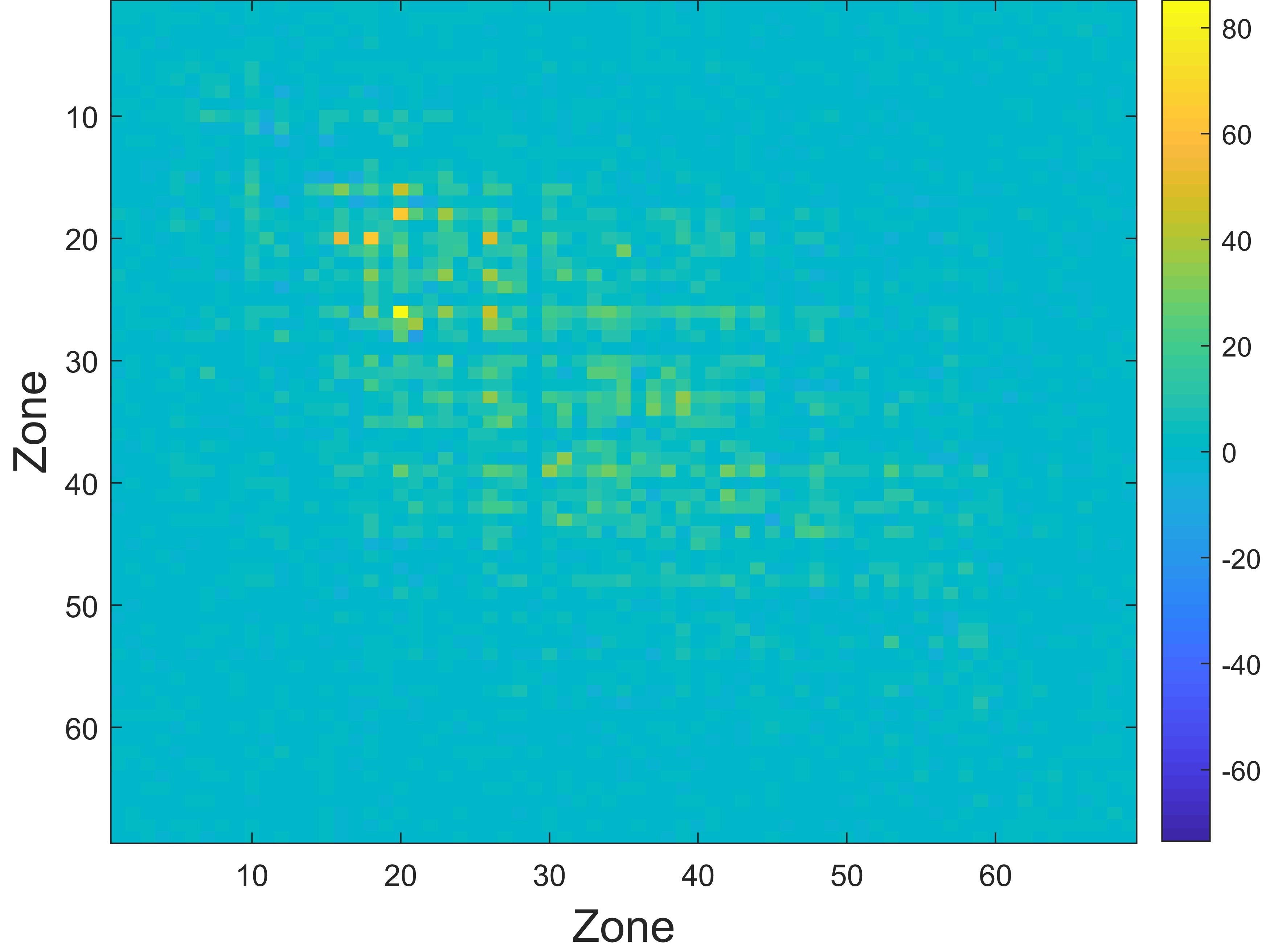}&
\includegraphics[width=32.3mm, height = 30.3mm]{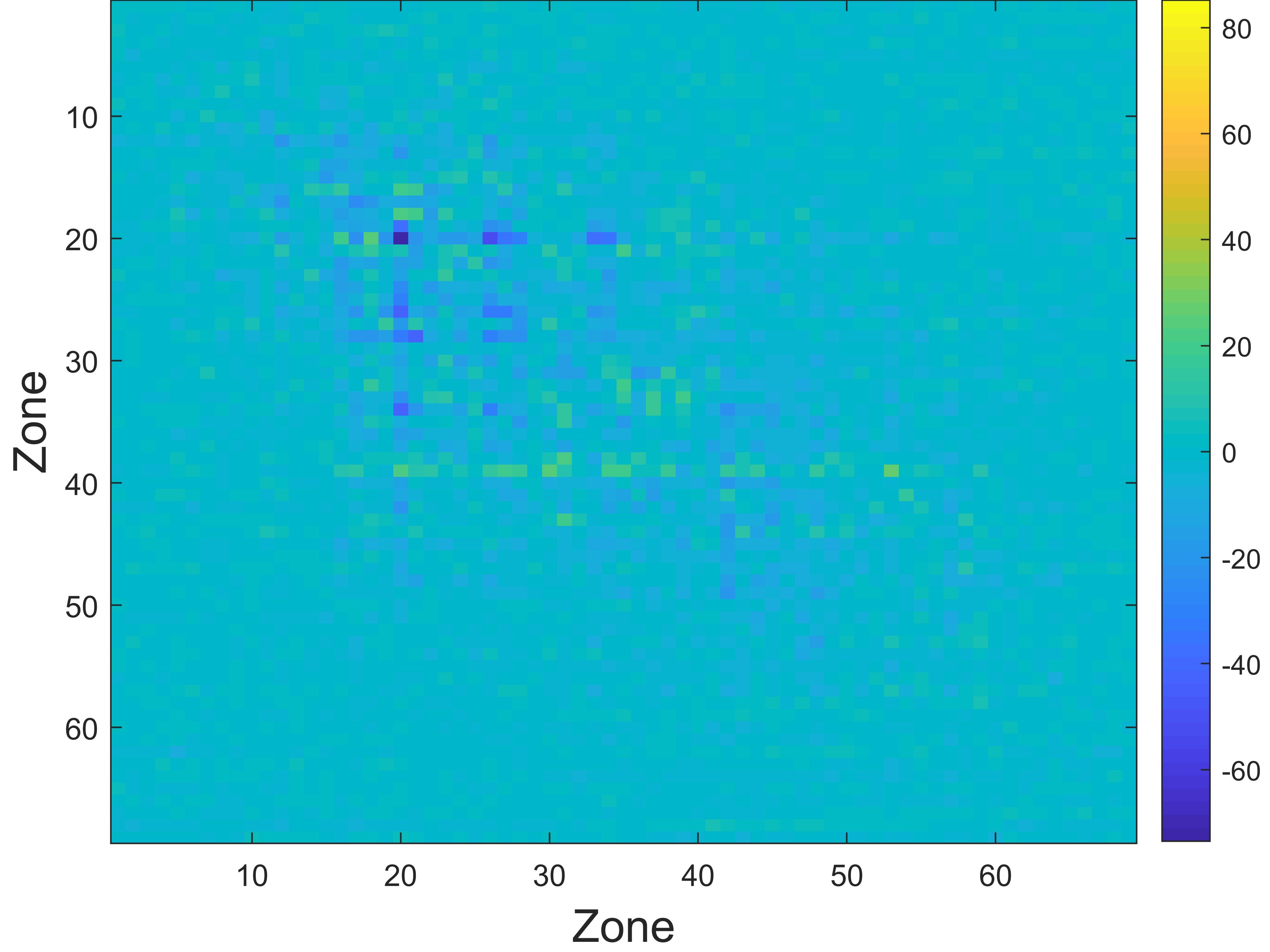}&
\includegraphics[width=32.3mm, height = 30.3mm]{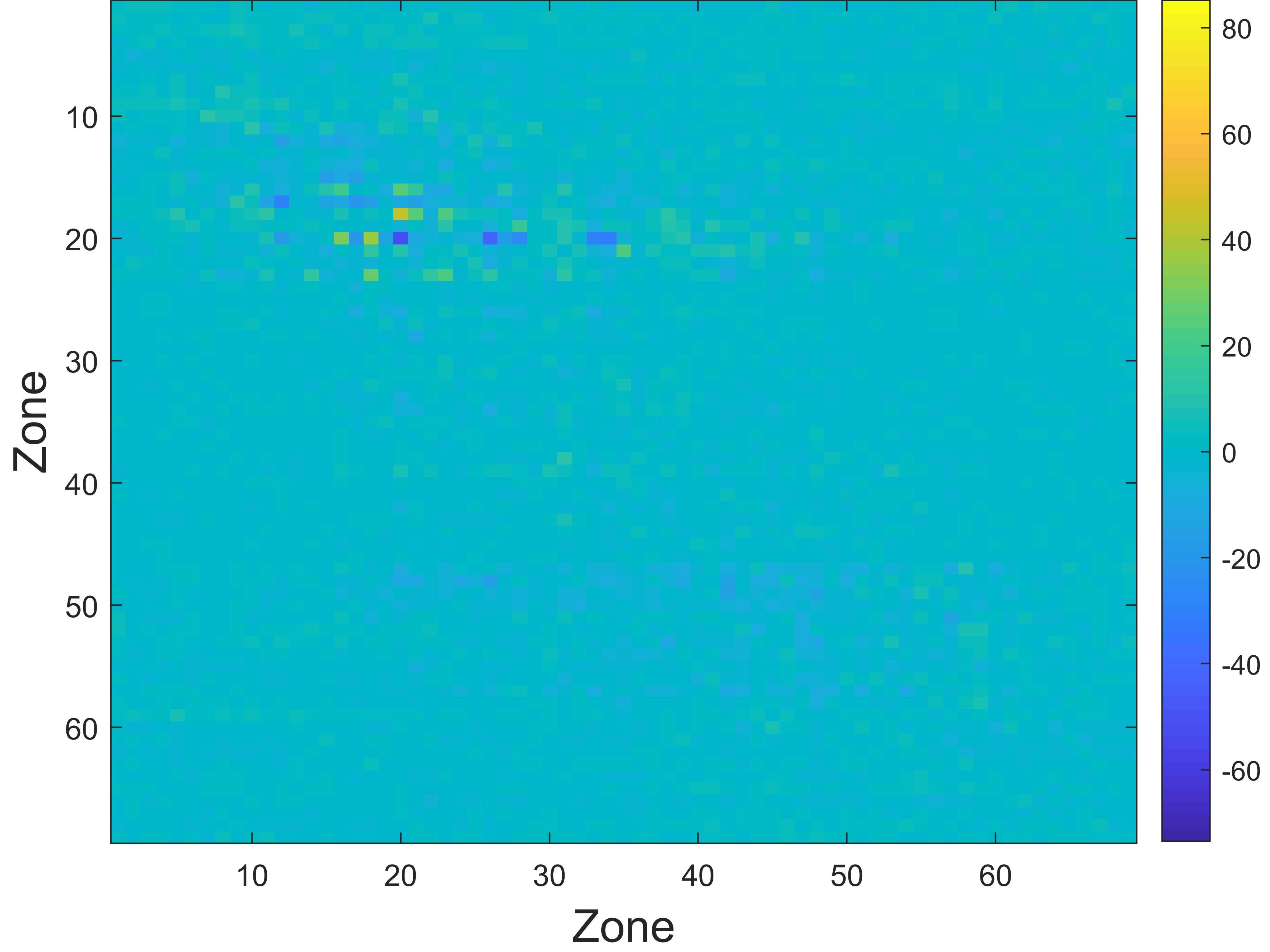}\\
\scriptsize \textbf{Pattern-1}& \scriptsize \textbf{Pattern-2} & \scriptsize \textbf{Pattern-3}\\
\end{tabular}
\caption{ The gap between the recovered and true values at the 46th time sampling point of NYC-yellow using the LRTC-TIDT model.}\label{fig.nyc}
\vspace{-0.3cm}
\end{figure}

During the experiments, we introduce Gaussian noise with a standard deviation of 1 under three non-random sampling patterns with a 30\% sampling rate. Each scenario thus involves 70\% non-random missing entries and 30\% noisy observations. 
For all scenarios in this task, the proposed LRTC-TIDT model is configured with parameters $k = 20$ and $\lambda = 0.01$.
The recovery accuracy under these settings is reported in Table \ref{tab: nyc_noise}, demonstrating the robustness of LRTC-TIDT in noisy environments. Furthermore, we compute the differences between the recovered and true values under different sampling patterns at the 46th time point, as illustrated in Fig.\ref{fig.nyc}. The results show that the discrepancies across all patterns are minimal, with most values close to zero.

\subsection{Applications to Temperature Field Prediction}

Accurate prediction of time-varying temperature fields is essential for optimizing energy systems and mitigating meteorological disasters. 
We evaluate the proposed LRTC-TIDT model on a Pacific Ocean sea surface temperature dataset\footnote{\url{http://iridl.ldeo.columbia.edu/SOURCES/.CAC/}}
 spanning 60 months (Jan 1970–Dec 1974). The data form a 60 × 30 × 84 tensor over a $2^\circ \times 2^\circ$latitude–longitude grid, with 60, 30, and 84 representing time, latitude, and longitude dimensions, respectively.

\begin{table}[!htbp]
\renewcommand{\arraystretch}{1.4}
\setlength\tabcolsep{3.5pt}
\footnotesize
  \caption{Performance comparison (in MAE/RMSE) of LRTC-TIDT and other baselines with different forecast scenarios. The forecast horizon (FH) row indicates the respective forecast horizons.
  }\label{tab:temperature_nonoise}
  \centering
  \vspace{-0.2cm}
\begin{tabular}{l|cccc}
    \hline
    Method/FH & h=4 & h=6 & h=8 & h=10\\
    \hline
     TCTV & 1.39/1.81 & 2.68/3.14 & 7.73/8.52
         & 11.36/12.30 \\    
     MDT-Tucker & 1.44/1.85 & 1.96/2.39 & 2.38/2.81 & 2.80/3.29  \\ 
     CNNM & \underline{1.25}/\underline{1.56} & \underline{1.66}/\underline{2.06} & \underline{1.81}/\underline{2.22} & \underline{1.85}/\underline{2.27} \\ 
     LRTC-TIDT & \textbf{0.97}/\textbf{1.34} & \textbf{1.14}/\textbf{1.58} & \textbf{1.25}/\textbf{1.69} & \textbf{1.29}/\textbf{1.72}  \\ 
    \hline
\end{tabular}
\vspace{-0.2cm}
\end{table}

\begin{figure}[!htbp]
\renewcommand{\arraystretch}{0.5}
\setlength\tabcolsep{0.5pt}
\centering
\begin{tabular}{ccccccc}
\centering
\includegraphics[width=25.3mm, height = 23.3mm]{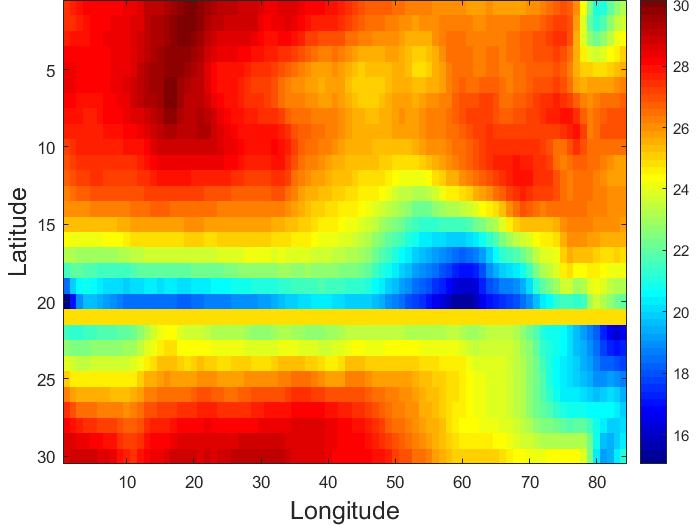}&
\includegraphics[width=25.3mm, height = 23.3mm]{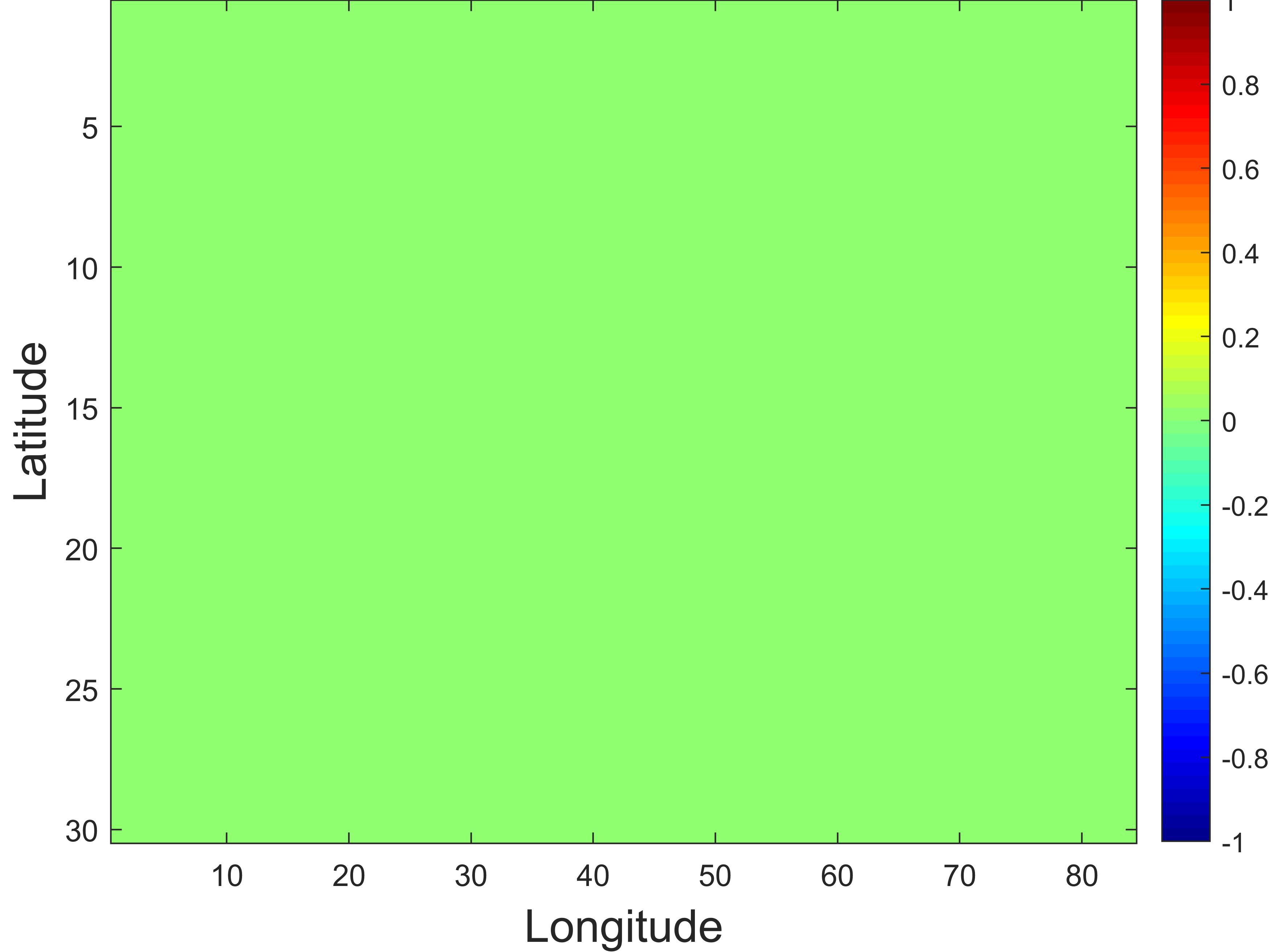}&
\includegraphics[width=25.3mm, height = 23.3mm]{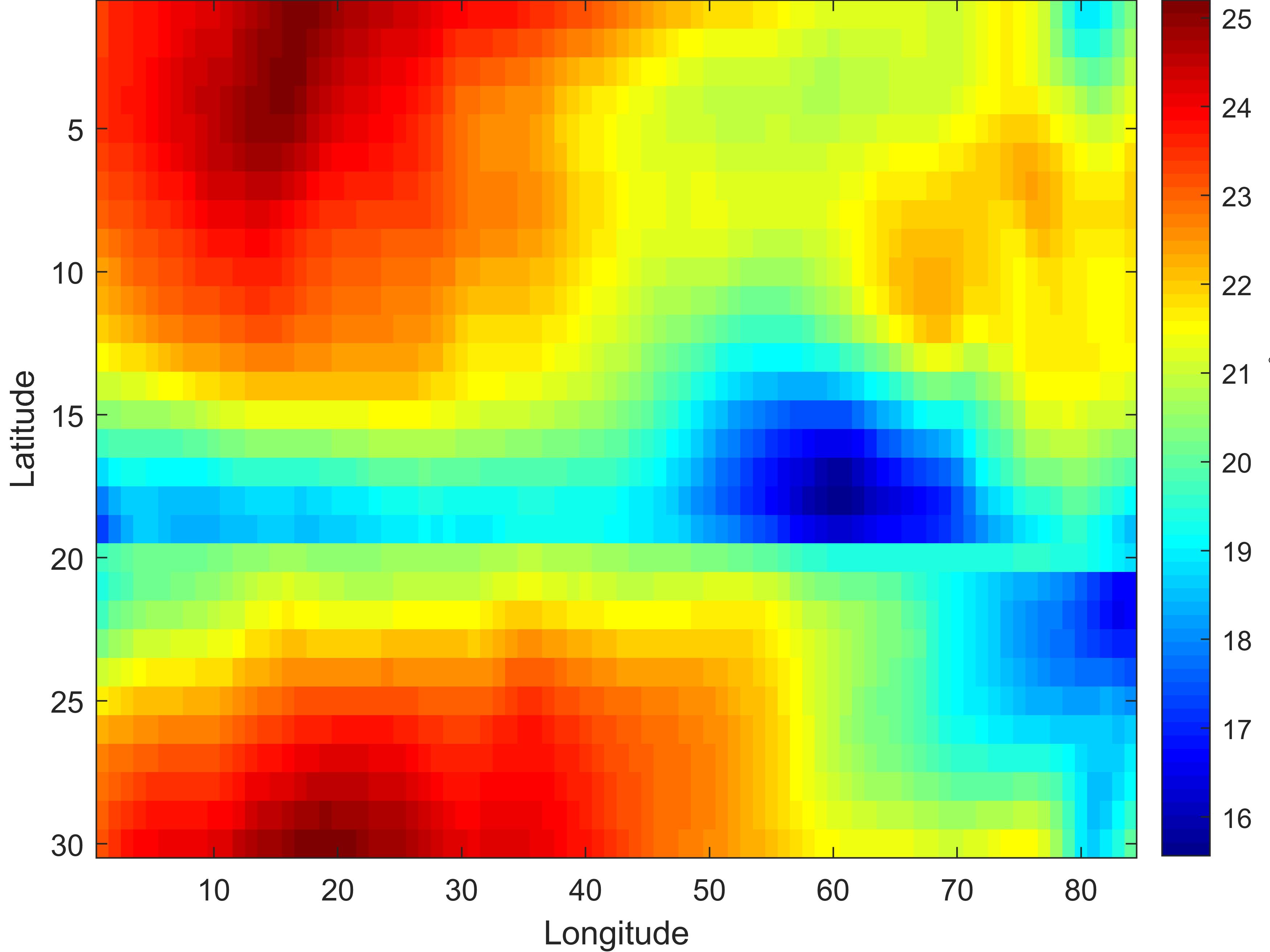}&
\includegraphics[width=25.3mm, height = 23.3mm]{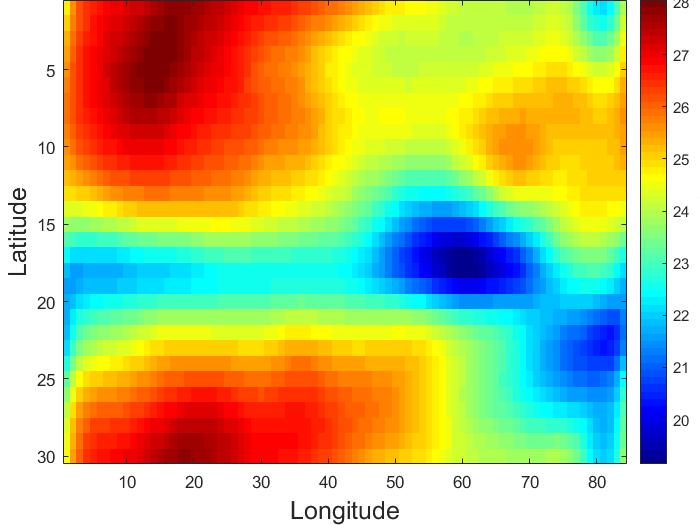}&
\includegraphics[width=25.3mm, height = 23.3mm]{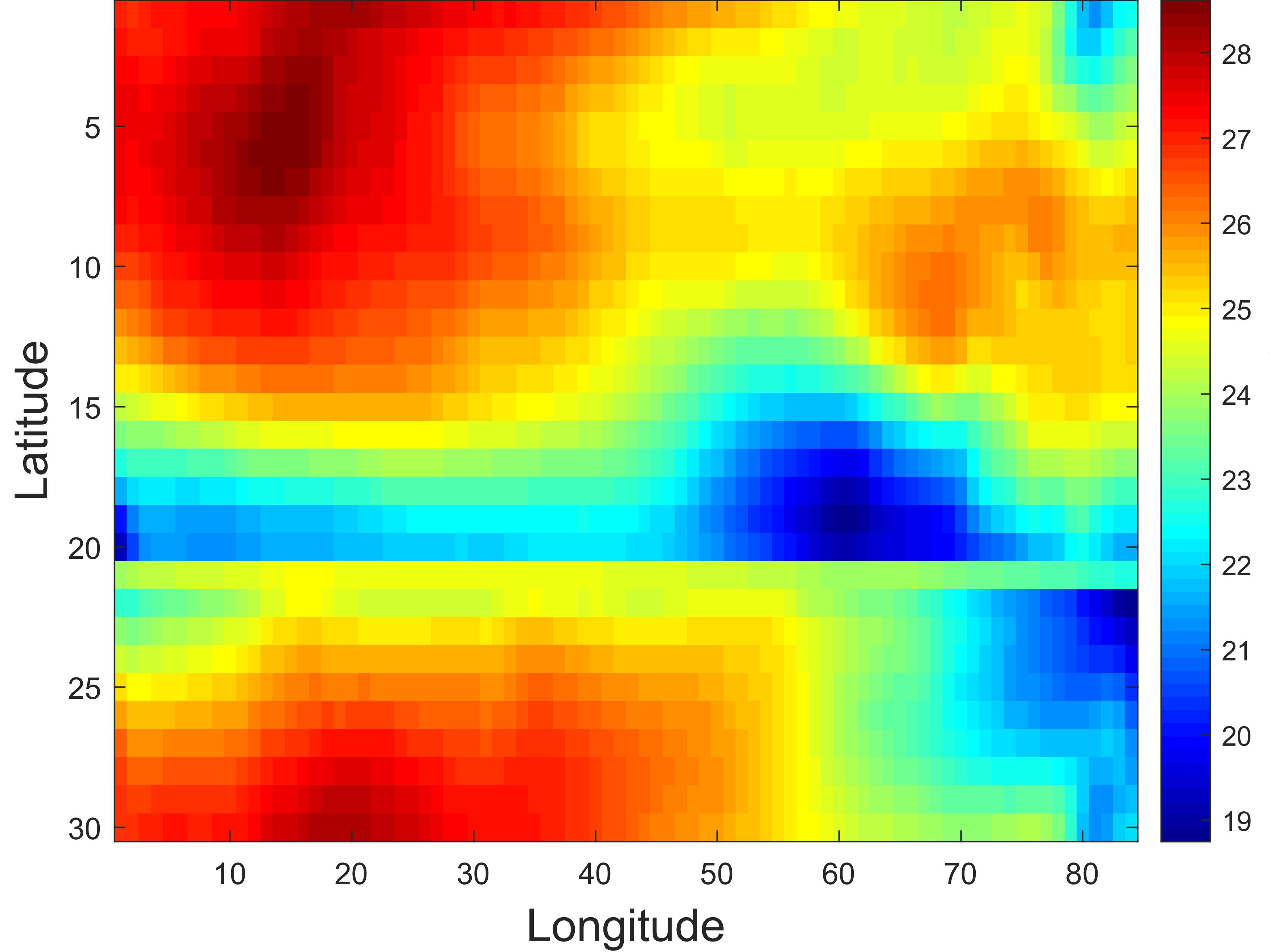}&
\includegraphics[width=25.3mm, height = 23.3mm]{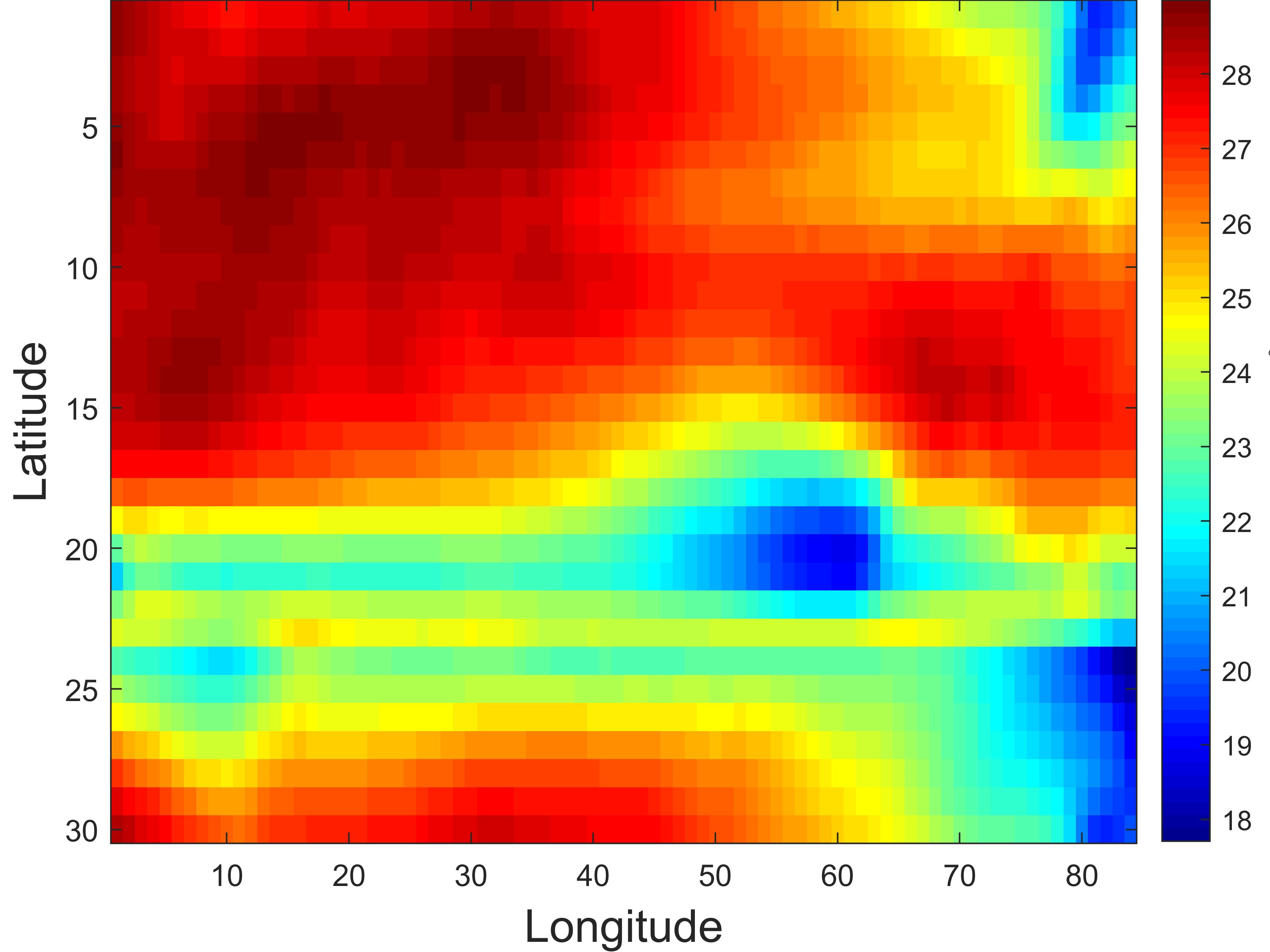}\\
\scriptsize \textbf{Ground Truth} & \scriptsize \textbf{HaLRTC $\&$ TNN}& \scriptsize \textbf{TCTV} &
\scriptsize \textbf{MDT-Tucker}& \scriptsize \textbf{CNNM} & \scriptsize \textbf{LRTC-TIDT}\\
\end{tabular}
\caption{The Prediction results of the Pacific surface temperature in May 1974 with h=8 using the LRTC-TIDT model and other  models.}\label{fig.Temperature_all_h8}
\vspace{-0.1cm}
\end{figure}

We test the prediction performance of different methods under various forecasting horizons, and the results are summarized in Table~\ref{tab:temperature_nonoise}.  For all forecasting horizons, the proposed LRTC-TIDT model is configured with parameters $k = 20$ and $\lambda = 1e10$.   As shown, the proposed LRTC-TIDT model consistently achieves the lowest MAE and RMSE across all scenarios. In addition, we visualize the forecasting results of each method for May 1974 under the forecast horizon $h=8$ in Fig.~\ref{fig.Temperature_all_h8}. It can be observed that the prediction generated by LRTC-TIDT is the closest to the true temperature field, whereas purely low-rank methods such as HaLRTC and TNN fail to produce meaningful results in forecasting scenarios.
To further demonstrate that the proposed LRTC-TIDT method possesses a stronger ability to capture temporal patterns in multidimensional time series compared with the MDT-Tucker approach, we apply both the LRTC-TIDT and MDT-Tucker models to the Pacific dataset with a forecast horizon  $h=6$. The predicted and ground-truth temperature variations over time are shown in Fig.~\ref{fig.temperature_h6}. Clearly, as time progresses, the predictions from MDT-Tucker deviate increasingly from the true values, whereas LRTC-TIDT consistently maintains close agreement with the ground truth. This effectively demonstrates the superior capability of LRTC-TIDT in capturing temporal dependencies in multidimensional time-series data.

\begin{figure*}[!ht]
\renewcommand{\arraystretch}{0.5}
\setlength\tabcolsep{0.5pt}
\centering
\begin{tabular}{ccccccc}
\centering

\tiny  \textbf{Jul. 1974}& \tiny  \textbf{Aug. 1974} & \tiny  \textbf{Sep. 1974} & \tiny \textbf{Oct. 1974} & \tiny  \textbf{Nov. 1974} & \tiny  \textbf{Dec. 1974}\\

\includegraphics[width=20.3mm, height = 19.3mm]{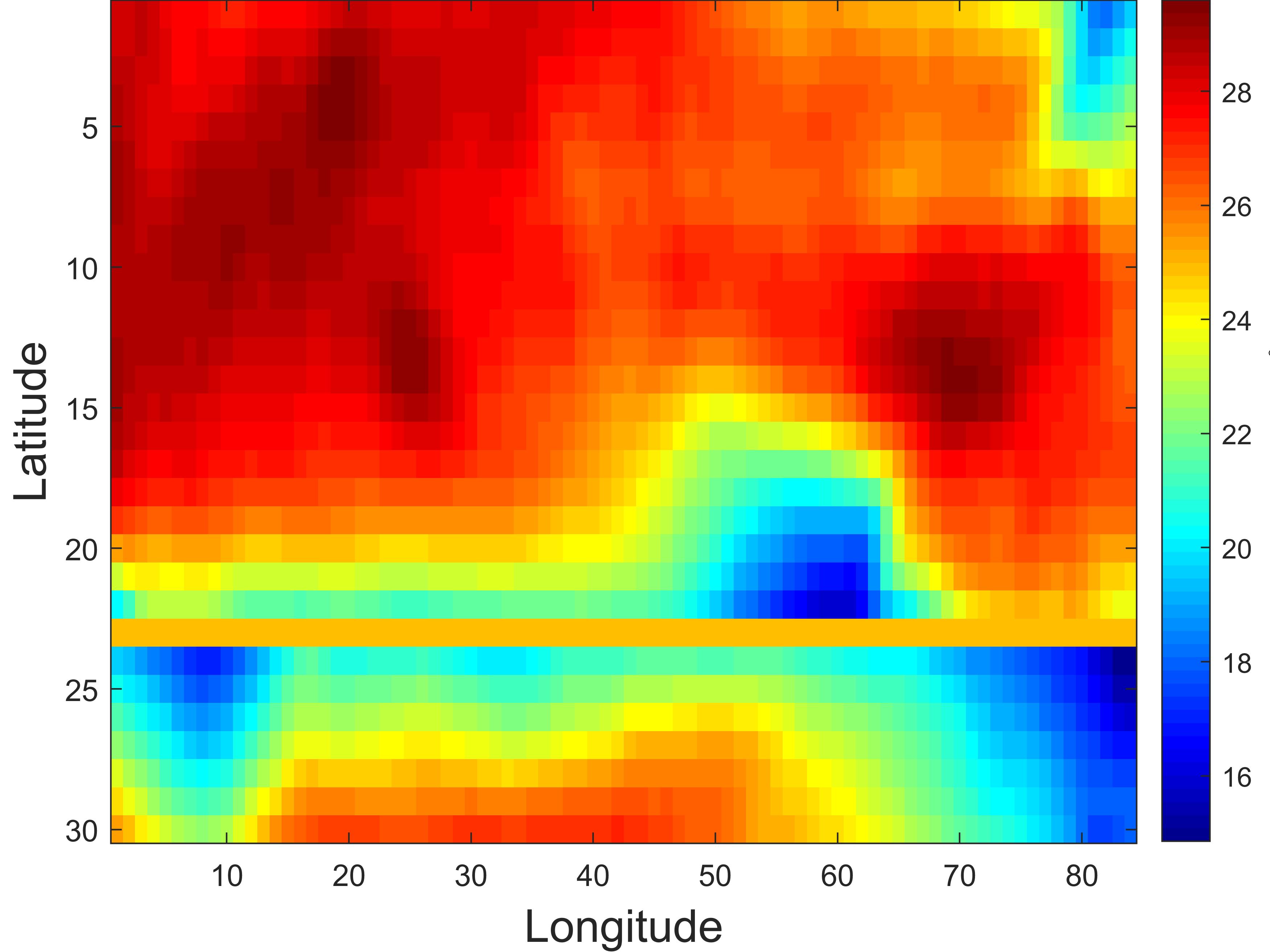}&
\includegraphics[width=20.3mm, height = 19.3mm]{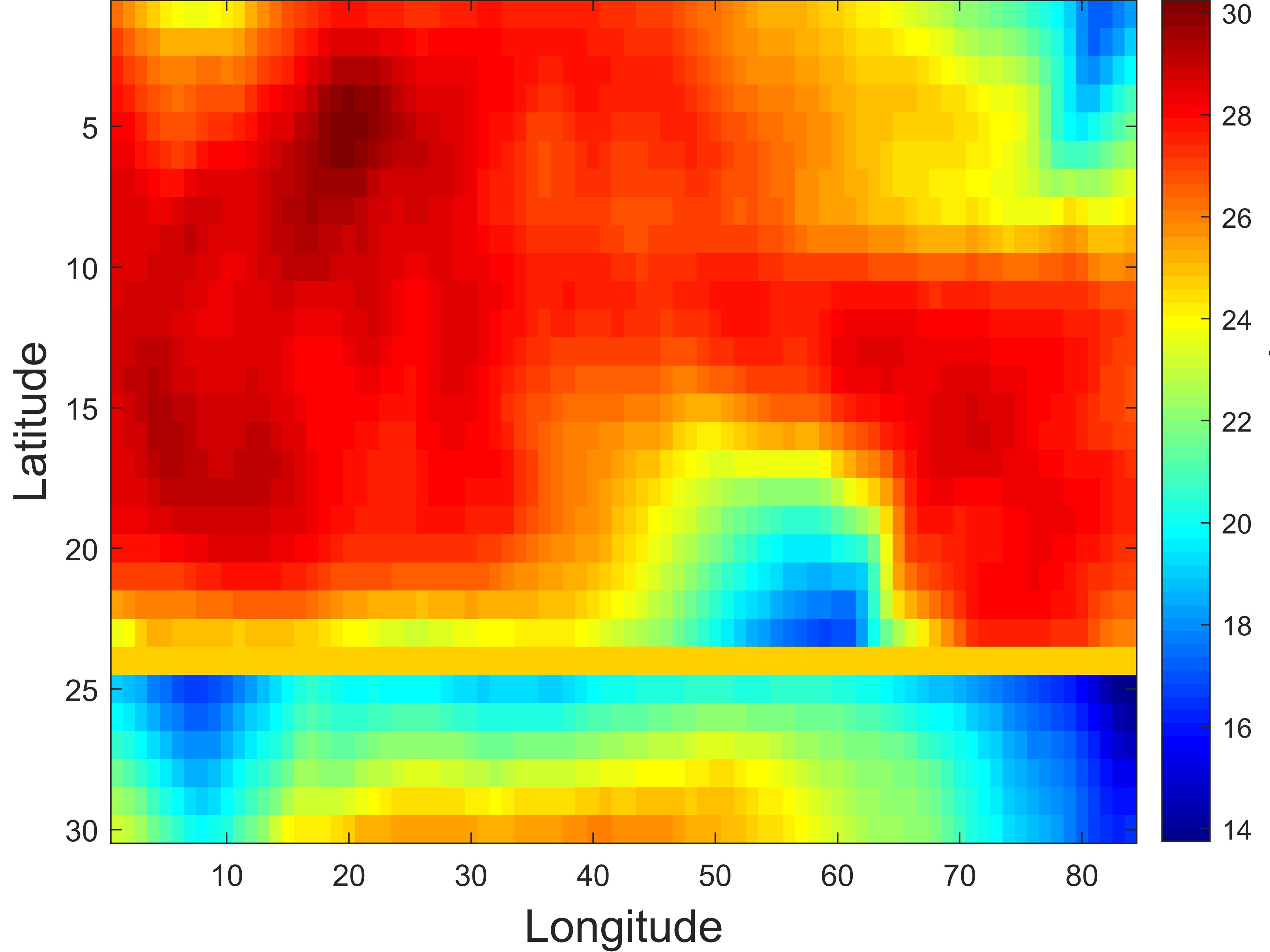}&
\includegraphics[width=20.3mm, height = 19.3mm]{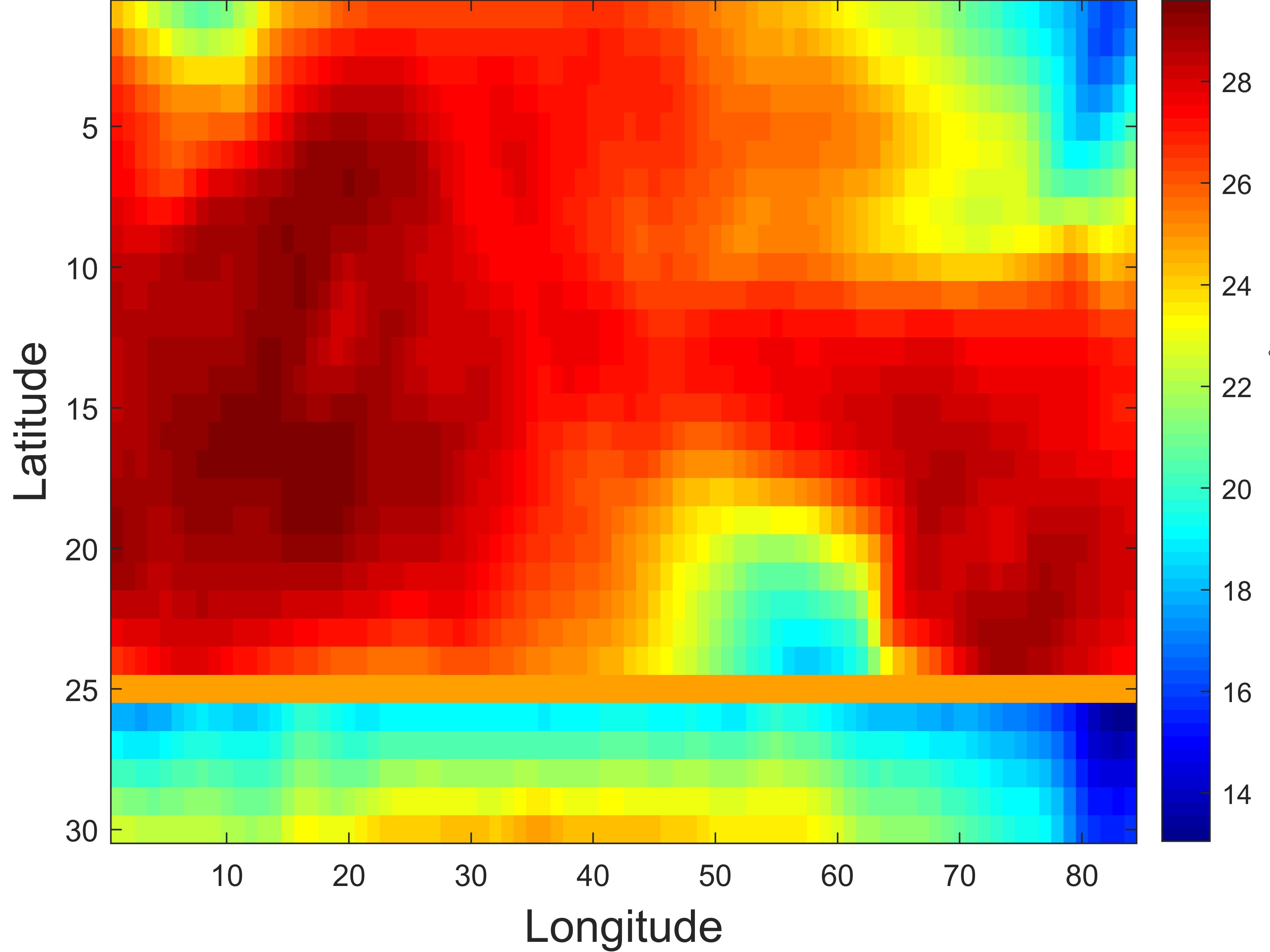}&
\includegraphics[width=20.3mm, height = 19.3mm]{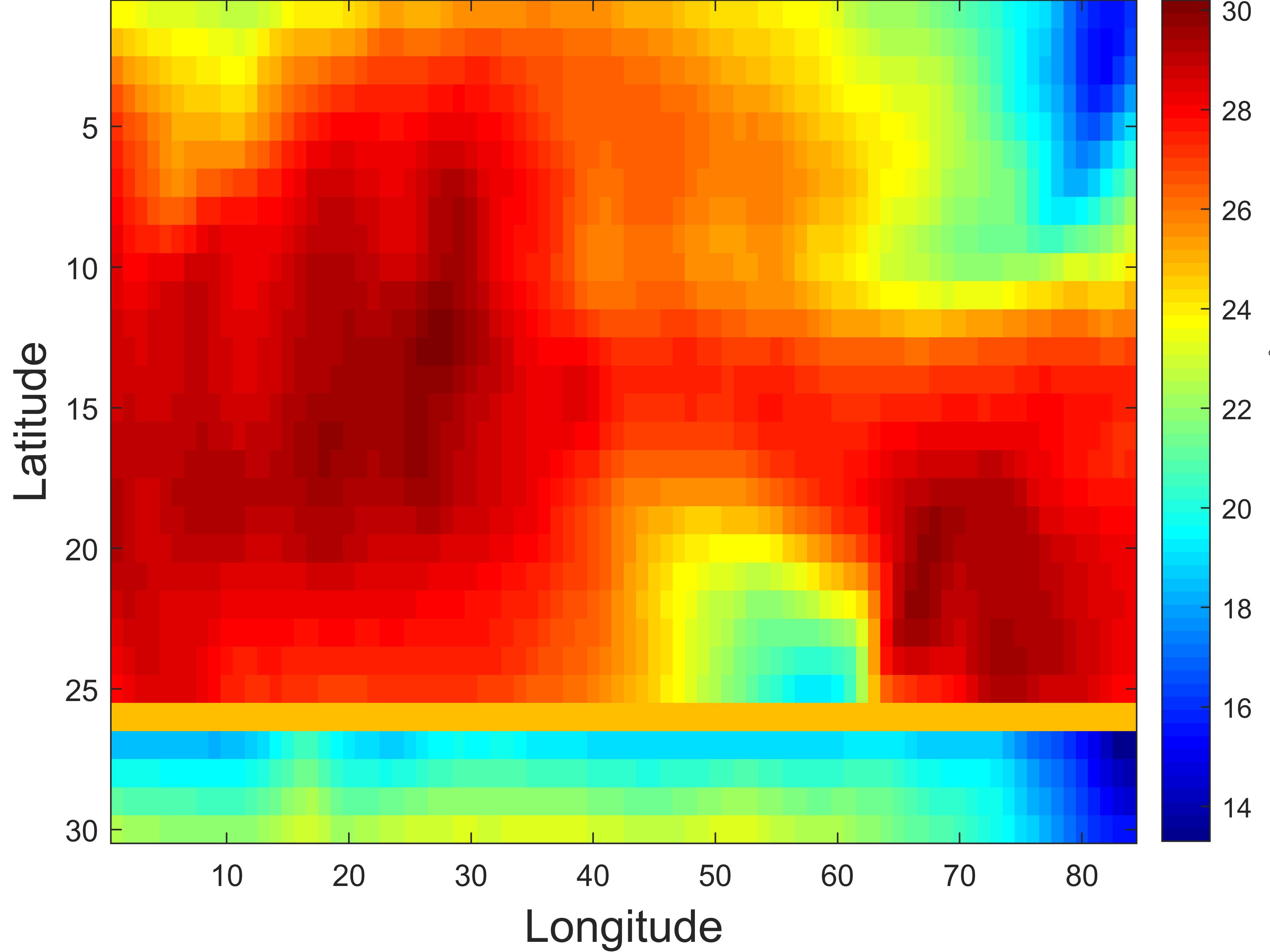}&
\includegraphics[width=20.3mm, height = 19.3mm]{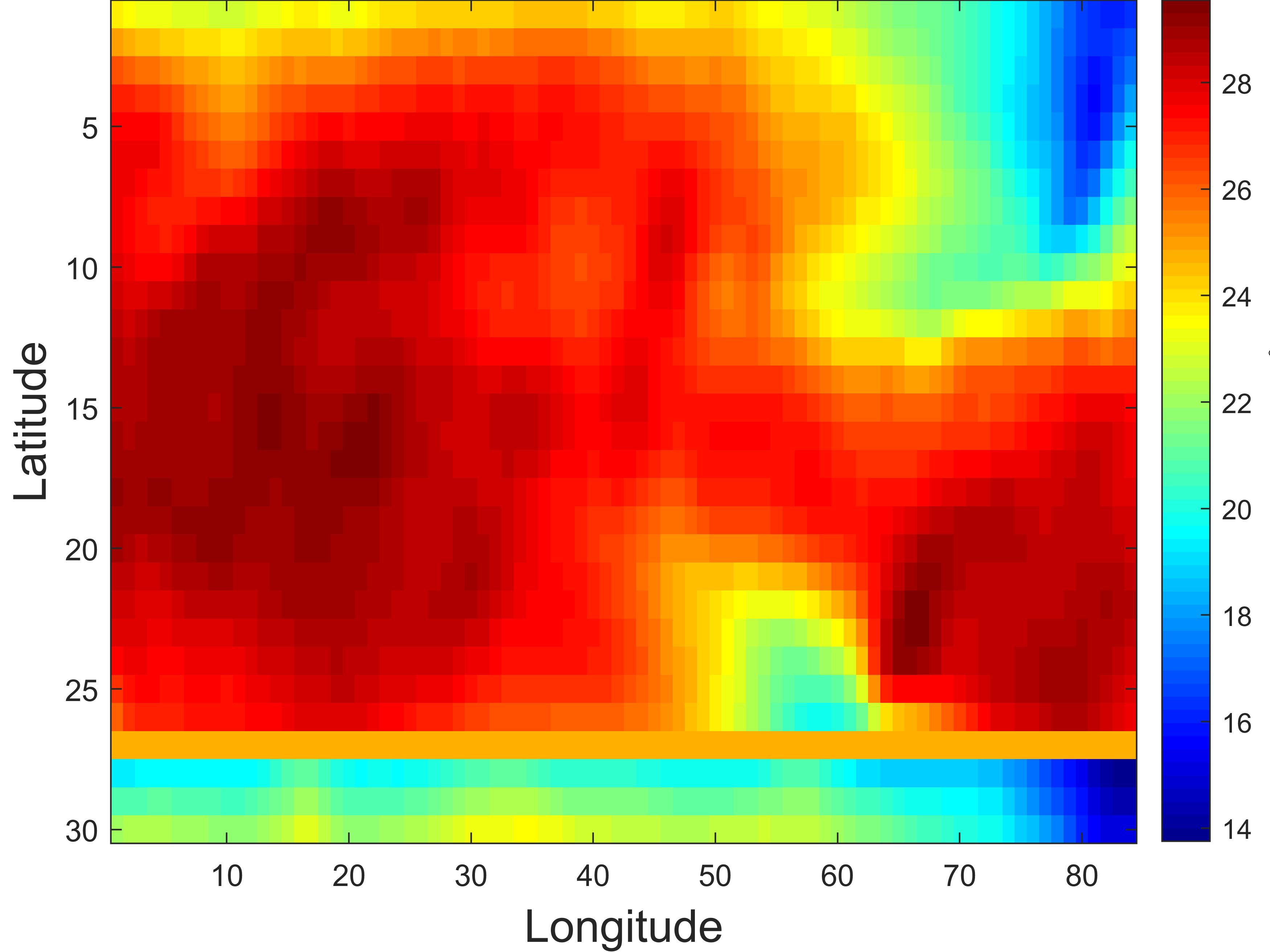}&
\includegraphics[width=20.3mm, height = 19.3mm]{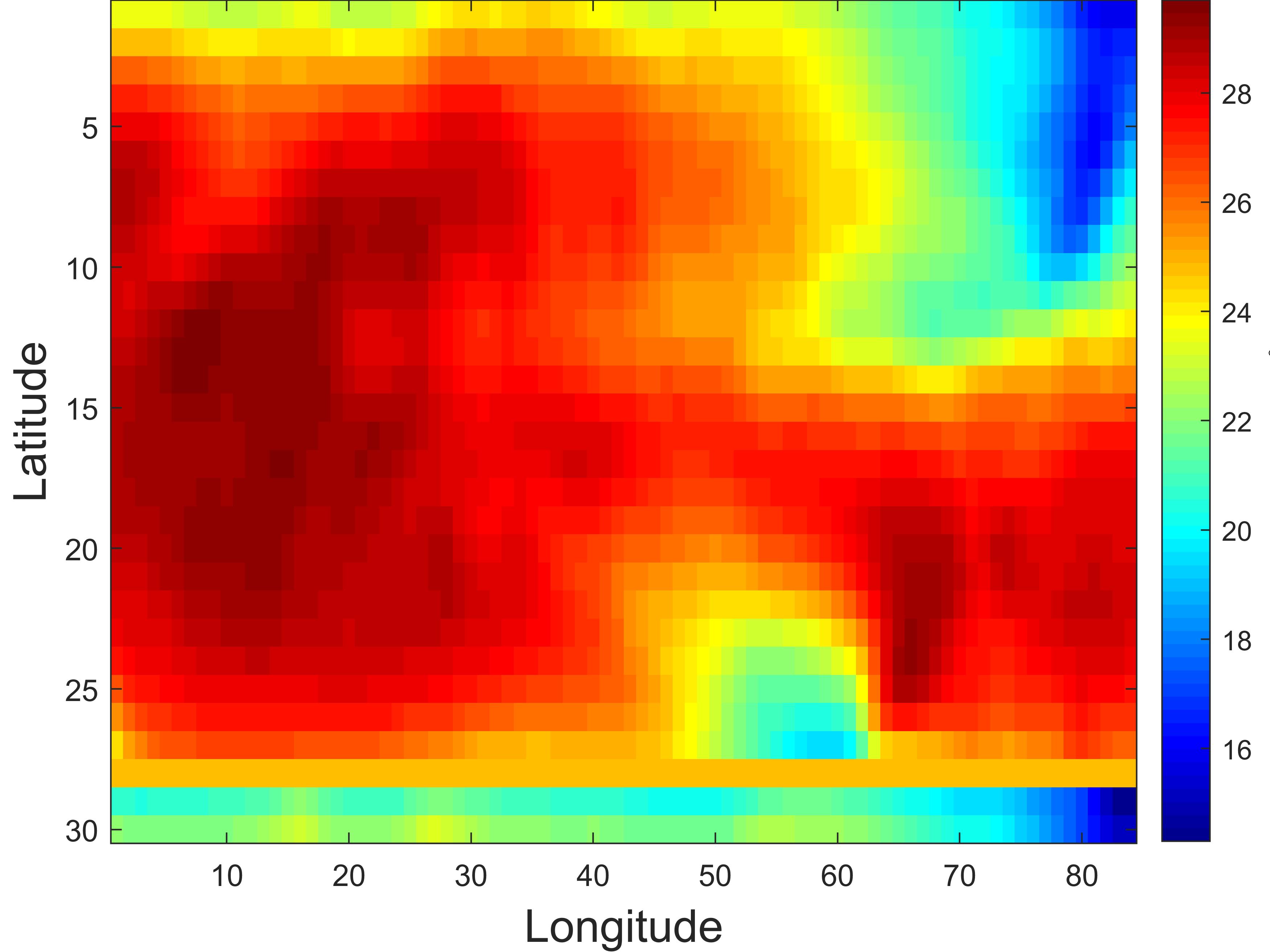}\\

\tiny  \textbf{Ground truth 1}& \tiny  \textbf{Ground truth 2} & \tiny  \textbf{Ground truth 3} & \tiny \textbf{Ground truth 4} & \tiny \textbf{Ground truth 5} & \tiny \textbf{Ground truth 6}\\

\includegraphics[width=20.3mm, height = 19.3mm]{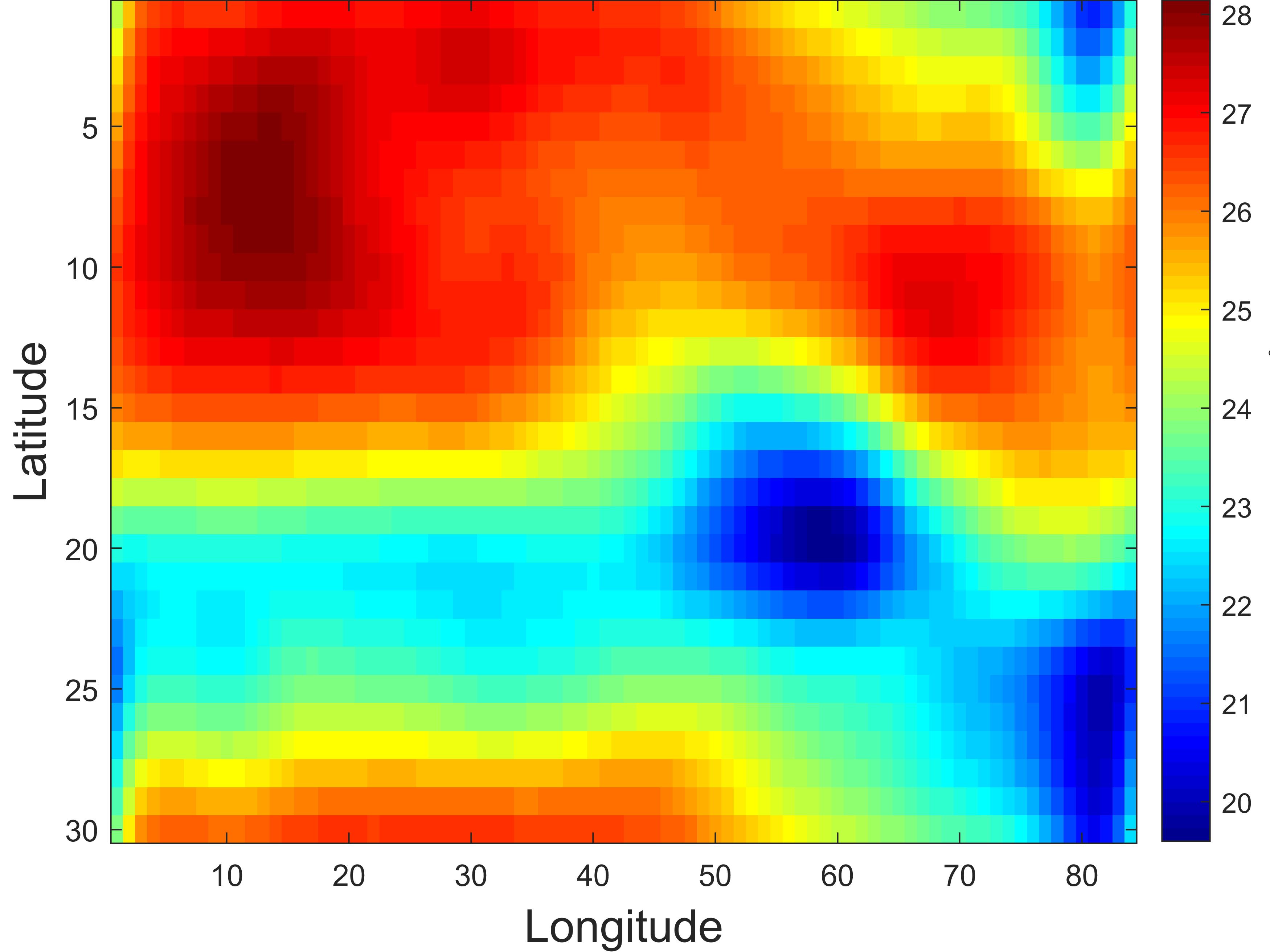}&
\includegraphics[width=20.3mm, height = 19.3mm]{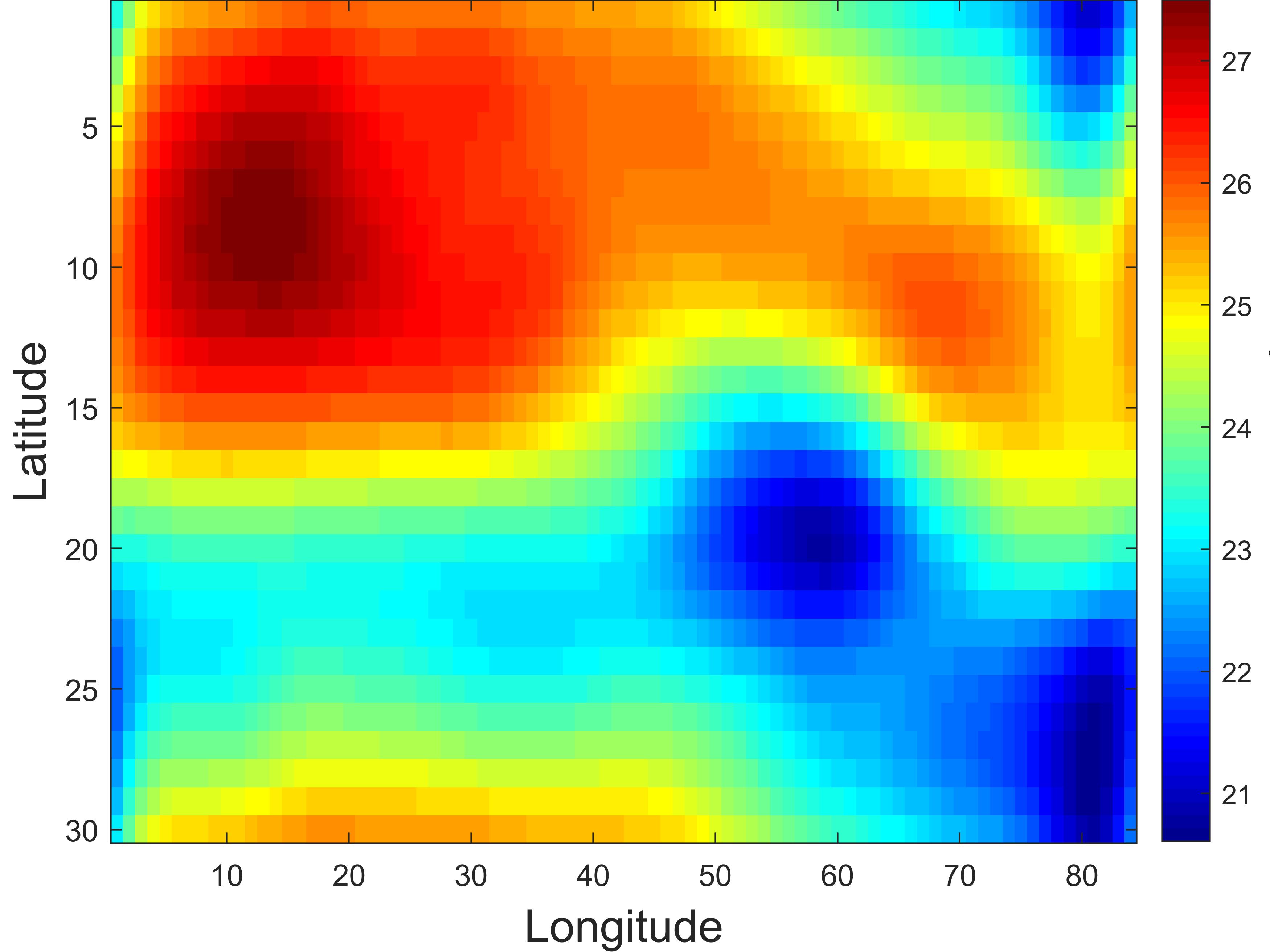}&
\includegraphics[width=20.3mm, height = 19.3mm]{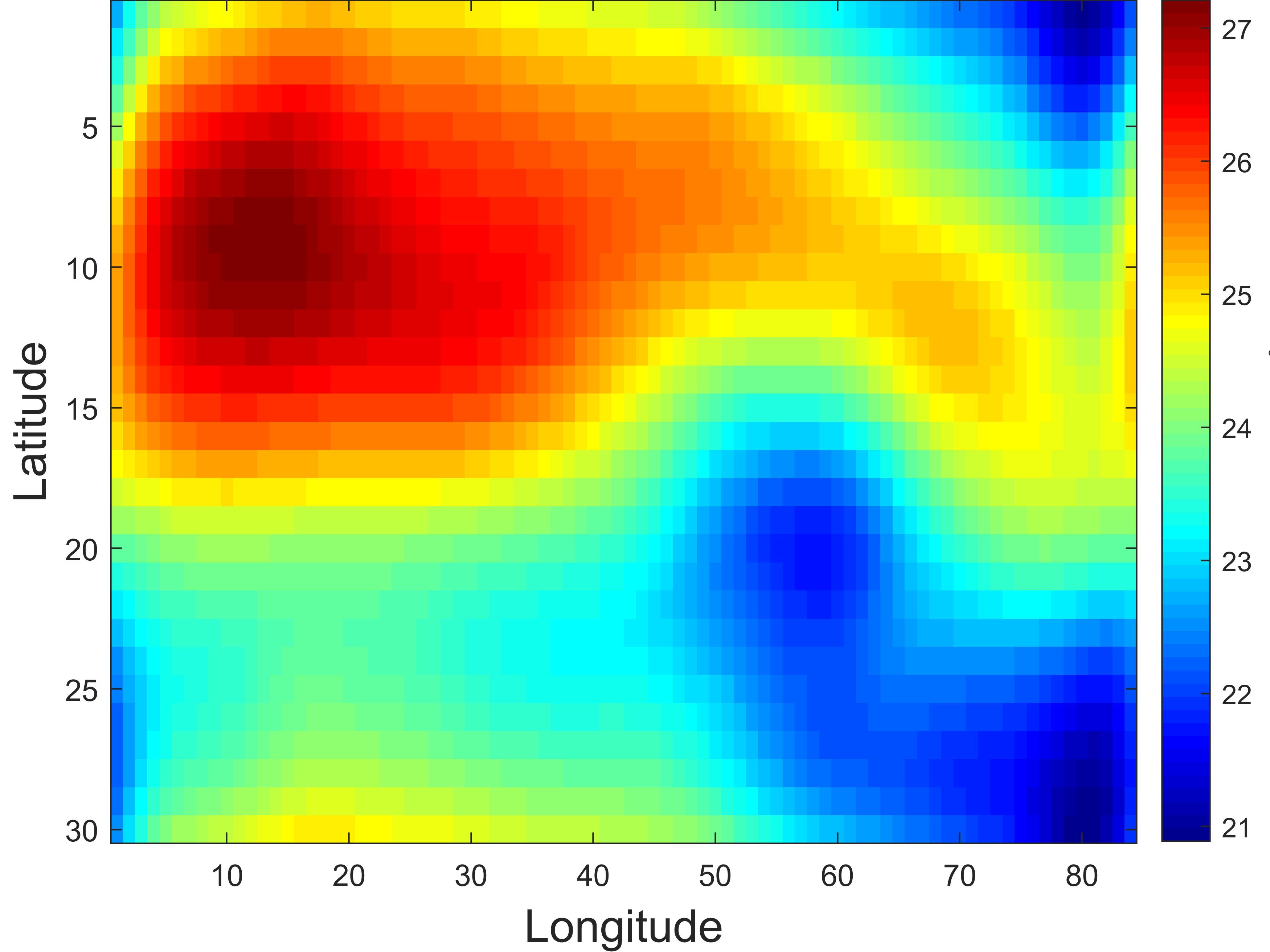}&
\includegraphics[width=20.3mm, height = 19.3mm]{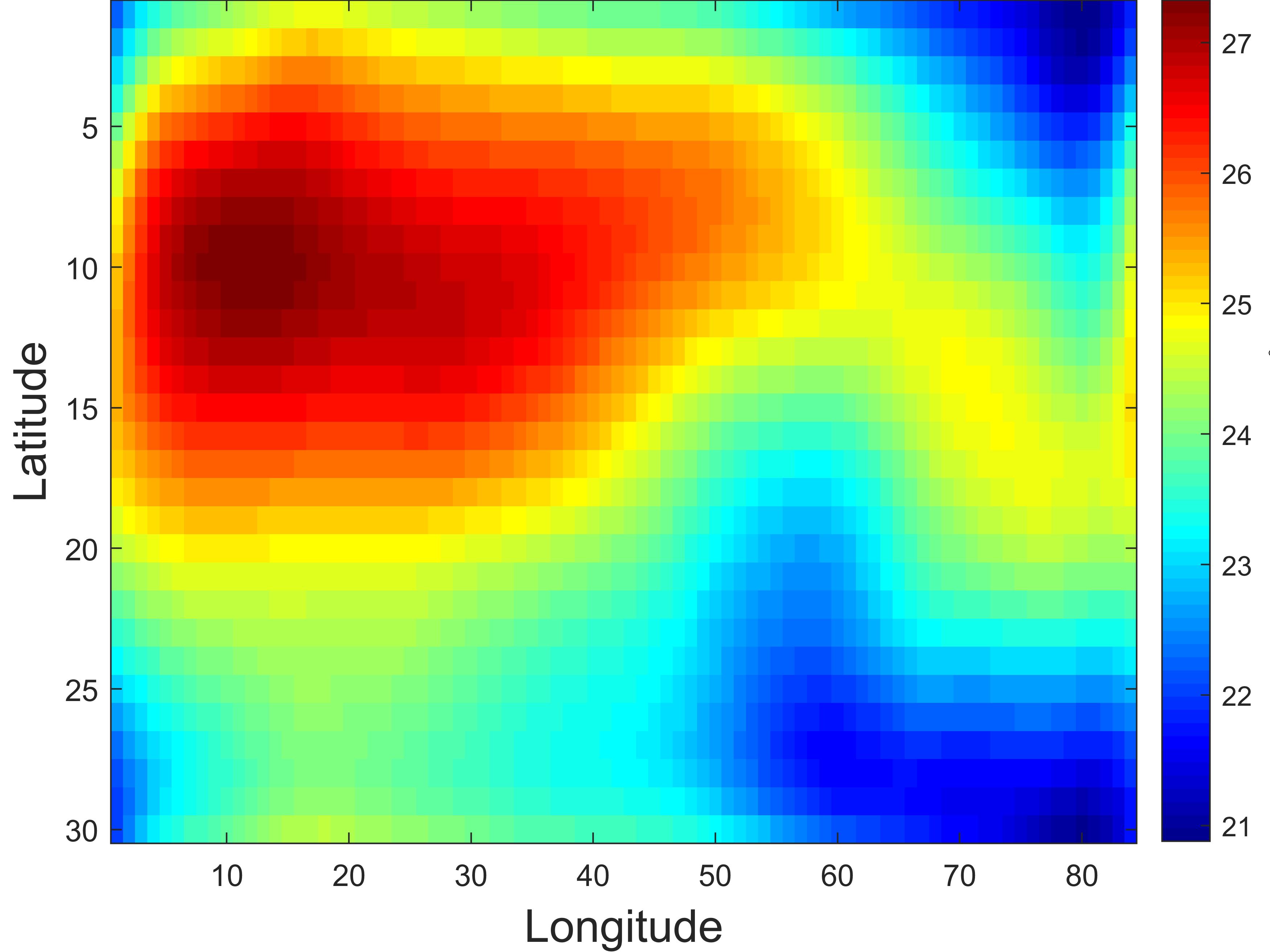}&
\includegraphics[width=20.3mm, height = 19.3mm]{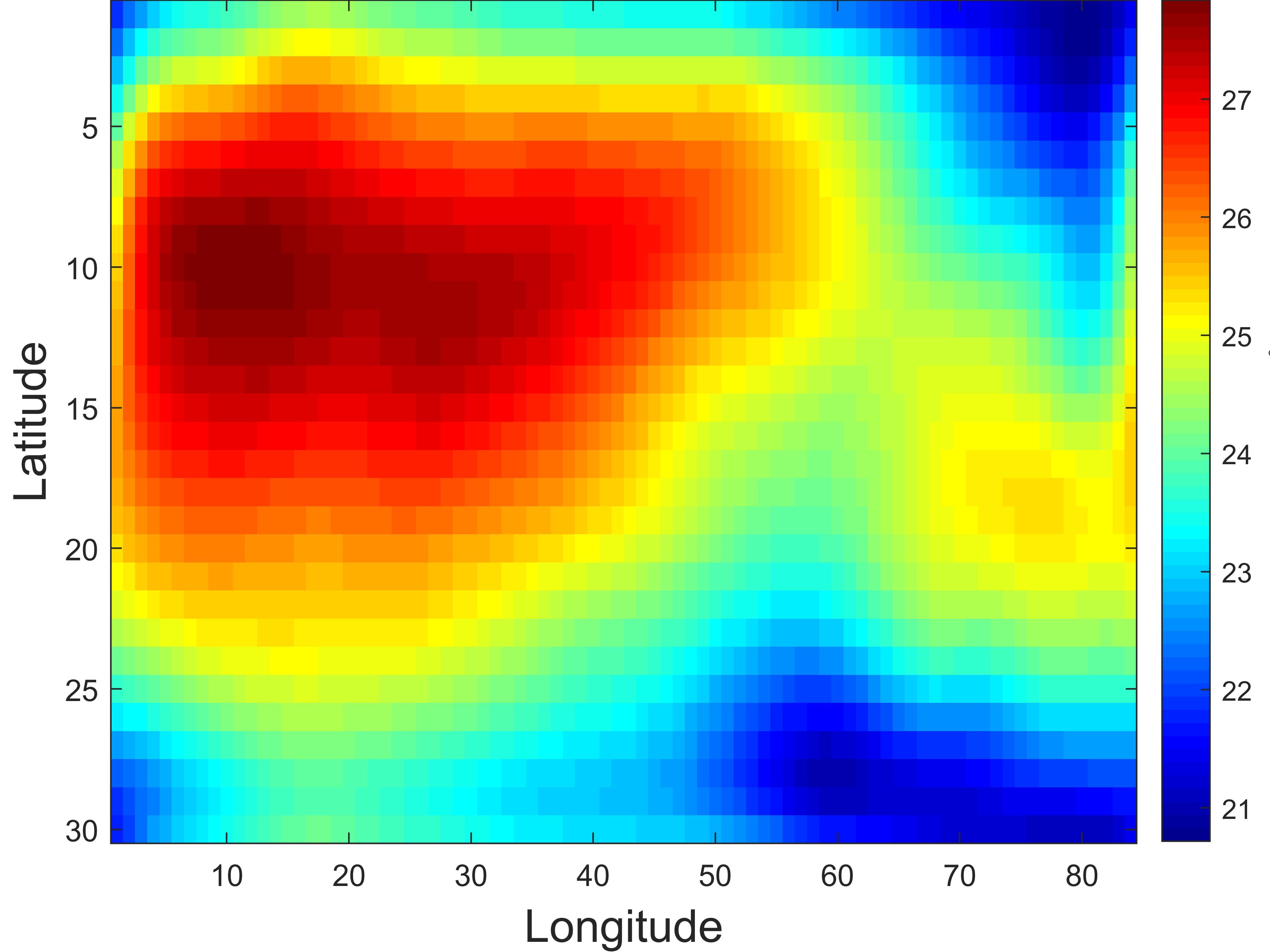}&
\includegraphics[width=20.3mm, height = 19.3mm]{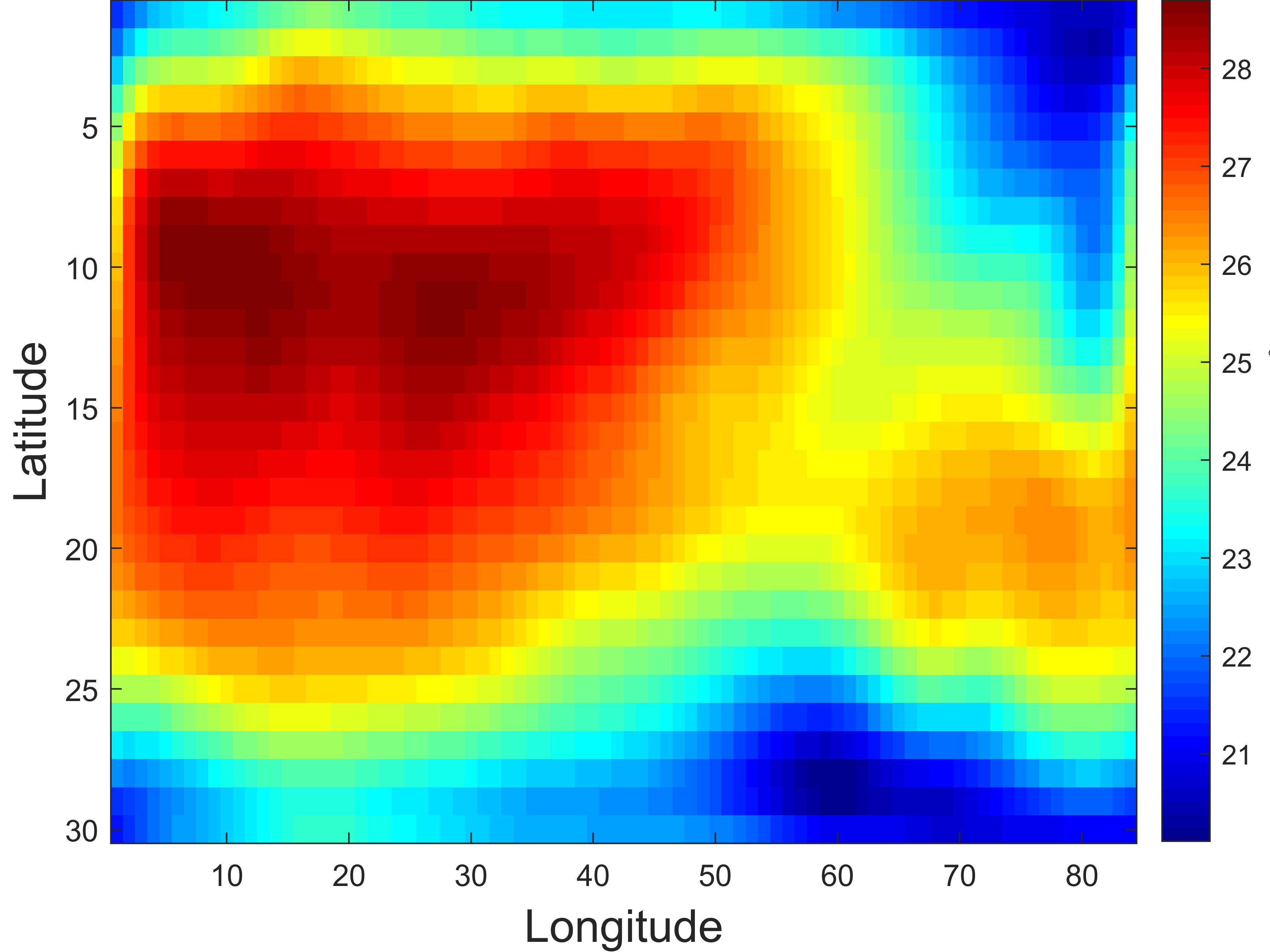}\\

\tiny  \textbf{MDT-Tucker 1}& \tiny  \textbf{MDT-Tucker 2} & \tiny  \textbf{ MDT-Tucker 3} & \tiny  \textbf{ MDT-Tucker 4} & \tiny  \textbf{MDT-Tucker 5} & \tiny  \textbf{ MDT-Tucker 6}\\

\includegraphics[width=20.3mm, height = 19.3mm]{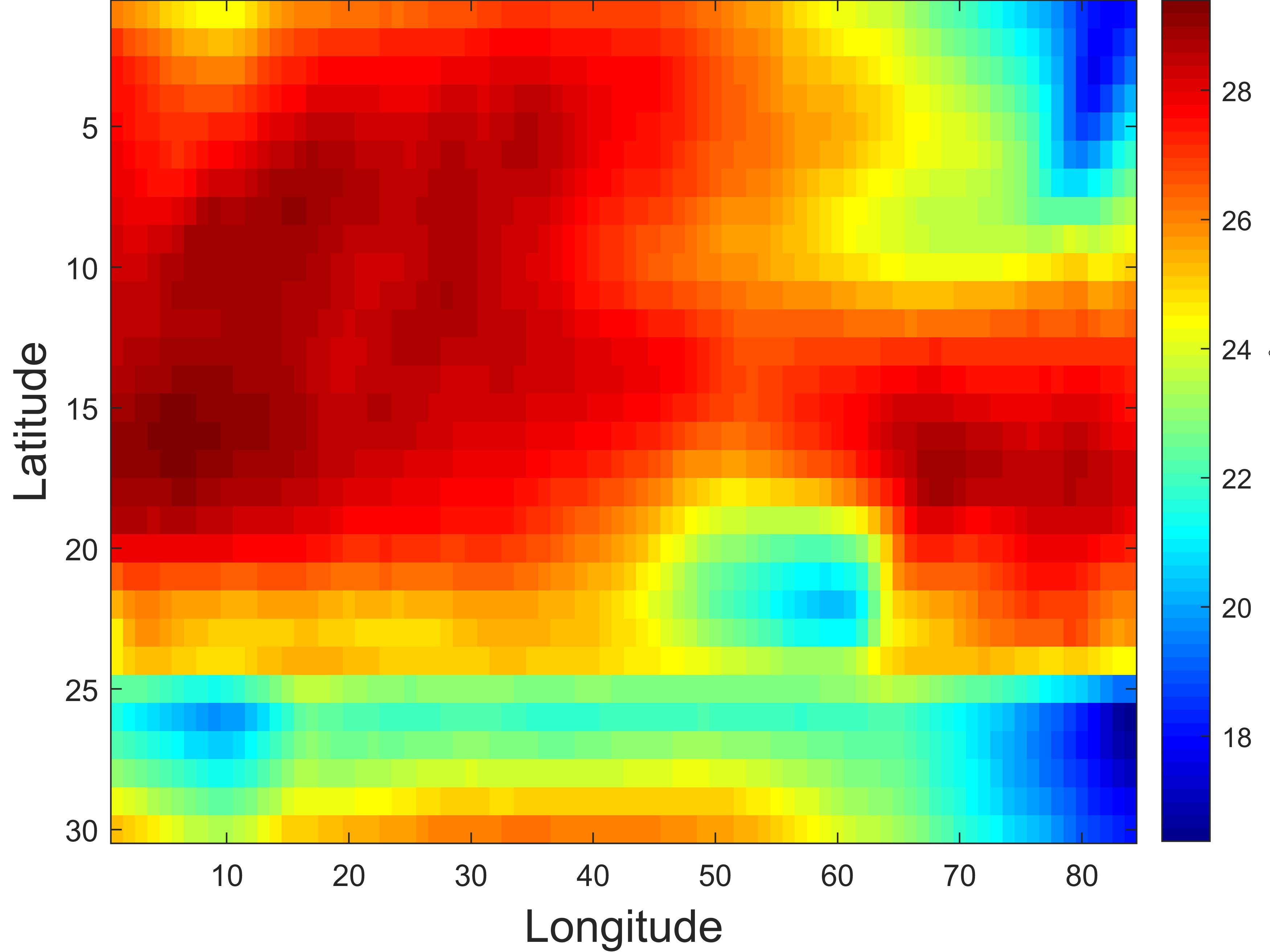}&
\includegraphics[width=20.3mm, height = 19.3mm]{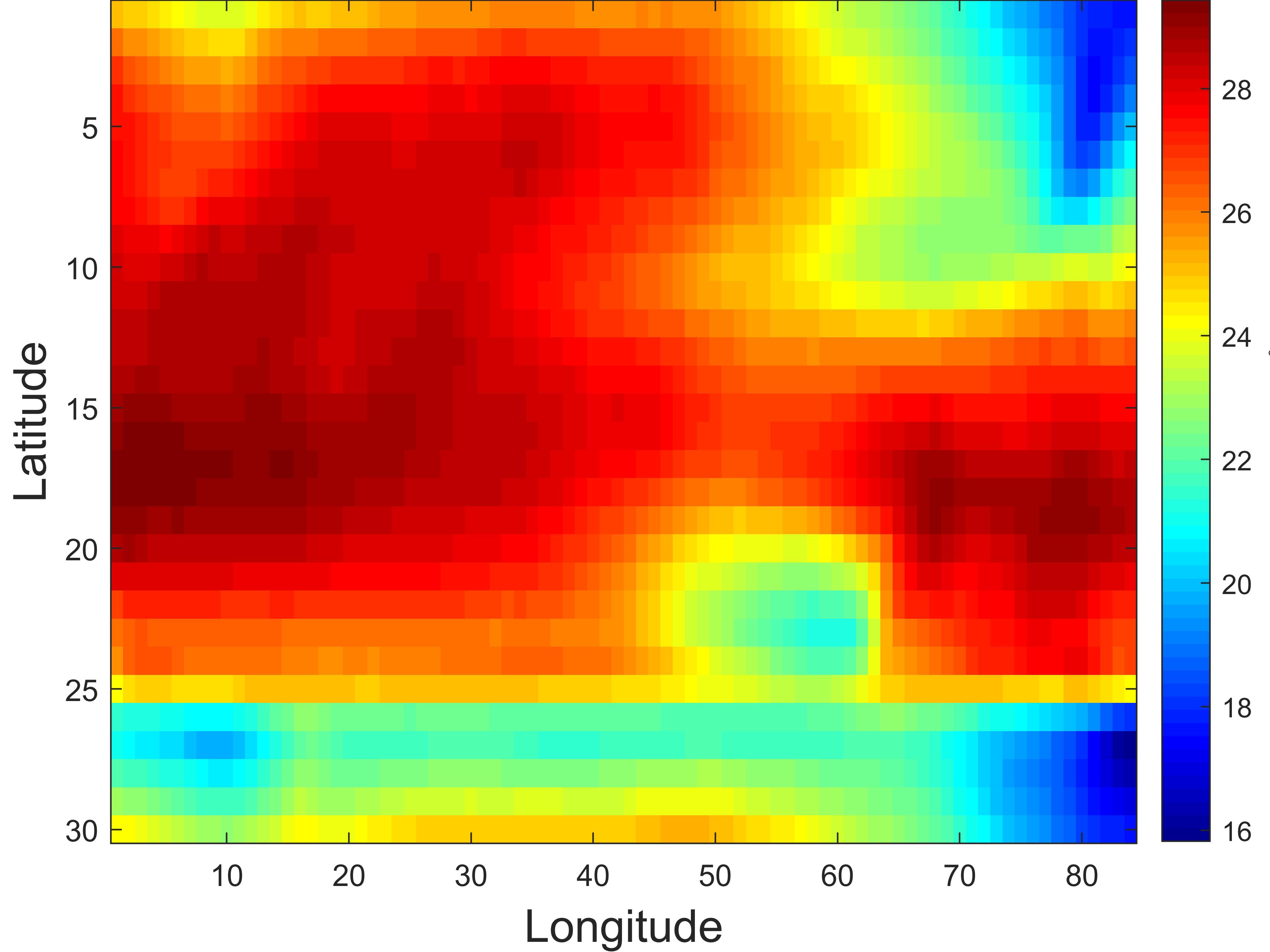}&
\includegraphics[width=20.3mm, height = 19.3mm]{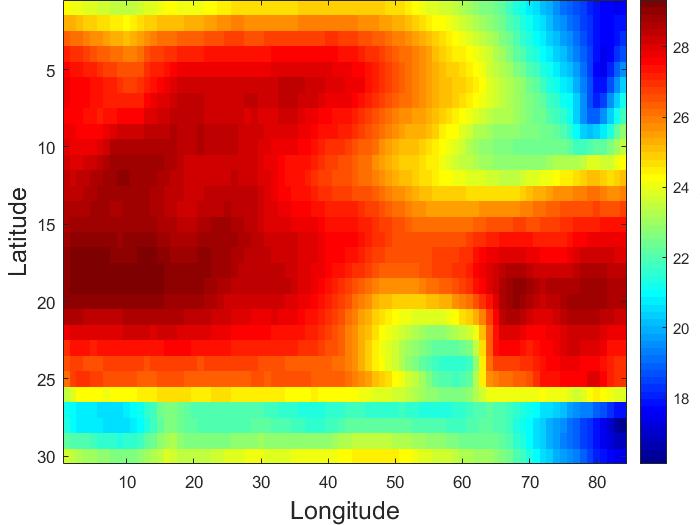}&
\includegraphics[width=20.3mm, height = 19.3mm]{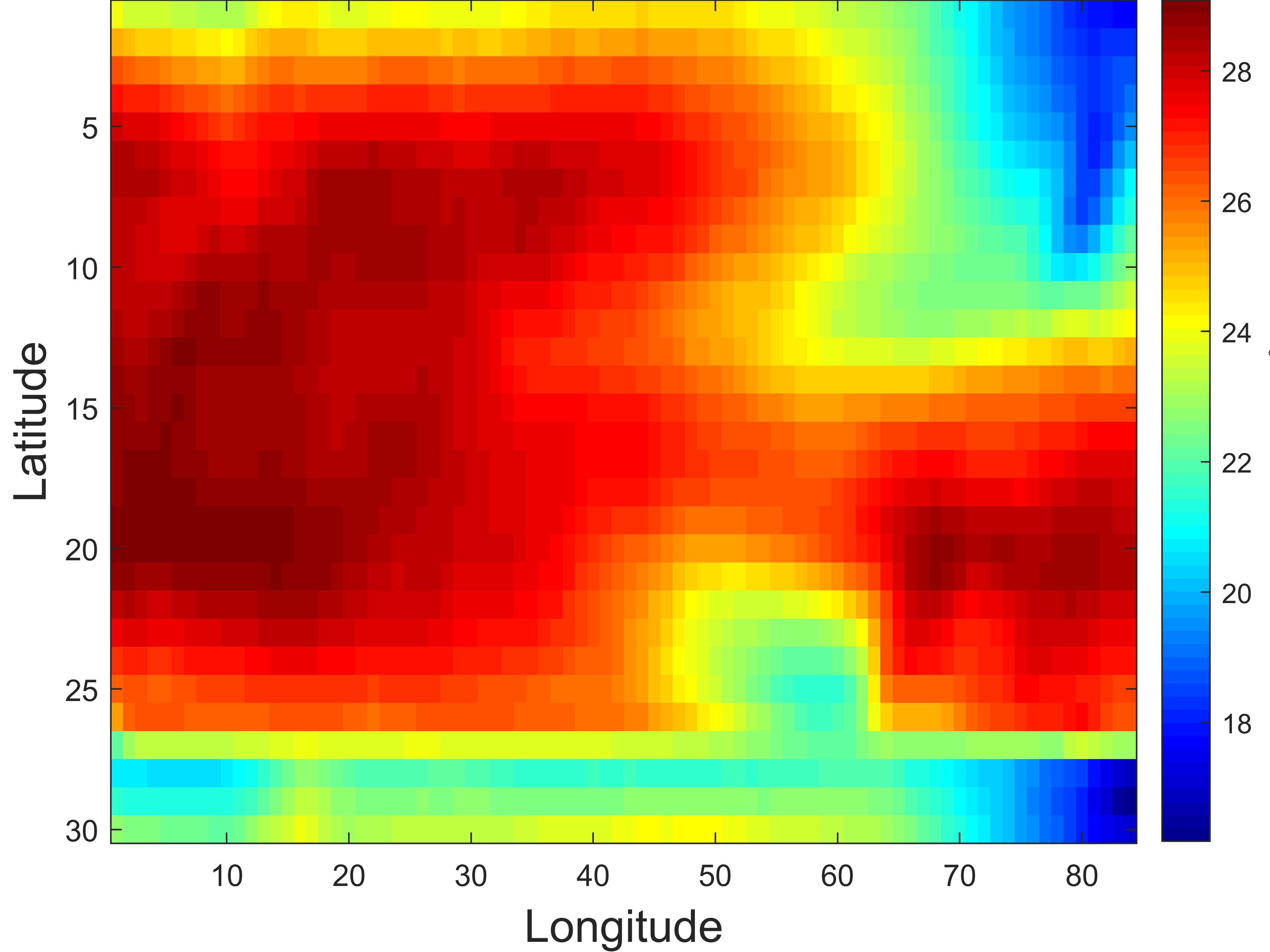}&
\includegraphics[width=20.3mm, height = 19.3mm]{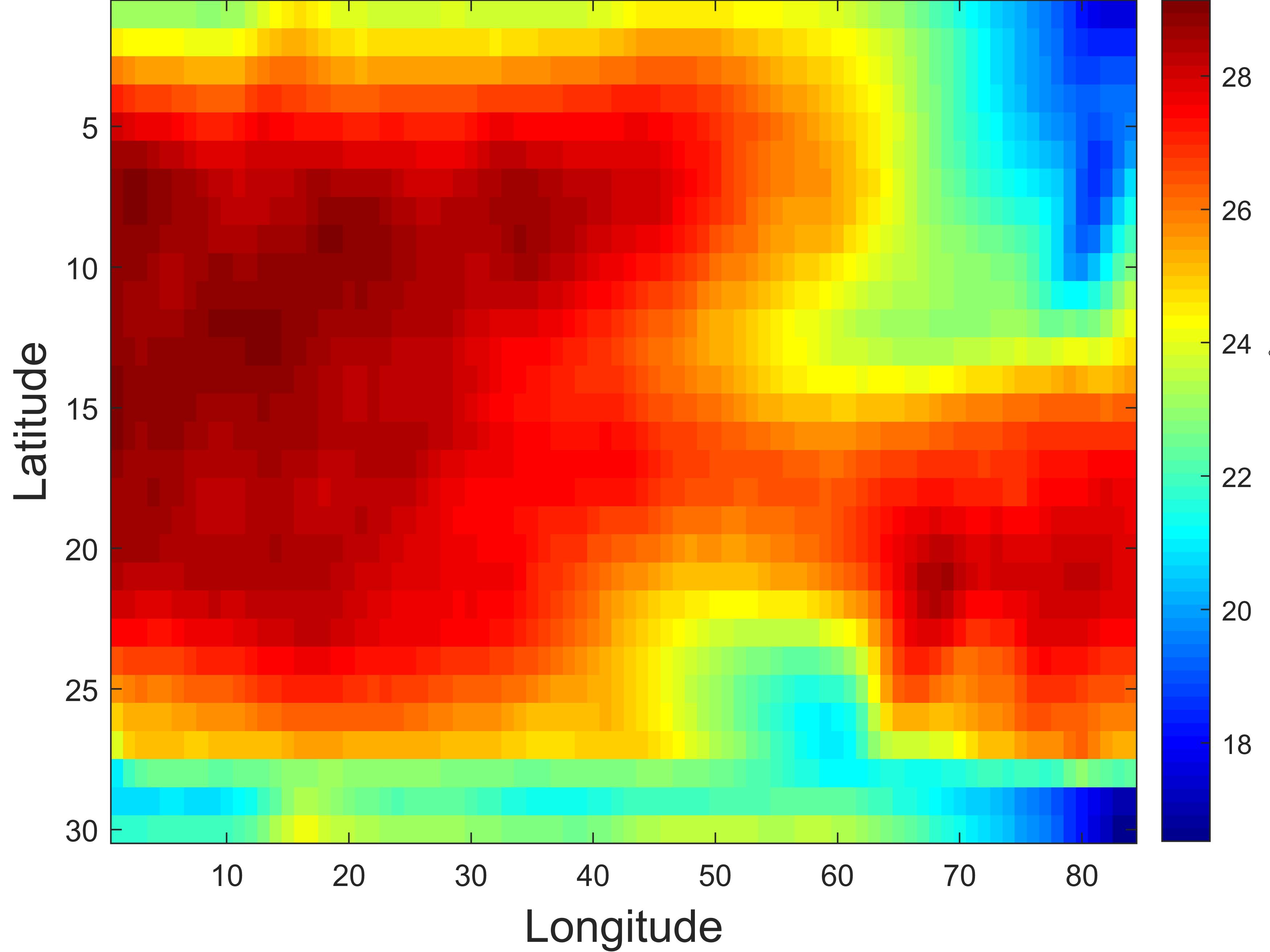}&
\includegraphics[width=20.3mm, height = 19.3mm]{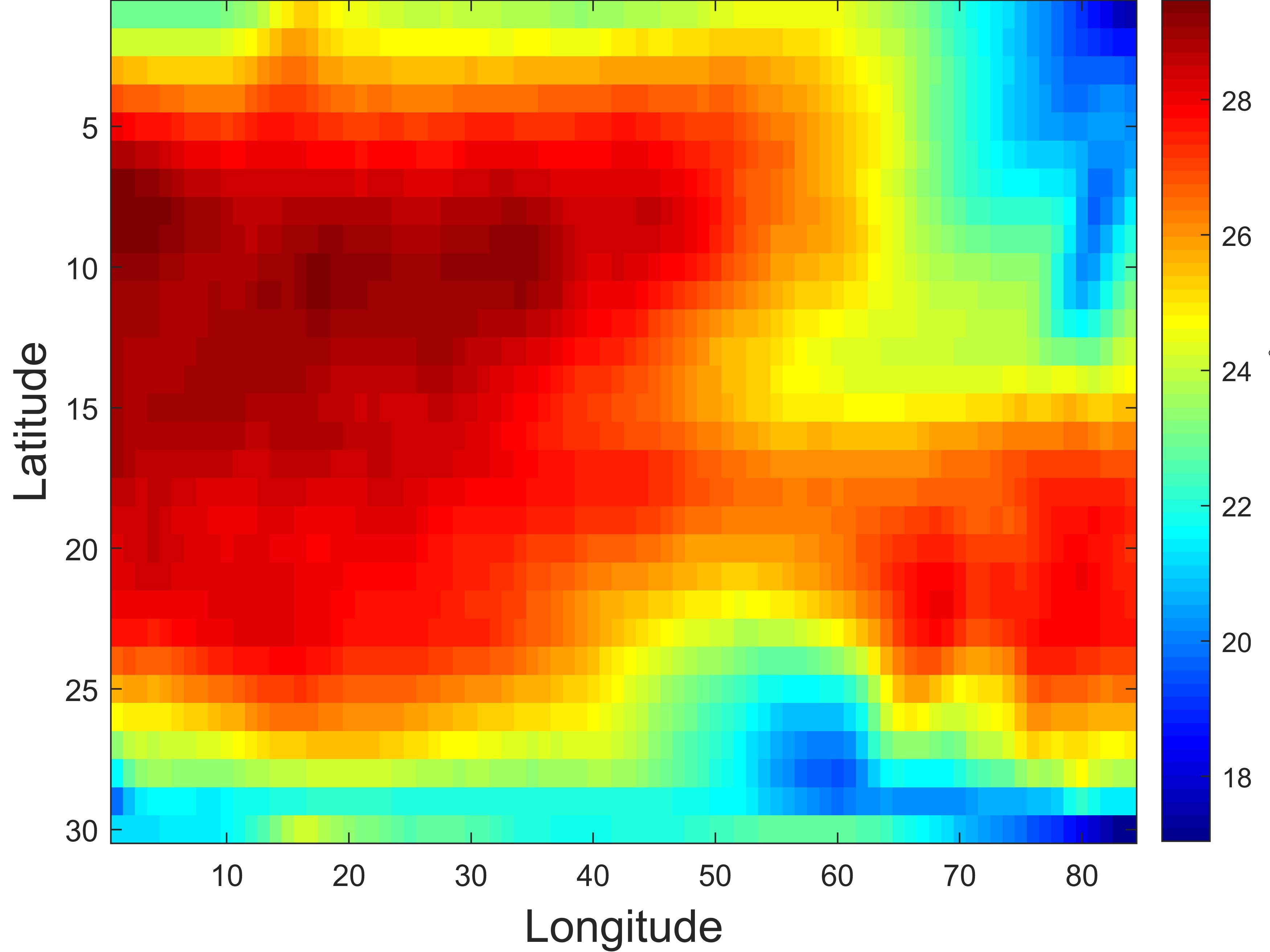}\\

\tiny  \textbf{LRTC-TIDT 1}& \tiny \textbf{LRTC-TIDT 2} & \tiny  \textbf{LRTC-TIDT 3} & \tiny  \textbf{LRTC-TIDT 4} & \tiny \textbf{LRTC-TIDT 5} & \tiny  \textbf{LRTC-TIDT 6}\\
\end{tabular}
\vspace{-0.2cm}
\caption{The predictive results of the LRTC-TIDT model and the MDT-Tucker
model on the  Pacific surface temperature dataset with h=6.}\label{fig.temperature_h6}
\vspace{-0.5cm}
\end{figure*}

\subsection{Discussion} This section presents only two discussions on model extensions. Other experimental analyses, including convergence and parameter analyses as well as multivariate time series recovery tests, are provided in Section 3 of the Supplementary Material (SM3).

\subsubsection{Corrected version using symmetric padding}
When the starting point $\M_1$ and the ending point $\M_t$ of the observed data are not connected, the lack of wrap-around continuity constrains the use of  temporal isometric delay-embedding transform and may degrade the model's accuracy. To address this issue, we apply symmetric padding to the observations
$\M = [\M_1, \dots, \M_t]^\top,$  
yielding
\begin{equation}
\M_{\mathrm{new}} = [\M_1, \dots, \M_t, \M_t, \dots, \M_1]^\top.
\end{equation}
Importantly, this symmetric padding strategy preserves the underlying recovery theory. We evaluate the proposed approach on the Abilene dataset under three sampling patterns (including Pattern-1 at $20\%$, Pattern-2 at $30\%$, and Pattern-3 at $50\%$), as  summarized in Table \ref{tab: Ab_padding}.  The results show that both RMSE and MAE are significantly improved relative to the standard scheme. This indicates that symmetric padding allows LRTC-TIDT to achieve superior performance when the observed endpoints are not connected.


\begin{table}[!htbp]
\renewcommand{\arraystretch}{1.4}
\setlength\tabcolsep{3.5pt}
\footnotesize
 \vspace{-0.3cm}
  \caption{Performance comparison (in MAE/RMSE) of LRTC-TIDT and its symmetric padding version under different non-random sampling scenarios using Abilene dataset. 
  }\label{tab: Ab_padding}
  \centering
  \vspace{-0.2cm}
\begin{tabular}{l|ccc}
     \hline
    Method/Sampling & Pattern-1  & Pattern-2  & Pattern-3\\
     \hline
     LRTC-TIDT (Standard)& 0.98/1.94 & 1.46/2.99 & 1.13/2.53 \\
    LRTC-TIDT (Symmetric Padding ) & \textbf{0.85}/\textbf{1.78} & \textbf{0.94}/\textbf{2.30} & \textbf{0.92}/\textbf{2.01} \\
     \hline
\end{tabular}
\vspace{-0.2cm}
\end{table}

\subsubsection{Extended version in the generalized T-SVD framework}
Within the t-SVD framework under arbitrary invertible linear transforms, the transformed form of a tensor $ \mathcal{Z}\in\mathbb{R}^{ n_1\times\cdots\times  n_d}$ is defined as
$\mathcal{Z}_\mathfrak{L} := \mathfrak{L}(\mathcal{Z}) 
= \mathcal{Z}\times_3 \mathrm{L}_{n_3}\times_4 \cdots \times_{p+2} \mathrm{L}_{n_d},$
where $\mathrm{L}_{n_j} \in \mathbb{R}^{n_j \times n_j}$ denotes an arbitrary invertible linear transform matrix, such as the \textit{Discrete Fourier Transform} (DFT) matrix or the \textit{Discrete Cosine Transform} (DCT) matrix.
These transform matrices satisfy$
(L_{n_d}^* \otimes \cdots \otimes L_{n_3}^*)(L_{n_d} \otimes \cdots \otimes L_{n_3}) 
= \ell \cdot I_{n_3 \cdots n_d},$
where $\otimes$ denotes the Kronecker product and $I_{n_3 \cdots n_d}$ is the identity matrix of compatible size. 
For example, $\ell = \prod_{j=3}^d n_j$ for DFT matrices $F_{n_j}$, since $F_{n_j}^* F_{n_j} = n_j I_{n_j}$, whereas $\ell = 1$ for DCT matrices $C_{n_j}$, as $C_{n_j}^* C_{n_j} = I_{n_j}$, for $j = 3, \dots, d$.
Under the  generalized t-SVD framework,  tensor nuclear norm (TNN) is defined as
$
\begin{aligned}
\|\mathcal{Z}\|_{\circledast,\mathfrak{L}}:=  \norm{\operatorname{bdiag}(\mathcal{Z}_\mathfrak{L})}_*/\ell,
\end{aligned}
$
where  $\|\cdot\|_*$ denotes the nuclear norm of a matrix. Based on this, we obtain the LRTC-TIDT model under the generalized t-SVD framework:
\begin{equation}\label{LRTC-TIDT-L}
\min_{\X\in \mathbb{R}^{t \times n_1 \times\cdots\times n_p}} ~\norm{\mathcal{H}_k(\X)}_{\circledast,\mathfrak{L}},\ \text{s.t.} \ \Pomega(\X) =\Pomega(\M).
\end{equation}
It is evident that when the transform $\mathfrak{L}$ is chosen as the Discrete Fourier Transform (DFT), model (\ref{LRTC-TIDT-L}) degenerates to the standard LRTC-TIDT model (\ref{LRTC-TIDT}). Note that the recovery theory of LRTC-TIDT under the generalized t-SVD framework can be derived in a manner analogous to the standard case.
In our experiments, we consider three invertible linear transforms, namely the FFT, the DCT, and a random orthogonal transform (ROT). We conduct evaluations on the Pacific dataset under prediction horizons of 4, 6, and 8. As illustrated in Fig.~\ref{fig.fftdctrot}, the results  demonstrate that the DFT-based transform consistently outperforms the other transforms, highlighting its superiority under the standard t-SVD framework.
\begin{figure}[!htbp]
\renewcommand{\arraystretch}{0.5}
\setlength\tabcolsep{0.5pt}
\centering
\vspace{-0.2cm}
\begin{tabular}{ccccccc}
\centering
\includegraphics[width=40.3mm, height = 38.3mm]{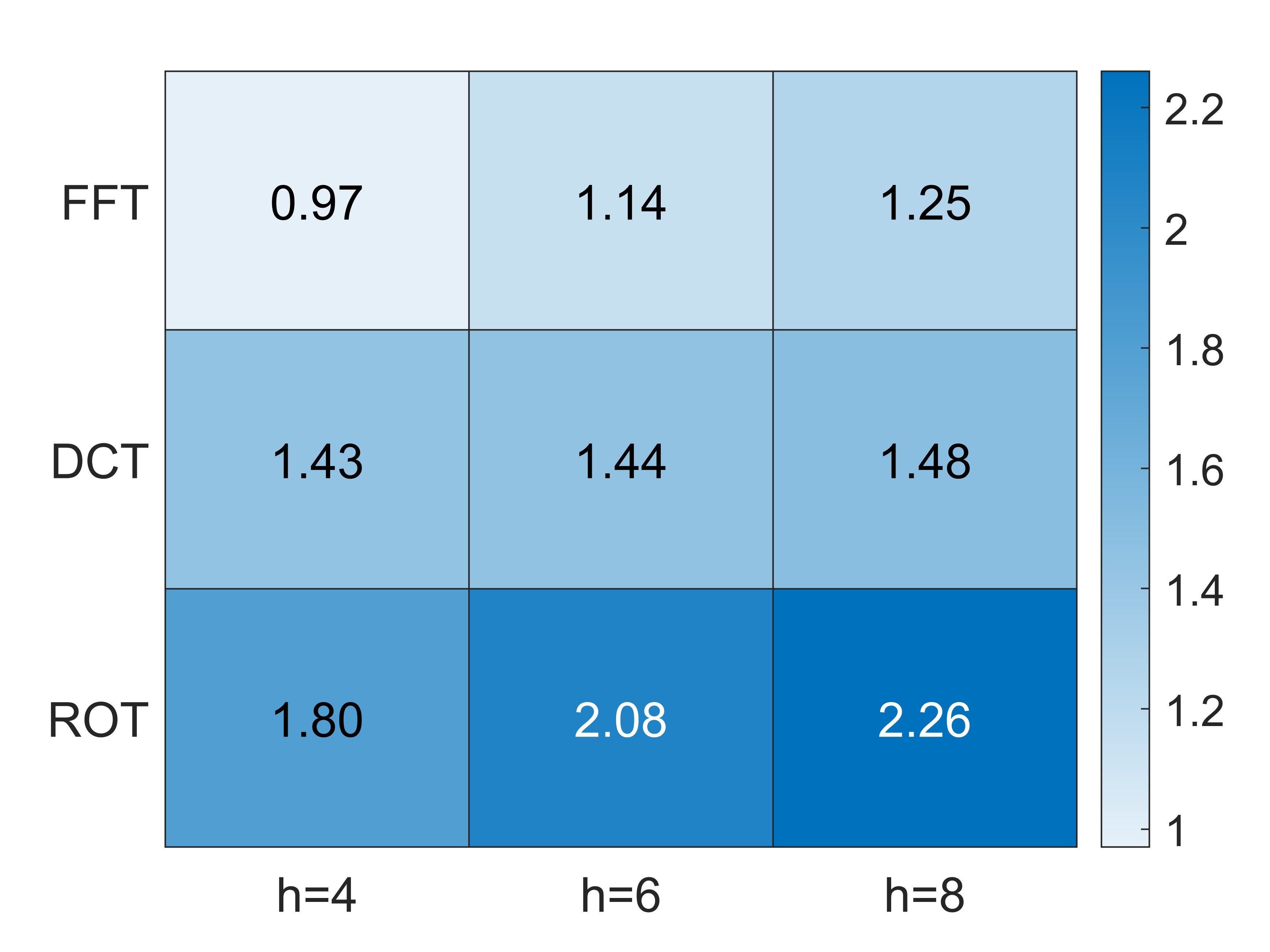}&
\includegraphics[width=40.3mm, height = 38.3mm]{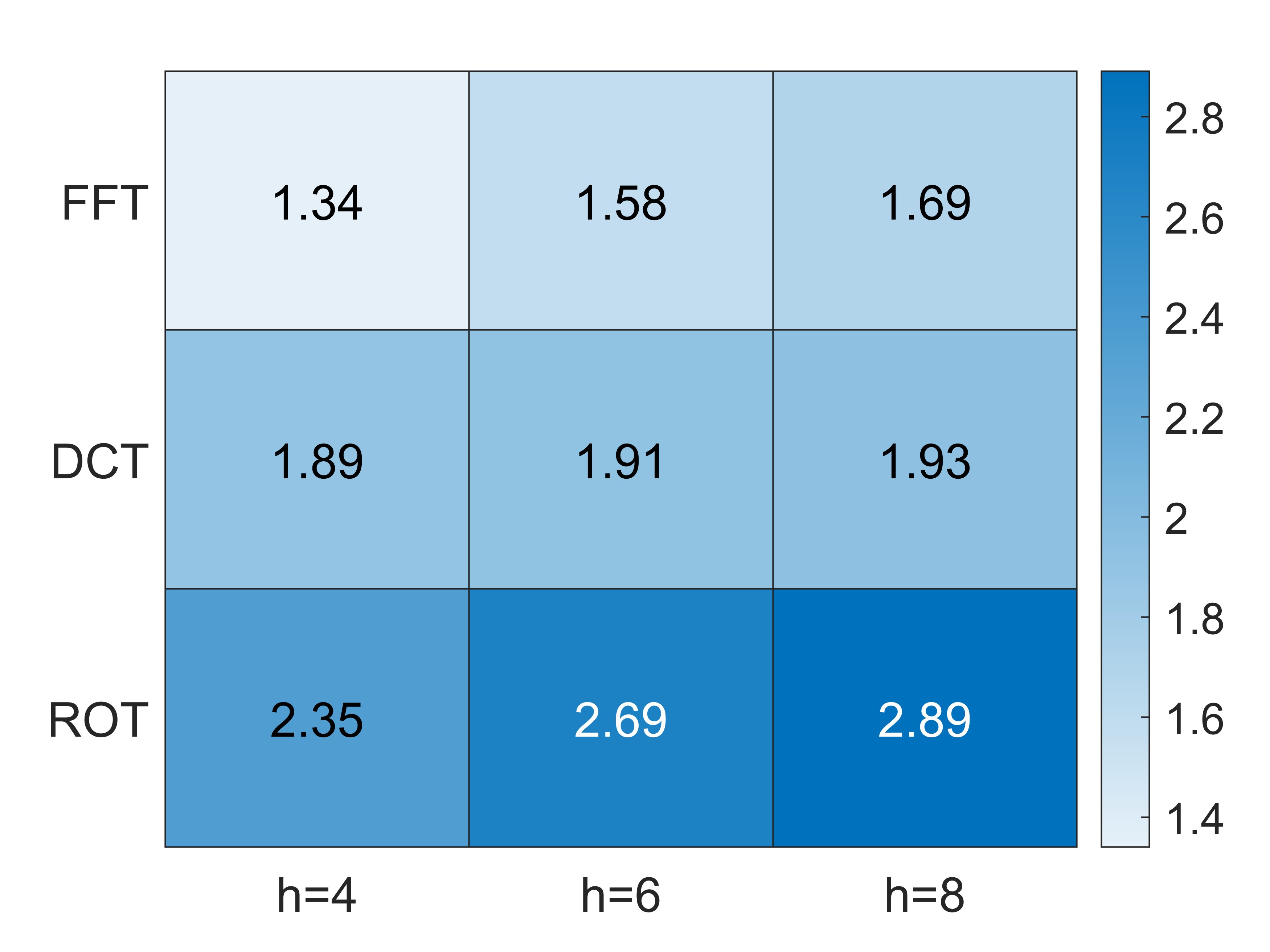}\\
\scriptsize \textbf{MAE}& \scriptsize \textbf{RMSE}\\
\end{tabular}
\vspace{-0.2cm}
\caption{The MAE (left) and RMSE (right) of Pacific dataset prediction  using  the  LRTC-TIDT  model  under different invertible linear  transforms.}\label{fig.fftdctrot}
\vspace{-0.2cm}
\end{figure}


\section{Conclusion}\label{sec_conclusion}

This study focuses on recovering multidimensional time series from partially observed data with non-random sampling, considering both noiseless and noisy scenarios. Specifically, we propose the LRTC-TIDT model, which leverages the high-order tensor nuclear norm after applying the temporal isometric delay-embedding transform. The model demonstrates excellent recovery performance across various time-series settings. More importantly, we establish an exact recovery guarantee theory for non-random sampling based on the LRTC-TIDT framework, revealing a significant relationship between recovery performance and the minimum temporal sampling rate.
This work not only provides an effective recovery tool for non-random missing data scenarios but also offers valuable insights into time-series data modeling. Extensive experiments on non-random missing network flow reconstruction, noisy urban traffic data estimation, and short-term temperature field prediction tasks validate the effectiveness of the proposed model in recovering and forecasting multidimensional time series.
However, as mentioned earlier, the temporal Hankel low-rank property relies on the smoothness or periodicity of the temporal dimension. In the absence of these characteristics, the proposed model may become less effective, which motivates us to explore more learning-based norm minimization approaches in future work \cite{liu2022time}. 

\appendix
\section{Proof of Lemma 4.1 \& 4.2}
\subsection{Proof of Lemma 4.1}\label{appedix:lemma4.1}
\begin{proof}
 According to the definition of isometric delay-embedding transform, we have
\begin{equation}
\mathcal{H}_k(\bm{m})=\frac{1}{\sqrt{k}}
\left[\begin{array}{cccc}
m_1 & m_2 & \cdots & m_{k} \\
m_2 & m_3 & \cdots & m_{k+1} \\
\vdots & \vdots & \vdots & \vdots \\
m_t & m_{1} & \cdots & m_{k-1}
\end{array}\right], \text{where} ~ 
\bm{m}=\left[\begin{array}{c}
m_1  \\
m_2\\
\vdots \\
m_t 
\end{array}\right] .
\end{equation}
Then $[\mathcal{H}_k(\bm{m})](:,j+1,:,\cdots,:)$ is given by $\mathcal{S}^{j}(\bm{m})$, i.e.,
$
\mathcal{H}_k(\bm{m})=\frac{1}{\sqrt{k}} \left[\bm{m}, \mathcal{S}(\bm{m}), \cdots, \mathcal{S}^{k-1}(\bm{m})\right],$
where  $\mathcal{S}(\bm{m})=[m_2,m_3,\cdots,m_t,m_1]^{\mathrm{T}}$.
 We divide $\mathcal{H}_{k}(\bm{m})$ into $r$ submatrices $\mathcal{H}_{k}(\bm{m})$ $=\frac{1}{\sqrt{k}}[A_1,A_2,\cdots,A_r]$, where  $A_i$ have $b_i$ columns with $1\leq b_i\leq \left\lceil\frac{k}{r}\right\rceil $ ($i=1, \cdots, r$) and $\sum_{i=1}^{r}b_i=k$. For every $A_i\in\mathbb{R}^{t\times b_i}$, a rank-1 matrix $\hat{A}_i\in\mathbb{R}^{t\times b_i}$ is contructed by repeating the first column of $A_i$, then we get
\begin{equation}
\begin{aligned}
\|A_i-\hat{A}_i\|_F &   =\sqrt {\norm{\mathcal{S}(\bm{m})-\bm{m}}_F^2+\cdots+\norm{\mathcal{S}^{b_i-1}(\bm{m})-\bm{m}}_F^2}\leq \sqrt {\sum_{j=1}^{b_i-1}j^2 \eta(\bm{m})^2}\\
&=\sqrt{\frac{b_i(b_i-1)(2b_i-1) }{6}}\eta(\bm{m})\leq b_i\sqrt{\frac{(b_i-1)}{3}}\eta(\bm{m}),
\end{aligned}
\end{equation} 
where the first inequality holds because 
\begin{equation}
\begin{aligned}
\norm{\mathcal{S}^j(\bm{m})-\bm{m}}_F &=\norm{\mathcal{S}^j(\bm{m})-\mathcal{S}^{j-1}(\bm{m})+\mathcal{S}^{j-1}(\bm{m})-\mathcal{S}^{j-2}(\bm{m})+\cdots+\mathcal{S}(\bm{m})-\bm{m}}_F\\
&\leq \norm{\mathcal{S}^j(\bm{m})-\mathcal{S}^{j-1}(\bm{m})}_F+\cdots+\norm{\mathcal{S}(\bm{m})-\bm{m}}_F \leq j \eta(\bm{m})
\end{aligned}
\end{equation}
Let $Z= \frac{1}{\sqrt{k}}[\hat{A}_1,\hat{A}_2,\cdots, \hat{A}_r]$, obviously, rank$(Z)\leq r$. Hence,
\begin{equation}
\begin{aligned}
&\epsilon_r(\mathcal{H}_k(\bm{m})) \leq \|\mathcal{H}_k(\bm{m})-Z\|_F=  \frac{1}{\sqrt{k}}\sqrt{\sum_{i=1}^{r}\|A_i-\hat{A}_i\|^2_F}\\
& \leq \frac{1}{\sqrt{k}}\sqrt{\sum_{i=1}^rb_i^2\frac{b_i-1}{3}\eta(\bm{m})^2}\leq\frac{1}{\sqrt{k}}\left\lceil\frac{k}{r}\right\rceil\sqrt{\frac{k-r}{3}}\eta(\bm{m})= \sqrt{\frac{k-r}{3k}}\left\lceil\frac{k}{r}\right\rceil \eta(\bm{m}). 
\end{aligned}
\end{equation} 
\end{proof}

\subsection{Proof of Lemma 4.2}\label{appedix:lemma4.2}
\begin{proof}
Set $a=\left\lceil\frac{k}{\tau}\right\rceil$, then decompose $\mathcal{H}_k(\bm{m})$ into the concatenation of $a$ subtensors, namely $\mathcal{H}_k(\bm{m})=\frac{1}{\sqrt{k}}\left[A_1, A_2, \cdots, A_a\right]$,
such that  $A_1,A_2,...,A_{a-1} \in\mathbb{R}^{t \times \tau} $  and $A_a \in\mathbb{R}^{t \times (k-(a-1)\tau)} $.
We consider $Z=\frac{1}{\sqrt{k}}[A_1,A_1,...,A_1^{'}]$, where $A_1^{'}=[A_1](:,1:k-(a-1)\tau)$. Since $\operatorname{rank}(Z) \leq r$,
\begin{equation}
 \varepsilon_{\tau}(\mathcal{H}_k(\bm{m}))  \leq \|\mathcal{H}_k(\bm{m})-Z\|_F
 \leq \frac{1}{\sqrt{k}}\sum_{i=1}^a\normF{A_i-A_1} \leq \frac{\tau}{\sqrt{k}}(\left\lceil\frac{k}{\tau}\right\rceil-1)\beta_{\tau}(\bm{m}).
\end{equation}
\end{proof}

\section{ Multidimensional extension  of Lemma 4.1 \& 4.2}\label{appde:b}
In this section, we will demonstrate that periodicity and smoothness along the time dimension of the 
original time series $ \M \in \mathbb{R}^{t \times n_1 \times \cdots \times n_p}$ can lead to the low rank of the temporal Hankel tensor $\mathcal{H}_k(\M) \in \mathbb{R}^{t \times k \times n_1 \times \cdots \times n_p}$. In other word, the temporal Hankel low-rankness can stem from the periodicity and smoothness of the original data.
Let $\widetilde{\M}\in \mathbb{R}^{t \times 1\times n_1 \times\cdots\times n_p }$ be a shape variant of $\M\in \mathbb{R}^{t \times n_1 \times\cdots\times n_p}$(i.e.,$\widetilde{\M}=reshape(\M,t,1,n_1,...,n_p)$).
Then $[\mathcal{H}_k(\M)](:,j+1,:,\cdots,:)$ is given by $\mathcal{S}^{j}(\M)$, i.e.,
$\mathcal{H}_k(\M)=\left[\widetilde{\M}, \mathcal{S}(\widetilde{\M}), \mathcal{S}^{2}(\widetilde{\M}), \cdots, \mathcal{S}^{k-1}(\widetilde{\M})\right],$
where $\mathcal{S}$ is an operator that delays the elements of a tensor by one position along the first dimension; namely,
\begin{equation}
\mathcal{S}(\widetilde\M)=\left[\begin{array}{c}
\M_2  \\
\M_3\\
\vdots \\
\M_{1} 
\end{array}\right],
\widetilde\M=\left[\begin{array}{c}
\M_1  \\
\M_2 \\
\vdots \\
\M_t 
\end{array}\right].
\end{equation}
Similarly, we define $\mathcal{N}(\M)$ by circularly shifting $\widetilde{M}$ by $\tau$ positions along the first dimension, namely,
\begin{equation}
\mathcal{N}(\widetilde\M)=\left[\begin{array}{c}
\M_{1+\tau} \\
\M_{2+\tau}\\
\vdots \\
\M_{t+\tau}\\
\end{array}\right],
\text{ assuming that $\tau$ is the period, and $\M_i = \M_{i-t}$ for $i > t$.}
\end{equation}
To reach a rigorous conclusion, we consider the rank-$r$ approximation error of temporal Hankel tensor $\mathcal{H}_k(\M) \in \mathbb{R}^{t \times k \times n_1 \times\cdots\times n_p} $, which is denoted by $\varepsilon_r(\mathcal{H}_k(\M) )$ and defined as
\begin{equation}
\varepsilon_r(\mathcal{H}_k(\M))=\min _{\Z \in \mathbb{R}^{t \times k \times n_1 \times\cdots\times n_p}}\|\mathcal{H}_k(\M)-\Z\|_F, \text { s.t. } \operatorname{rank}_{\operatorname{t-SVD}}(\Z) \leq r.
\end{equation}

\begin{Prop}\label{smoothness-tensor}
For a multidimensional time‐series tensor $\M \in \mathbb{R}^{t \times n_1 \times \cdots \times n_p}$, its smoothness can be characterized by: $\eta(\M)=\|\M-\mathcal{S}(\M)\|_F$. For the transformed Hankel tensor $\mathcal{H}_{k}(\M)\in\mathbb{R}^{t \times k \times n_1 \times \cdots \times n_p}$,
its  rank-$r$ approximate error satisfies
\begin{equation}
\epsilon_r(\mathcal{H}_{k}(\M))\leq \sqrt{\frac{k-r}{3k}}\left\lceil\frac{k}{r}\right\rceil \eta(\M).
\end{equation}
\end{Prop}
\begin{proof}
 The proof is provided in Section 1.1 of the Supplementary Material (SM1.1).   
\end{proof}

\begin{Prop}\label{periodicity-tensor}
For a multidimensional time‐series tensor $\M \in \mathbb{R}^{t \times n_1 \times \cdots \times n_p}$, its periodicity  can be characterized by: $\beta_{\tau}(\M)=\|\M-\mathcal{N}(\M)\|_F$. For the transformed Hankel tensor $\mathcal{H}_{k}(\M)\in\mathbb{R}^{t \times k \times n_1 \times \cdots \times n_p}$,
its  rank-$r$ approximate error satisfies
\begin{equation}
\epsilon_r(\mathcal{H}_{k}(\M))\leq \frac{\tau}{\sqrt{k}}(\left\lceil\frac{k}{\tau}\right\rceil-1)\beta_{\tau}(\M),
\end{equation}
\end{Prop}
\begin{proof}
 The proof is provided in Section 1.2 of the Supplementary Material (SM1.2).   
\end{proof}

\section{Proof of  Theorem 4.2}\label{appendix:c}
The proof of Theorem \ref{thm: approximate non-random tensor completion} consists of four main parts:
1. Defining key concepts such as the projection of the temporal Hankel tensor and the subgradient of the temporal Hankel tensor nuclear norm;
2. Establishing the relationships among various operator norms (Lemmas \ref{lemmauv}, \ref{lemmatuv}, \ref{lemmaeq}, and \ref{lemma2.18});
3. Constructing the dual certificates (Lemma \ref{lemmadu});
4. Completing the recovery guarantee by verifying the dual conditions (Lemma \ref{lemmaxm}). Once Lemma \ref{lemmaxm} is established, Theorem \ref{thm: approximate non-random tensor completion} can be derived by following the proof strategy in \cite{candes2010matrix}. The detailed proof of Theorem \ref{thm: approximate non-random tensor completion}, together with the proofs of Lemmas \ref{lemmauv}-\ref{lemmaxm}, is provided in Section 2 of the Supplementary Material (SM2).

 

\begin{defn} [Projection in temporal Hankel transform space]\label{TUV}
Let $ \mathcal{H}_k(\M) \in \mathbb{R}^{t \times k \times n_1\cdots \times n_p}$ with $\operatorname{rank}_{\operatorname{t-SVD}}(\mathcal{H}_k(\M))=r$, and its skinny t-SVD is $\mathcal{H}_k(\M) = \U*\mathcal{S}*\V^\mathrm{T}$, where $\U \in \mathbb{R}^{t \times r \times n_1 \cdots\times n_p }$ and $\V \in \mathbb{R}^{k \times r \times n_1 \cdots\times n_p }$ are the left and right singular tensor, respectively.  Define $\mathbb{T}$ by the set
$\mathbb{T}=\{\U * \X^\mathrm{T}+\Y *\V^\mathrm{T} \mid \X \in \mathbb{R}^{k \times r \times n_1 \times \cdots \times n_p},
\Y \in 
\mathbb{R}^{t \times r \times n_1 \times \cdots \times n_p}\}$
and by $\mathbb{T}^{\perp}$ its orthogonal complement.
For any $\Z \in \mathbb{R}^{t \times k \times n_1\cdots \times n_p}$, the projections onto $\mathbb{T}$ and its complementary set $\mathbb{T}^{\perp}$ are respectively denoted as
$$
\mathcal{P}_{\mathbb{T}}(\Z)=\U * \U ^\mathrm{T} *  \Z+\Z * \V * \V^\mathrm{T}-\U * \U * \Z * \V * \V^\mathrm{T},
$$
and
$\mathcal{P}_{\mathbb{T}^{\perp}}(\Z)=\Z-\mathcal{P}_{\mathbb{T}}(\Z).$
Similarly, define $\mathbb{U},\mathbb{V}$ by the set
$$
\begin{aligned}
 \mathbb{U}=\left\{\U * \X^\mathrm{T} \mid \X \in \mathbb{R}^{k \times r \times n_1 \times \cdots \times n_p} \right\},
\mathbb{V}=\left\{\Y *\V^\mathrm{T} \mid  \Y \in \mathbb{R}^{t \times r \times n_1 \times \cdots \times n_p}\right\},   
\end{aligned}
$$
the projection on the set $\mathbb{U},\mathbb{V}$ is as follows:
$$
\mathcal{P}_{\mathbb{U}}(\Z)=\U * \U ^\mathrm{T} *  \Z
,
\mathcal{P}_{\mathbb{V}}(\Z)=\Z *\V * \V^\mathrm{T}.
$$
\end{defn}

\begin{defn}[Subgradient of temporal Hankel tensor nuclear norm]
Suppose $\mathcal{H}_k(\M)$ $\in  \mathbb{R}^{t\times k \times n_1 \times \cdots \times n_p}$ with $\operatorname{rank}_{\text {t-SVD }}(\mathcal{H}_k(\M))=r$, and it  has skinny t-SVD  $\mathcal{H}_k(\M) = \U*\mathcal{S}*\V^\mathrm{T}$. Then the subdifferential (the set of subgradients) of $\|\cdot\|_{\circledast}$ at $\mathcal{H}_k(\M)$ is
\begin{equation}
\partial\|\mathcal{H}_k(\M)\|_{\circledast}=\left\{\U * \mathcal{V}^\mathrm{T}+\W \mid \U^\mathrm{T} * \W=0,
\W *\V=0,\|\W\| \leq 1 \right\}.
\end{equation}
\end{defn}
\begin{lemmaa}\label{lemmauv}
  Suppose that $\M \in \mathbb{R}^{t  \times n_1\times\cdots\times n_p}$  obeys the temporal Hankel  tensor incoherence conditions and $\rho(\Omega) > 1-\alpha k /(2 \mu r (r_s+1)t )$  where $\rho(\Omega) $ is the minimum  temporal sampling rate of the original sampling set  $\Omega \in  \left[t \right] \otimes \left[n_1\right]  \otimes  \cdots \otimes \left[n_p\right] $, $\Omega_{\mathcal{H}} \in  \left[t \right] \otimes \left[k \right] \otimes\left[n_1\right]  \otimes  \cdots \otimes \left[n_p\right] $ is the temporal Hankel sampling set corresponding to the original sampling set $\Omega$,  then 
  \begin{equation}
  \left\| \Pu \Pomegahc \Pu \right\|_{op} < \frac{\alpha }{2 (r_s+1)}, ~\left\| \Pv \Pomegahc \Pv \right\|_{op} <  \frac{\alpha }{2 (r_s+1)}.
  \end{equation}
\end{lemmaa}
\begin{proof}
 The proof is provided in Section 2.1 of the Supplementary Material (SM2.1).   
\end{proof}
\begin{lemmaa}\label{lemmatuv}
  Let $\mathcal{H}_k(\M) \in \mathbb{R}^{ t\times k \times n_1 \times \cdots \times n_p}$ with skinny t-SVD $\mathcal{H}_k(\M) = \U*\mathcal{S}*\V^\mathrm{T}$,  $\Omega_{\mathcal{H}}^{\perp} \subseteq \left[t \right] \otimes \left[n_1\right]  \otimes  \cdots \otimes \left[n_p\right] $, $\Pt,\Pu,\Pv$ are given by  Definition \ref{TUV}, then we have 
 \begin{equation}
 \normop{\Pt\Pomegahc\Pt} \leq \normop{\Pu\Pomegahc\Pu}+ \normop{\Pv\Pomegahc\Pv}.
 \end{equation}
  \end{lemmaa}
  \begin{proof}
 The proof is provided in Section 2.2 of the Supplementary Material (SM2.2).   
\end{proof}
\begin{lemmaa}\label{lemmaeq}
Let $\mathcal{H}_k(\M) \in \mathbb{R}^{ t\times k \times n_1 \times \cdots \times n_p}$ with skinny t-SVD $\mathcal{H}_k(\M) = \U*\mathcal{S}*\V^\mathrm{T}$, $\Pt$ is given by  $\Pt(\cdot)=\U * \U ^\mathrm{T} * (\cdot)+(\cdot) * \V * \V^\mathrm{T}-\U * \U * (\cdot) * \V * \V^\mathrm{T}$, $\Omega_{\mathcal{H}}\subseteq \left[t \right] \otimes \left[n_1\right]  \otimes  \cdots \otimes \left[n_p\right] $, then  $\Pt \Pomegah \Pt$ is invertible and
 $\normop{\Pt \Pomegahc \Pt}<1$ are equivalent.
 \end{lemmaa}
 \begin{proof}
 The proof is provided in Section 2.3 of the Supplementary Material (SM2.3).   
\end{proof}
\begin{lemmaa}\label{lemma2.18}
  Let $\mathcal{H}_k(\M) \in \mathbb{R}^{t\times k \times n_1 \times \cdots \times n_p}$ with $\operatorname{rank}_{\operatorname{t-SVD}}(\mathcal{H}_k(\M))=r$, and it has
the skinny t-SVD $\mathcal{H}_k(\M)= \U*\mathcal{S} *\V^\mathrm{T}$,  if the operator $\Pt\Pomegah\Pt$ is invertible, then we have   $$\normop{\Ptc \Pomegah \Pt (\Pt\Pomegah\Pt)^{-1}}=\sqrt{\frac{1}{1-\normop{\Pt\Pomegahc\Pt}}-1}.$$
\end{lemmaa}
\begin{proof}
 The proof is provided in Section 2.4 of the Supplementary Material (SM2.4).   
\end{proof}
\begin{lemmaa}\label{lemmadu}
 Suppose that $\M \in \mathbb{R}^{t  \times n_1\times\cdots\times n_p}$  obeys the temporal Hankel  tensor incoherence conditions, if 
$\rho(\Omega) > 1- \alpha k/( 2\mu rt( r_s +1))$, where
$\alpha < 1 $, $\rho(\Omega)$ is the temporal sampling rate of the original sampling set  $\Omega \in  \left[t \right] \otimes \left[n_1\right]  \otimes  \cdots \otimes \left[n_p\right] $, $\Omega_{\mathcal{H}} \in  \left[t \right] \otimes \left[k \right] \otimes\left[n_1\right]  \otimes  \cdots \otimes \left[n_p\right] $ is the temporal Hankel sampling set corresponding to the original sampling set $\Omega$, 
then the following conditions hold:
1. $\left\| \Pt \Pomegac \Pt\right\| < 1$;
2. There exists a dual certificate $ \Lambda \in \mathbb{R}^{t\times k \times n_1 \times \cdots \times n_p}$ such that $\Pomegah(\Lambda)=\Lambda$ and
 $(a)\left\|\Pt(\Lambda)\right\| < \sqrt{\frac{\alpha r_s}{r_s+1-\alpha}}<  1 $;
 $(b)\Pt(\Lambda)=\U *\V^\mathrm{T} $.
\end{lemmaa}
\begin{proof}
 The proof is provided in Section 2.5 of the Supplementary Material (SM2.5).   
\end{proof}
\begin{lemmaa}\label{lemmaxm}
  Suppose that $\M \in \mathbb{R}^{t \times n_1\times\cdots\times n_p}$ obeys the temporal Hankel tensor incoherence conditions, 
if 
$\rho(\Omega) > 1-\alpha k /(2 \mu r (r_s+1)t )$,
then for any tensor $\X \in \mathbb{R}^{t\times n_1\times\cdots\times n_p}$  with
$\Pomega(\X)=\Pomega(\M)$, we have
$$
\begin{aligned}
&\norm{\mathcal{H}_k(\X)}_{\circledast} \geq \norm{\mathcal{H}_k(\M)}_{\circledast} 
+\left(1-\sqrt{\frac{\alpha r_s}{r_s+1-\alpha}}\right)\norm{\Ptc(\mathcal{H}_k(\X)-\mathcal{H}_k(\M))}_{\circledast}.
\end{aligned}
$$
\end{lemmaa}
\begin{proof}
 The proof is provided in Section 2.6 of the Supplementary Material (SM2.6).   
\end{proof}

\section{Proof of  Proposition 5.2}\label{appendix:d}
\begin{proof}
Under the assumption of Bernoulli sampling $\Omega\sim\operatorname{Ber}(\theta)$,
\begin{equation}
\mathbb{P}(\frac{x_1+\cdots+x_t}{t} \leq \theta-a) \leq \exp \left(-2 a^2 t\right)
\end{equation} holds for each  temporal mask vector $(i_t,:,\cdots,:)$ of $\M \in \mathbb{R}^{t \times n_1\times\cdots\times n_p}$ according to the Hoefding inequality.
Assume that the sampling of each  temporal mask vector is independent, then
\begin{equation}
\mathbb{P}(\rho(\Omega) \leq \theta-a) \leq \exp \left(-2 a^2 m_0\right),
\end{equation}
where $m_0=tkn_1 \cdots n_p$.
Setting $a=\theta-1+k/{ 2 \mu rt (r_s+1)}$ implies 
\begin{equation} 
\mathbb{P}\left(\rho(\Omega) \leq 1-\frac{k}{2 u rt\left(r_s+1\right)}\right) \leq \exp \left(-2 a^2 m_0\right),
\end{equation}
which in turn means that
\begin{equation}
 \mathbb{P}\left(\rho(\Omega)>1-\frac{k}{2 urt\left(r_s+1\right)}\right) \geqslant 1-\exp \left(-2 a^2  m_0\right).
\end{equation}
\end{proof}


\bibliographystyle{siamplain}
\bibliography{references}

\end{document}